\documentclass[10pt,journal,compsoc]{IEEEtran}
\ifCLASSOPTIONcompsoc
  \usepackage[nocompress]{cite}
\else
  \usepackage{cite}
\fi
\ifCLASSINFOpdf
\else
\fi

\usepackage{amsmath,amssymb,graphicx,epsf,psfrag,epsfig,cite,color,amsthm,algorithm,enumerate,bm,subfloat}
\usepackage[noend]{algpseudocode}
\usepackage{epstopdf}
\usepackage{subcaption}
\makeatletter
\def\BState{\State\hskip-\ALG@thistlm}
\makeatother

\newcommand{\stcomp}[1]{\overline{#1}}
\newtheorem{theorem}{Theorem}

\newtheorem{proposition}{Proposition}
\newtheorem{assumption}{Assumption}

\newtheorem{corollary}{Corollary}
\newtheorem{definition}{Definition}
\newtheorem{remark}{Remark}
\DeclareMathOperator*{\argmin}{arg\,min}
\DeclareMathOperator*{\argmax}{arg\,max}

\begin{document}
%
\title{Multi-Armed Bandits on Partially Revealed \\Unit Interval Graphs}
%
%
%
%

\author{Xiao Xu,~\IEEEmembership{Student Member,~IEEE,}
		Sattar Vakili, 
        Qing~Zhao,~\IEEEmembership{Fellow,~IEEE,}
        and~Ananthram~Swami,~\IEEEmembership{Fellow,~IEEE}
\IEEEcompsocitemizethanks{\IEEEcompsocthanksitem Xiao Xu and Qing Zhao are with the School
of Electrical and Computer Engineering, Cornell University, Ithaca,
NY, 14853, USA. E-mail: \{xx243;qz16\}@cornell.edu.
\IEEEcompsocthanksitem Sattar Vakili is with Prowler.io, Cambridge, UK. \protect\\E-mail: sv388@cornell.edu.
\IEEEcompsocthanksitem Ananthram Swami is with the CCDC Army Research Laboratory, Adelphi, MD, 20783, USA. E-mail: a.swami@ieee.org.}
\thanks{This work was supported in part by the Army Research Laboratory Network Science CTA under Cooperative Agreement W911NF-09-2-0053. The work of Qing Zhao was supported in part by the European Union’s Horizon 2020 research and innovation programme under the Marie Skłodowska-Curie grant agreement No 754412, during her visit at Chalmers University, Sweden.}
\thanks{Parts of the work have been presented at the 36th IEEE Military Communication Conference (MILCOM), October, 2017 and the 52nd Asilomar Conference on Signals, Systems and Computers, October, 2018.}}

\IEEEtitleabstractindextext{%
\begin{abstract}
A stochastic multi-armed bandit problem with side information on the similarity and dissimilarity across different arms is considered. The action space of the problem can be represented by a unit interval graph (UIG) where each node represents an arm and the presence (absence) of an edge between two nodes indicates similarity (dissimilarity) between their mean rewards. Two settings of complete and partial side information based on whether the UIG is fully revealed are studied and a general two-step learning structure consisting of an offline reduction of the action space and online aggregation of reward observations from similar arms is proposed to fully exploit the topological structure of the side information. In both cases, the computation efficiency and the order optimality of the proposed learning policies in terms of both the size of the action space and the time length are established.
\end{abstract}

\begin{IEEEkeywords}
Multi-armed bandits, unit interval graph, side information.
\end{IEEEkeywords}}

\maketitle

\IEEEdisplaynontitleabstractindextext

%
\IEEEpeerreviewmaketitle

\IEEEraisesectionheading{\section{Introduction}\label{sec:introduction}}
\IEEEPARstart{A}{}number of emerging applications involve large-scale online learning in which the objective is to learn, in real time, the most rewarding actions among a large number of options. Example applications include various socio-economic applications (e.g. ad display in search engines, product/news recommendation systems, targeted marketing and political campaigns) and networking issues (e.g. dynamic channel access and route selection) in large-scale communication systems such as Internet of things. For such problems, a linear scaling of the learning cost with the problem size resulting from exploring every option to identify the optimal is undesirable, if not infeasible. The key to achieving a sublinear scaling~with~the~problem~size~is~to~exploit~the~inherent~structure~of~the action space, i.e., various relations among the vast number of options.

A classic framework for online learning and sequential decision-making under unknown models is the multi-armed bandit (MAB) formulation. In the classic setting, a player chooses one arm (or more generally, a fixed number of arms) from a set of $K$ arms (representing all possible options) at each time and obtains a reward drawn i.i.d. over time from an unknown distribution specific to the chosen arm. The design objective is a sequential arm selection policy that maximizes the total expected reward over a time horizon of length $T$ by striking a balance between learning the unknown reward models of all arms (exploration) and capitalizing on this information to maximize the instantaneous gain (exploitation). The performance of an arm selection policy is measured by regret, defined as the expected cumulative reward loss against an omniscient player who knows the reward models and always plays the best arm.

A traditionally adopted assumption in MAB is that arms are independent and that there is no structure in the set of reward distributions. In this case, reward observations from one arm provide no information on other arms, resulting in a linear regret order in $K$. The main focus of the classic MAB problems has been on the regret order in $T$, which measures the learning efficiency over time. The seminal work by Lai and Robins showed that the minimum regret has a logarithmic order in $T$ \cite{lai1985asymptotically}. A number of learning policies have since been developed that offer the optimal regret order in $T$ (see \cite{auer2002finite, garivier2011kl,vakili2013deterministic} and references therein). Developed under the assumption of independent arms and relying on exploring every arm sufficiently often, however, these learning policies are not suitable for applications involving a massive number of arms, especially in the regime of~$K > T$.
\vspace{-.5cm}
\subsection{Main Results}

In addressing the challenge of massive number of arms, there has been a growing body of studies aiming at exploiting certain side information on the relations among the large number of arms. Among various formulations of the side information (see a more detailed discussion in Sec. \ref{relatedwork}), one notable example is the statistical similarity and dissimilarity among arms. For instance, in recommendation systems and information retrieval, products, ads, and documents in the same category (more generally, close in some feature space) have similar expected rewards. At the same time, it may also be known \emph{a priori} that some arms have considerably different mean rewards, e.g., news with drastically different opinions, products with opposite usage, documents associated with key words belonging to distant categories in the taxonomy. Such side information opens the possibility of efficient solutions~that~scale~well~with~the~large~action~space.

In this paper, we study a bandit problem with side information on similarity and dissimilarity relations across actions. We first show that the similarity-dissimilarity structure of the action space can be represented by a unit interval graph (UIG) where the presence (absence) of an edge between two arms indicates that the difference of their mean rewards is within (beyond) a given threshold. Based on whether the UIG is fully revealed to the player, we consider two cases of complete and partial side information. For both cases, we propose a general two-step learning structure---LSDT (Learning from Similarity-Dissimilarity Topology)---to achieve a full exploitation of the topological structure of the side information. The first step is an offline reduction of the action space to a candidate set, which consists of arms that can assume the largest mean rewards under certain assignments of reward distributions without violating the side information. Arms outside the candidate set are sub-optimal and eliminated from online exploration. The second step carries out an online learning algorithm that further exploits the similarity structure through collective exploration by aggregating reward observations from similar arms.

In the case of complete side information, we show that the candidate set is given by the set of left anchors of the UIG, which can be identified by a Breadth-First-Search (BFS) based algorithm in polynomial time. By defining an equivalence relation between arms through the neighbor sets in the UIG, we obtain an equivalence class partition of arms. We show that the candidate set consists of at most two equivalence classes if the UIG is connected. We exploit this topological structure by maintaining two UCB (upper confidence bound) indices, one at the class level aggregating observations from arms within the same class, the other at the arm level. At each time, the arm with the largest arm index within the class with the largest class index is played. We establish the order optimality of the proposed policy in terms of both $K$ and $T$ by deriving an upper bound on regret and a matching lower bound feasible among uniformly good policies.

In the case of partial side information, we represent the partially revealed UIG by a multigraph with two types of edges indicating the presence and the absence of the corresponding UIG edges. We show the NP-completeness of finding the candidate set and propose a polynomial time approximation algorithm to reduce the action space. We show that under certain probabilistic assumptions on the partial side information, the size of the reduced actions space is comparable to that of the ground truth candidate set as determined by the underlying UIG. In the second step of online learning, the key to a full exploitation of the similarity relation is to determine the frequency of exploring an arm based on its exploration value, which measures the topological significance of the node in a similarity graph. By sequentially eliminating arms less likely to be optimal through a UCB index aggregating observations from similar arms, only arms close to the optimal one remain after a sufficient number of plays. We provide performance guarantee for the proposed policy and establish its order optimality under certain probabilistic assumptions on the side information.

It should be noted that the main issue and main contribution of this paper are on how to succinctly model and fully exploit the side information  on the similarity and dissimilarity relations across arms. The solution to the former is the UIG representation of the actions space, and to the latter is the two-step learning structure LSDT, which is independent of the specific arm selection rule adopted at the online learning step. In particular, different arm selection techniques developed for the original bandit problems may be incorporated into the second step of LSDT, except based on aggregated observations. In Sec. 5.3, we discuss the use of Thompson Sampling (TS), one of the most well-known learning techniques in bandit problems (see \cite{thompson1933likelihood,chapelle2011empirical,agrawal2012analysis,agrawal2013further} and references therein), with LSDT to fully exploit the side information.

In summary, we develop a UIG formulation of side information on arm similarity and dissimilarity in MAB problems and consider two cases of complete and partial side information in this paper. We propose a general and computationally efficient two-step learning structure achieving full exploitation of side information and establish the order optimality of the proposed learning policies through theoretical upper bounds on regret as well as matching lower bounds in both cases.

\vspace{-.25cm}
	
\subsection{Related Work}\label{relatedwork}

Existing studies on MAB with structured reward models can be categorized based on the types of arm relations adopted in the MAB models. The first type is realization-based relation that assumes a certain known probabilistic dependency across arms. Examples include combinatorial bandits \cite{liu2012adaptive,gai2012combinatorial,chen2013combinatorial,kveton2015tight}, linearly parameterized bandits \cite{dani2008stochastic,rusmevichientong2010linearly,abbasi2011improved}, and spectral bandits for smooth graph functions \cite{valko2014spectral,hanawal2015efficient}. The second type of arm relation can be termed as observation-based relation \cite{caron2012leveraging,buccapatnam2014stochastic,alon2014nonstochastic}. Specifically, playing an arm provides additional side observations about its neighboring arms. See \cite{valko2016bandits} for a survey on various bandit models with structured action spaces.

The problem studied in this paper considers another type of relation among arms: ensemble-based relation that aims to capture the relations on ensemble behaviors (i.e., mean rewards) across arms, rather than probabilistic dependencies in their realizations\footnote{The mean rewards of different arms exhibit certain relations (e.g., closeness or in certain orders) but the realized random rewards of arms being played do not need to exhibit any probabilistic dependency.}. Related work includes Lipschitz bandits \cite{agrawal1995continuum,kleinberg2008multi,magureanu2014lipschitz}, taxonomy bandits \cite{slivkins2011multi} and unimodal bandits \cite{combes2014unimodal}. Specifically, in Lipschitz bandits, the mean reward is assumed to be a Lipschitz function of the arm parameter. Taxonomy bandits have a tree-structured action space where arms in the same subtree are close in their mean rewards. In unimodal bandits, the action space is represented by a graph where from every sub-optimal arm, there exists a path to the optimal arm along which the mean reward increases. Different from these existing studies, the bandit model studied in this paper considers an action space represented by a UIG indicating not only similarity but also dissimilarity relations across actions. Besides, the structure of the proposed learning policy consists of a two-level exploitation of the UIG structure, which is fundamentally different from the existing ones. Recently, a general formulation of structured bandits was proposed in \cite{combes2017minimal}, which includes a variety of known bandit models (e.g., Lipschitz bandits, unimodal bandits, linear bandits, etc.) as well as the bandit model studied in this work as special cases. The learning policy developed in \cite{combes2017minimal}, however, was given only implicitly in the form of a linear program (LP) that needs to be solved at every time step. For the problem studied in this paper, the LP does not admit polynomial-time solutions (unless P=NP).

Side information has also been used to refer to context information in the so-called contextual bandits (see \cite{langford2008epoch,chapelle2011empirical,li2010contextual} and references therein). Under this formulation, context information is revealed at each time, which affects the arm reward distributions. A contextual bandit problem can thus be viewed as multiple simple bandits, one for each context, that are interleaved in time according to the context stream. The complexity of the problem comes from the coupling of these simple bandits by assuming various models on how context affects the arm reward distributions. The problem is fundamentally~different~from~the~one~studied~here.

\section{Multi-Armed Bandits on Unit Interval Graphs}\label{sec:MABonUIG}

\subsection{Problem Formulation}\label{subsec:formulation}
Consider a stochastic $K$-armed bandit problem. At each time $t$, a player chooses one arm to play. Playing an arm $i$ yields a reward $X_i(t)$ drawn i.i.d. from an unknown distribution $f_i$ with mean $\mu_i$. We assume that $f_i$ belongs to the family of sub-Gaussian distributions\footnote{A random variable $Y$ with mean $\mu$ is sub-Gaussian with parameter $\sigma$ (or $\sigma$ sub-Gaussian) if $\mathbb{E}[e^{\lambda(Y-\mu)}]\le e^{\sigma^2\lambda^2/2}$, for all $\lambda\in\mathbb{R}$ \cite{buldygin2000metric}.} for all $i$. Extensions to other distribution types are discussed in Sec. \ref{sec:discussion}.

Across $K$ arms, the similarity and dissimilarity relations are defined through a parameter $\epsilon>0$: two arms are similar (dissimilar) if the difference between their mean rewards is below (above) $\epsilon$. The similarity-dissimilarity structure of the action space can be represented by an undirected graph $\mathcal{G}_{\epsilon}^*=(\mathcal{V},\mathcal{E}_{\epsilon}^*)$. In the graph representation, every node $i\in\mathcal{V}$ represents an arm with reward distribution $f_i$ and the presence (absence) of an edge $(i,j)$ corresponds to a similar (dissimilar) arm pair. Throughout the paper, $1\le i\le K$ is used to refer to an arm or a node, exchangeably. We first show that $\mathcal{G}_{\epsilon}^{*}$ is a UIG. 

\begin{definition}[\textbf{Unit interval graph and unit interval model}]
A graph $\mathcal{G}=(\mathcal{V,E})$ is a unit interval graph if there exists a set of unit length intervals\footnote{If a UIG is finite (with a finite number of nodes), there is no difference between taking open intervals or closed intervals to represent nodes \cite{frankl1987open}. Without loss of generality, we assume that $I_i=(l_i,r_i)$ where $l_i,r_i$ are the left and right coordinates of interval $I_i$.} $\{I_{i}\}_{i\in\mathcal{V}}$ on the real line such that each interval $I_i$ corresponds to a node $i\in\mathcal{V}$ and there exists an edge $(i,j)\in\mathcal{E}$ if and only if $I_i\cap I_j\neq\emptyset$. The set of intervals $\{I_i\}_{i\in\mathcal{V}}$ is a unit interval model (UIM) for the UIG.
\end{definition} 
Through a mapping from every node $i\in\mathcal{V}$ to an $\epsilon$-length interval $I_i=(\mu_i,\mu_i+\epsilon)$, it is not difficult to see that 
\begin{align}
	|\mu_i-\mu_j|<\epsilon\Leftrightarrow I_i\cap I_j\neq\emptyset,
\end{align}
which indicates that $\mathcal{G}_{\epsilon}^*$ is a UIG (see an example in Fig. \ref{fig:uig}). Without loss of generality, we assume that $\mathcal{G}_{\epsilon}^*$ is connected. Extensions to the disconnected case are discussed in Sec. \ref{sec:discussion}.
\begin{figure}[t!]
	\begin{center}	
	\vspace{-.8cm}
		\includegraphics[width=0.45\textwidth]{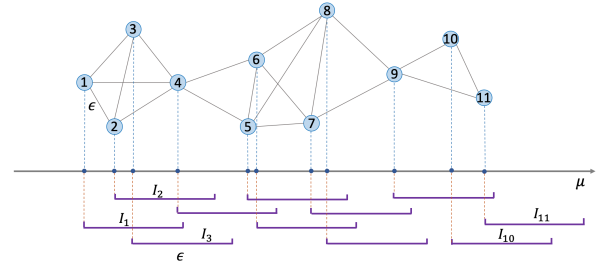}
		\vspace{-.8cm}
		\caption{\small{action space as a UIG: each node (arm) $i$ is associated with an $\epsilon$-length interval $I_i$}.}
		\label{fig:uig}
		\vspace{-.5cm}
	\end{center}
\end{figure}

We define $\mathcal{E}_{\epsilon}^S$, $\mathcal{E}_{\epsilon}^D$ as the side information on arm similarity and dissimilarity. Based on whether $\mathcal{E}_{\epsilon}^S,\mathcal{E}_{\epsilon}^D$ fully reveal the UIG $\mathcal{G}_{\epsilon}^*$, we consider the following two cases separately. In the case of complete side information, $\mathcal{E}_{\epsilon}^S,\mathcal{E}_{\epsilon}^D$ are identical to the edge set and the complement edge set of $\mathcal{G}_{\epsilon}^*$, i.e., $\mathcal{E}_{\epsilon}^S=\mathcal{E}_{\epsilon}^*,\mathcal{E}_{\epsilon}^D=\stcomp{\mathcal{E}_{\epsilon}^*}$. In the case of partial side information, they are subsets of the latter, i.e., $\mathcal{E}_{\epsilon}^{S}\subseteq{\mathcal{E}_{\epsilon}^*}$, $\mathcal{E}_{\epsilon}^{D}\subseteq\stcomp{\mathcal{E}_{\epsilon}^*}$.

The objective is an online learning policy $\pi$ that specifies a sequential arm selection rule at each time $t$ based on both past observations of selected arms and the side information $\mathcal{E}_{\epsilon}^{S},\mathcal{E}_{\epsilon}^{D}$. The performance of policy $\pi$ is measured by regret $R_{\pi}(T;\mathcal{E}_{\epsilon}^S,\mathcal{E}_{\epsilon}^D)$ defined as the expected reward loss against a player who knows the reward model and always plays the best arm $i_{\max}$ (chosen arbitrarily in the case of multiple optimal arms), i.e.,
\begin{equation}\label{regret}
	R_{\pi}(T;\mathcal{E}_{\epsilon}^S,\mathcal{E}_{\epsilon}^{D})=\mathbb{E}_{\pi}\left[\sum_{t=1}^{T}\mu_{i_{\max}}(t)-\sum_{t=1}^{T}X_{\pi_t}(t)\right],
\end{equation}
where $\mu_{i_{\max}}$ is the largest mean reward and $\pi_t$ is the arm selected by policy $\pi$ at time $t$. The dependency of regret on the unknown reward distributions ${\bm{f}}=(f_1,...,f_K)$ is omitted in the notation. When there is no ambiguity, the notation is simplified to $R(T)$.

Let $\tau_i(T)$ denote the number of times that arm $i$ has been selected up to time $T$. We rewrite the regret as:
\begin{equation}\label{regret2}
	R(T)=\mu_{i_{\max}}T-\sum_{i=1}^{K}\mu_i\mathbb{E}[\tau_i(T)]=\sum_{i=1}^{K}\Delta_i\mathbb{E}[\tau_i(T)],
\end{equation}
where $\Delta_i=\mu_{i_{\max}}-\mu_i$. The objective of maximizing the expected cumulative reward is equivalent to minimizing the regret over a time horizon of length $T$. In order to minimize regret, it can be inferred from (\ref{regret2}) that every sub-optimal arm ($\Delta_i>0$) should be distinguished from the optimal one with the least number of plays.

\subsection{Two-Step Learning Structure}\label{subsec:learningstructure}
While classic bandit algorithms have to try out every arm sufficiently often to distinguish the sub-optimal arms from the optimal one, which induces a linear scaling of regret in the number of arms, the side information on arm similarity and dissimilarity allows the possibility of identifying a set of sub-optimal arms without even playing them. To be specific, we define a candidate set $\mathcal{B}$ determined by the side information $\mathcal{E}_{\epsilon}^S,\mathcal{E}_{\epsilon}^D$ as follows.

\begin{definition}[\textbf{Candidate Arm and Candidate Set}]
Given the side information $\mathcal{E}_{\epsilon}^S,\mathcal{E}_{\epsilon}^D$, an arm $i$ is a candidate arm if there exists an assignment of reward distributions with means $\bm{\mu}=(\mu_1,...,\mu_K)$ conforming to $\mathcal{E}_{\epsilon}^S,\mathcal{E}_{\epsilon}^D$ and $\mu_i=\max_{1\le j\le K}\mu_j$. The candidate set $\mathcal{B}$ is the set consisting of all candidate arms.
\end{definition}

Note that the optimal arm $i_{\max}$ under the ground truth assignment of reward distributions in the bandit problem always belongs to the candidate set $\mathcal{B}$. It is clear that if we can find the candidate set $\mathcal{B}$ from the side information efficiently, the action space can be reduced to $\mathcal{B}$. Only arms in $\mathcal{B}$ need to be explored. Furthermore, certain topological structures of the revealed UIG on the reduced action space can be further exploited to accelerate learning. In estimating the mean reward of every arm in the candidate set, observations from similar arms can also be leveraged as approximations, which reduces the number of plays required to distinguish sub-optimal arms from the optimal one.

The aforementioned facts motivate a general two-step learning structure: \emph{Learning from Similarity-Dissimilarity Topology (LSDT)} for both cases of complete and partial side information. Specifically, LSDT consists of (1) an offline elimination step that reduces the action space to the candidate set and (2) online learning of the optimal arm by aggregating observations from similar ones. We specify each step for the cases of complete and partial side information separately in Sec. \ref{sec:complete} and Sec. \ref{sec:partial}.

\section{Complete Side Information}\label{sec:complete}
We first consider the case of complete side information that fully reveals the UIG $\mathcal{G}_{\epsilon}^*$. We follow the two-step learning structure proposed in Sec. \ref{subsec:learningstructure} and develop a learning policy: \emph{LSDT-CSI (Learning from Similarity-Dissimilarity Topology with Complete Side Information)} along with theoretical analysis on its regret performance. While restrictive in applications, this case provides useful insights for tackling the general case of partial side information addressed in Sec. \ref{sec:partial}. 
\vspace{-.3cm}
\subsection{Offline Elimination}

The first step of LSDT-CSI is an offline preprocessing that aims at identifying the candidate set from the complete side information. Since the UIG $\mathcal{G}_{\epsilon}^*$ is fully revealed, we denote the candidate set in this case as $\mathcal{B}^*$ to distinguish from the case of partial side information. We show that $\mathcal{B}^*$ is identical to the set of {\it{left anchors}} of the UIG $\mathcal{G}_{\epsilon}^*$.

\begin{definition}[\textbf{Left Anchor}]
Given a UIG $\mathcal{G}=({\mathcal{V,E}})$, a node $i\in\mathcal{V}$ is a left anchor if there exists a UIM for $\mathcal{G}$ where $i$ corresponds to the leftmost interval along the real line.
\end{definition}
Since the mirror image of an UIM with respect to the origin is also an UIM for the same UIG, the node corresponding to the rightmost interval in a UIM is also a left anchor. Based on the definition of the UIG $\mathcal{G}_{\epsilon}^*$ that represents the similarity-dissimilarity structure of the arm set in Sec. \ref{subsec:formulation}, it is not difficult to see that the candidate set $\mathcal{B}^*$ is identical to the set of left anchors of $\mathcal{G}_{\epsilon}^*$, which can be identified through a BFS-based algorithm proposed in \cite{corneil1995simple}. The BFS-based algorithm starts from an arbitrary node in a UIG and returns a set of left anchors. We apply the algorithm two times: in the first time, we start from an arbitrary node in $\mathcal{G}_{\epsilon}^*$ and obtain a set of left anchors. In the second time, we re-apply the algorithm starting from one of the returned node in the last time. One can directly infer from Proposition 2.1 and Theorem 2.3 in \cite{corneil1995simple} that the obtained set is the candidate set $\mathcal{B}^*$. The detailed algorithm is summarized below. Note that the computation complexity of the offline elimination step is $O(|\mathcal{E}_{\epsilon}^*|)$, which is polynomial in the problem size.
\begin{algorithm}
\caption*{~~{\bf{LSDT-CSI}} (Step 1): Offline Elimination}\label{getcandidateset}
	\begin{algorithmic}
		\BState \textbf{Input}: Fully revealed UIG $\mathcal{G}_{\epsilon}^*$.
		\vspace{.05cm}
		\BState \textbf{Output}: Candidate set $\mathcal{B}^*$.
		\vspace{.05cm}
		\BState \textbf{Initialization}: $\mathcal{B}^*=\emptyset$.
		\vspace{.05cm}

			\State Start from an arbitrary node $i$ and perform a BFS on $\mathcal{G}_{\epsilon}^*$.
			\State Let $\mathcal{L}$ be the set of nodes in the last level of the BFS.
			\For{each $j\in\mathcal{L}$}
				\If{ deg$(j)=\min_{k\in \mathcal{L}}\textrm{deg}(k)$}
					 \State $\mathcal{B}^*\leftarrow\mathcal{B}^*\cup\{j\}$.
				\EndIf
			\EndFor
		\BState Start from a node $j\in\mathcal{B}^*$ and repeat the previous steps.
		
	\end{algorithmic}
\end{algorithm}
\vspace{-.3cm}
\subsection{Online Aggregation}

We now present the second step of online learning that further exploits topological structures of the candidate set $\mathcal{B}^*$. We first introduce an equivalence relation between nodes in the UIG $\mathcal{G}_{\epsilon}^*$.
\begin{definition}[\textbf{Neighborhood Equivalence}]
	
	Two nodes $i,j$ in $\mathcal{G}_{\epsilon}^*$ are (neighborhood) equivalent if $\mathcal{N}[i]=\mathcal{N}[j]$, where $\mathcal{N}[i]$ is the set of neighbors of $i$ in $\mathcal{G}_{\epsilon}^*$, including $i$. Moreover, let $\{\mathcal{B}_{i}^*\}$ denote the partition of the arm set $\mathcal{V}$ in $\mathcal{G}_{\epsilon}^*$ with respect to the neighborhood equivalence relation.
\end{definition}

Note that arms within the same equivalence class have the same set of neighbors and thus, they are topologically indistinguishable in the UIG. Based on the equivalence class partition, we obtain a closed-form expression for $\mathcal{B}^*$.

\begin{theorem}\label{candidateset}
	When the side information fully reveals the UIG $\mathcal{G}_{\epsilon}^*$ (assumed to be connected), the candidate set $\mathcal{B}^*$ is the union of two equivalence classes containing the optimal arm $i_{\max}$ and the worst arm $i_{\min}$ (with minimum mean reward)\footnote{Note that the two equivalence classes containing the optimal arm and the worst arm are identical in the special case where $\mathcal{G}^*$ is fully connected. The proposed algorithm and analysis still apply in this case. Without loss of generality, we assume that $\mathcal{G}_{\epsilon}^*$ is not fully connected.}, i.e., 
	\begin{align}\label{B}
		\mathcal{B}^*=\mathcal{B}_{i_{\max}}^*\cup\mathcal{B}^*_{i_{\min}},
	\end{align}
	 where 
\begin{align}\label{B*}
	\mathcal{B}_{i_{\max}}^*=\{j:\mathcal{N}[j]=\mathcal{N}[i_{\max}]\},
\end{align}
\begin{align}\label{B0}
	\mathcal{B}_{i_{\min}}^*=\{j:\mathcal{N}[j]=\mathcal{N}[i_{\min}]\}.
\end{align}
\end{theorem}
\begin{proof}
	See Appendix B in the supplementary material.
\end{proof}
\begin{figure}
\vspace{-.3cm}
	\begin{center}
		\includegraphics[width=0.5\textwidth]{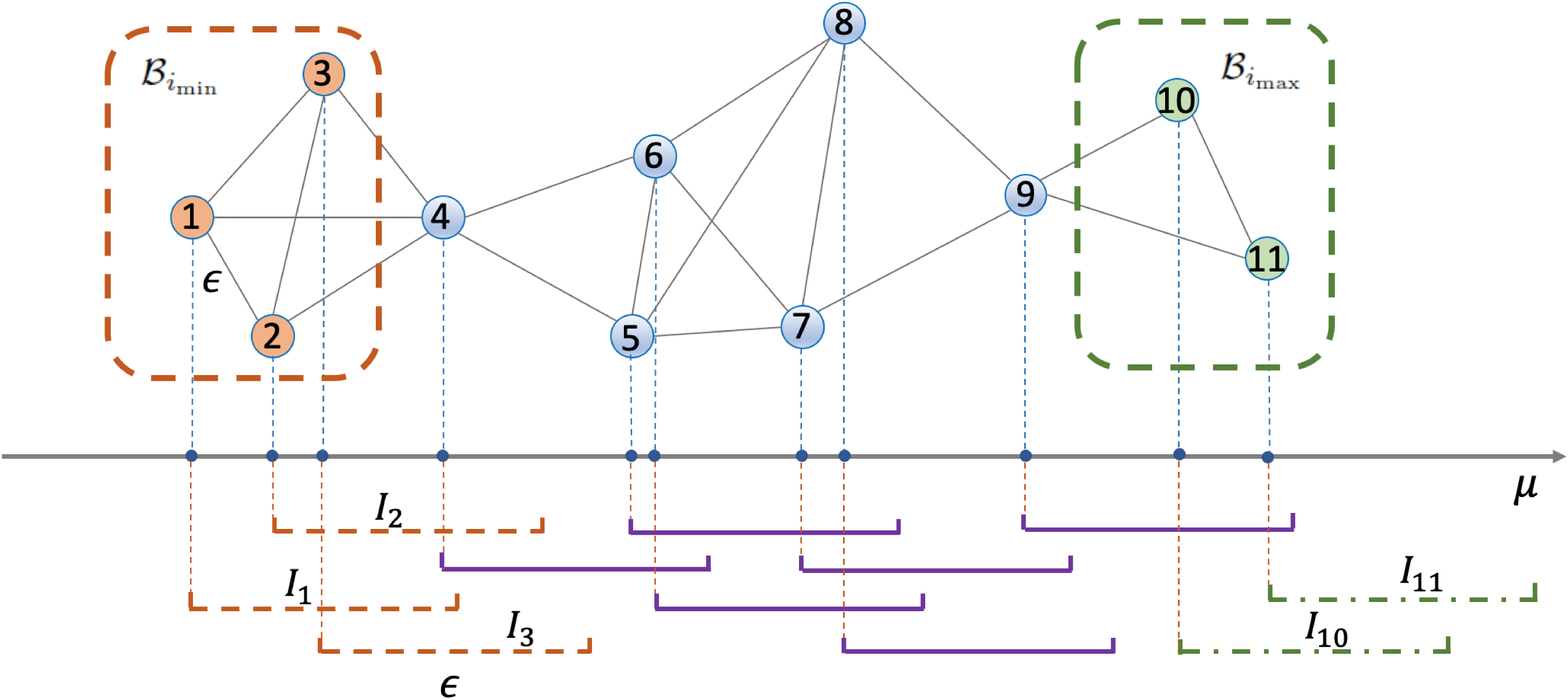}
		\vspace{-1.6cm}
		\caption{\small{left anchors and candidate set: the node corresponding to $I_1$ (or $I_{11}$) is the left anchor under the current UIM (or its mirroring). Switching $I_1,I_2,I_3$ (or $I_{10},I_{11}$) does not change the graph connectivity, i.e., each node in $\mathcal{B}_{i_{\min}}$ (or $\mathcal{B}_{i_{\max}}$) is a left anchor. Hence the candidate set $\mathcal{B}=\{1,2,3\}\cup\{10,11\}$}.}
		\label{fig:UIG}
	\end{center}
		\vspace{-.5cm}
\end{figure}
The result is also illustrated in Fig. \ref{fig:UIG}: the candidate set $\mathcal{B}^*$ is the union of two equivalence classes $\mathcal{B}^*_{i_{\min}}=\{1,2,3\}$ and $\mathcal{B}^*_{i_{\max}}=\{10,11\}$, which can be directly obtained through the offline elimination step. 

Based on the topological structure of the candidate set, we develop a hierarchical online learning policy that aggregates observations from arms within the same equivalence class. By considering each class as a super node (arm), we reduce the problem to a simple two-armed bandit problem. 

Specifically, the second step of LSDT-CSI carries out a hierarchical UCB-based online learning on the candidate set $\mathcal{B}^*$ by maintaining a class index $H_i(t)$ for each equivalence class $\mathcal{B}_i^*$ and an arm index $L_j(t)$ for each individual arm $j$ in $\mathcal{B}^*$. The arm index is defined as:
\begin{align}\label{Lindex}
	L_j(t)=\bar{x}_j(t)+\sqrt{\frac{8\log t}{\tau_j(t)}},
\end{align}
where $\bar{x}_j(t)$, $\tau_j(t)$ are the empirical average of observations from arm $j$ and the number of times that arm $j$ has been played up to time $t$. The class index $H_i(t)$ aggregates the same statistics across arms in the class:
\begin{align}\label{Hindex}
	H_i(t)=\frac{\sum_{j\in \mathcal{B}^*_i}\bar{x}_{j}(t)\tau_{j}(t)}{\sum_{j\in \mathcal{B}^*_i}\tau_{j}(t)}+\sqrt{\frac{8\log t}{\sum_{j\in \mathcal{B}^*_i}\tau_{j}(t)}}.
\end{align}
At each time, the online learning procedure selects the equivalence class with the largest class index and plays the arm with the largest arm index within the selected class. Once the reward has been observed, both class indices and arm indices are updated. 
\begin{algorithm}[h!]
\caption*{~~{\bf{LSDT-CSI}} (Step 2): Online Aggregation}\label{Online Learning Procedure}
	\begin{algorithmic}
		\BState \textbf{Input}: Candidate set $\mathcal{B}^*=\mathcal{B}^*_1\cup\mathcal{B}^*_2$ where $\mathcal{B}^*_1,\mathcal{B}^*_2$ are two disjoint equivalence classes.
		\BState \textbf{Initialization}: Play each arm in $\mathcal{B}^*$ once, update all the arm indices $\{L_j(t)\}_{j\in\mathcal{B}^*}$ and class indices $H_1(t),H_2(t)$ defined in (\ref{Lindex}) and (\ref{Hindex}).

		\vspace{0.1cm}
		\For{$t=|\mathcal{B}^*|+1,|\mathcal{B}^*|+2,...$}
		\vspace{0.1cm}
			\State Let $i_t^*=\argmax_{i=1}^{2}H_i(t-1)$. 
			\State Play arm $j_t^*=\argmax_{j\in\mathcal{B}^*_{i_t^*}}L_j(t-1).$
		\EndFor
	\end{algorithmic}
\end{algorithm}
\vspace{-.3cm}
\subsection{Order Optimality}\label{subsec:optHUCB}
We first present the regret analysis of LSDT-CSI, which focuses on upper bounding the expected number of times that each suboptimal arm has been played up to time $T$. We show that when the total number of times that arms in $\mathcal{B}^*_{i_{\min}}$ have been played is greater than $\Omega(\log T)$, the class index $H_{i_{\min}}(t)$ will not be chosen with high probability. Besides, if each suboptimal arm $j\in\mathcal{B}_{i_{\max}}^*$ has been played more than $\Omega(\log T)$ times, the arm index $L_j(t)$ will not be chosen with high probability. The following theorem provides the performance guarantee for LSDT-CSI.

\begin{theorem}\label{thmupperbound}
		Suppose that $\mathcal{G}_{\epsilon}^*$ is connected. Assume that the reward distribution for each arm is sub-Gaussian with parameter $\sigma=1$ \footnote{See Sec. \ref{sec:discussion} for extensions to general $\sigma$.}. Then the regret of LSDT-CSI up to time $T$ is upper bounded as follows:
		
		\begin{equation}\label{upperbound}
		\begin{aligned}
		R(T)\le&\Bigg(\frac{32\max_{i\in \mathcal{B}_{i_{\min}}}\Delta_i}{(\min_{j\in \mathcal{B}^*_{i_{\min}}}\Delta_j-\max_{k\in \mathcal{B}^*_{i_{\max}}}\Delta_k)^2}\\
		&+\sum_{i\in \mathcal{B}^*_{i_{\max}}\setminus\mathcal{A}}\frac{32}{\Delta_i}\Bigg)\log T +O(|\mathcal{B}^*|),
		\end{aligned}
		\end{equation}
		where $\mathcal{A}$ is the set of arms with largest mean rewards ($i_{\max}\in\mathcal{A}$).
	\end{theorem}

\begin{proof}
	See Appendix C in the supplementary material.
\end{proof}

\begin{remark}
	For fixed $\Delta_i$, the regret of LSDT-CSI is of order 
	\begin{align}
	O\Big((1+|\mathcal{B}^*_{i_{\max}}\setminus\mathcal{A}|)\log T\Big),
	\end{align} 
	as $T\to\infty$. In certain scenarios (e.g., $\mathcal{G}^*_{\epsilon}$ is a line graph), $|\mathcal{B}^*_{i_{\max}}|\ll K$, which indicates a sublinear scaling of regret in terms of the number of arms given such side information.
	 
\end{remark}
\begin{remark}
	If $\mathcal{G}^*_{\epsilon}$ is fully connected (e.g., $\epsilon$ is large), then $\mathcal{B}^*_{i_{\max}}=\mathcal{B}^*_{i_{\min}}=\mathcal{V}$. In this case, LSDT-CSI degenerates to the classic UCB policy and $R(T)\sim O(K\log T)$.
\end{remark}
We discuss in Sec. \ref{sec:numerical} that if the mean reward of each arm is independently and uniformly chosen from $[0,1]$ and $\epsilon$ is bounded away from $0$ and $1$, the expected value of $|\mathcal{B}^*|$ is smaller than $O(K^{1/2}\log K)$, which indicates a sublinear scaling of regret in terms of the size of the action space. We also use a numerical example to verify the result in Sec. \ref{sec:numerical}.

To establish the order optimality of LSDT-CSI, we further derive a matching lower bound on regret. We focus here on the case that the unknown mean reward of each arm is unbounded (i.e., can be any value on the real line). We adopt the same parametric setting as in \cite{lai1985asymptotically} on classic MAB where the rewards are drawn from a specific parametric family of distributions with known distribution type\footnote{Although the upper bound on regret of LSDT-CSI is derived under the non-parametric setting (the distribution type is unknown), the non-parametric lower bound suffices to show the order optimality of LSDT-CSI since it should be no smaller than that in the parametric one.}. Specifically, the reward distribution of arm $i$ has a univariate density function $f(\cdot;\theta_i)$ with an unknown parameter $\theta_i$ from a set of parameters $\Theta$. Let $I(\theta||\lambda)$ be the Kullback-Leibler (KL) distance between two distributions with density functions $f(\cdot;\theta)$ and $f(\cdot;\lambda)$ and with means $\mu(\theta)$ and $\mu(\lambda)$ respectively. We assume the same regularity assumptions on the finiteness of the KL divergence and its continuity with respect to the mean values as in \cite{lai1985asymptotically}.

\begin{assumption}\label{lowassum1}
	For every $f(\cdot;\theta),f(\cdot;\lambda)$ such that $\mu(\lambda)>\mu(\theta)$, we have $0<I(\theta||\lambda)<\infty$.
\end{assumption}

\begin{assumption}\label{lowassum2}
	For every $\epsilon>0$ and $\theta,\lambda\in\Theta$ with $\mu(\lambda)>\mu(\theta)$, there exists $\eta>0$ for which $|I(\theta||\lambda)-I(\theta||\rho)|<\epsilon$ whenever $\mu(\lambda)<\mu(\rho)<\mu(\lambda)+\eta,~\rho\in\Theta$.
\end{assumption}

The following theorem provides a lower bound on regret for uniformly good policies\footnote{A policy $\pi$ is uniformly good if for every $\bm{f}$, the regret of $\pi$ satisfies $R(T)=o(T^\alpha),\forall \alpha>0$, as $T\to\infty$ \cite{lai1985asymptotically}.}.
\begin{theorem}\label{lowerbound}
	Suppose $\mathcal{G}_{\epsilon}^*$ is connected. Assume that Assumptions \ref{lowassum1}, \ref{lowassum2} hold and the mean reward of each arm can be any value in $\mathbb{R}$. For any uniformly good policy, the regret up to time $T$ is lower bounded as follows:
	\begin{equation}
		\lim_{T\to\infty}\frac{R(T)}{\log T}\ge C_1,
	\end{equation}
	where $C_1$ is the optimal value of an LP that only depends on $f_1,...,f_K$ and $\epsilon$ (see (\ref{lowerboundLP}) in Appendix \ref{pflowerbound} for details). It can be shown that for fixed $\Delta_i$, $I(\theta_i||\theta_i')$ and $I(\theta_i||\theta_{i_{\max}})$, the regret for any uniformly good policy is of order $${\Omega}\Big((1+|\mathcal{B}^*_{i_{\max}}\setminus\mathcal{A}|)\log T\Big),$$
	as $T\to\infty$.

\end{theorem}

\begin{proof} 
See Appendix D in the supplementary material.
\end{proof}

\begin{remark}
	LSDT-CSI is order optimal since its upper bound on regret matches the lower bound shown in Theorem \ref{lowerbound}.
\end{remark}
\begin{remark}
If there is a unique optimal arm, i.e., $|\mathcal{A}|=1$,
	$
	R(T)\sim\Theta\Big(|\mathcal{B}_{i_{\max}}^*|\log T\Big),
	$
	as $T\to\infty$.
\end{remark}

\section{Partial Side Information}\label{sec:partial}
In this section, we consider the general case of partial side information where the UIG $\mathcal{G}_{\epsilon}^*$ is partially revealed. We develop a learning policy: \emph{LSDT-PSI (Learning from Similarity-Dissimilarity Topology with Partial Side Information)} following the two-step structure proposed in Sec. \ref{subsec:learningstructure} and provide theoretical analysis on the regret performance.

\subsection{Offline Elimination}
A partially revealed UIG can be represented by an undirected edge-labeled multigraph $\mathcal{G}_{\epsilon}=(\mathcal{V},\mathcal{E}_{\epsilon}^S,\mathcal{E}_{\epsilon}^D)$ (see Fig. \ref{fig:observedSIG}). Specifically, $\mathcal{G}_{\epsilon}$ consists of two types of edges: type-S edges ($\mathcal{E}_{\epsilon}^S$) and type-D edges ($\mathcal{E}_{\epsilon}^D$) indicating the presence and the absence of the corresponding UIG edges. The absence of an edge between two nodes indicates an unknown relation between the two arms.

We first show that finding the candidate set under partial side information $\mathcal{E}_{\epsilon}^S,\mathcal{E}_{\epsilon}^D$ is NP-complete. We notice that finding the candidate set is equivalent to considering every node $i$ individually and deciding if $i$ can be a left anchor of a UIG $\mathcal{G}'_{\epsilon}=(\mathcal{V},\mathcal{E}^{P}_{\epsilon})$ consisting of the same set of nodes with $\mathcal{G}_{\epsilon}$ and the potential edge set $\mathcal{E}_{\epsilon}^{P}$ satisfying 
\begin{align}\label{requir1}
\mathcal{E}^{S}_{\epsilon}\subseteq\mathcal{E}^{P}_{\epsilon},
\end{align}
\vspace{-.5cm}
\begin{align} \label{requir2}
\mathcal{E}^{P}_{\epsilon}\cap\mathcal{E}^{D}_{\epsilon}=\emptyset.
\end{align} 

Specifically, we show the NP-completeness of the following decision problem.

\begin{itemize}
	\item[]	\emph{LEFTANCHOR}
	\item[] \emph{{{[INPUT]}}: A multigraph $\mathcal{G}=(\mathcal{V},\mathcal{E}_1,\mathcal{E}_2)$ knowing that there exists a UIG $\mathcal{G}'=(\mathcal{V},\mathcal{E}_3)$ where $\mathcal{E}_1\subseteq\mathcal{E}_3$ and $\mathcal{E}_3\cap\mathcal{E}_2=\emptyset$, and a specific node $i$.}

	\item[]\emph{{[QUESTION]}: Does there exist a UIG $\mathcal{G}''=(\mathcal{V},\mathcal{E}_4)$ where $\mathcal{E}_1\subseteq\mathcal{E}_4$ and $\mathcal{E}_4\cap\mathcal{E}_2=\emptyset$ such that node $i$ is a left anchor of $\mathcal{G}''$?}
\end{itemize}

\begin{theorem}\label{LEFTNP}
LEFTANCHOR is NP-complete.
\end{theorem}
\begin{proof}

To show the NP-completeness of LEFTANCHOR, we give a reduction from a variant of the 3-SAT problem: CONSISTENT-NAE-3SAT. Due to the page limit, we include the definition of CONSISTENT-NAE-3SAT as well as its proof of NP-completeness in Appendix E in the supplementary material. The reduction to LEFTANCHOR and the remaining proof are presented in Appendix F.
\end{proof}
\begin{figure}
	\begin{center}
		
		\includegraphics[width=0.5\textwidth]{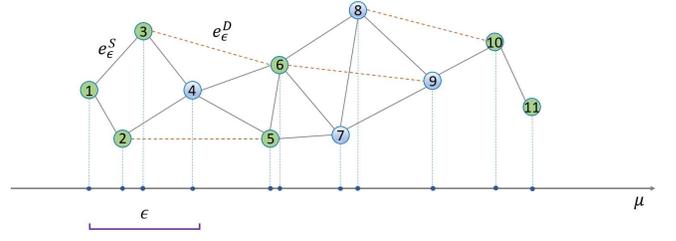}
		\vspace{-2.5cm}
		\caption{\small{Partially revealed UIG as an undirected edge-labeled multigraph: black solid lines represent type-S edges and red dash lines represent type-D edges. The candidate set $\mathcal{B}=\{1,2,3,5,6,10,11\}$: take $\epsilon=0.15$, there exists a graph conforming assignment of mean rewards $\bm{\mu}=(0.8,0.8,0.8,0.9,$ $1,1,0.9,0.9,0.8,0.7,0.6)$, where node $5$ and $6$ are optimal}.}
		\label{fig:observedSIG}
		\vspace{-.6cm}
	\end{center}
	
\end{figure}

It should be noted that LEFTANCHOR is similar to the so-called UIG Sandwich Problem\cite{kaplan1994complexity} where two graphs $\mathcal{G}_1=(\mathcal{V},\mathcal{E}_1)$ and $\mathcal{G}_2=(\mathcal{V},\mathcal{E}_2)$ are given satisfying $\mathcal{E}_1\subseteq\mathcal{E}_2$. The question is whether a UIG $\mathcal{G}_3=(\mathcal{V},\mathcal{E}_3)$ exists satisfying $\mathcal{E}_1\subseteq\mathcal{E}_3\subseteq\mathcal{E}_2$. It is not difficult to see that the type-S edge set $\mathcal{E}^{S}_{\epsilon}$ corresponds to $\mathcal{E}_1$ in the sandwich problem and the complement of $\mathcal{E}^{D}_{\epsilon}$ corresponds to $\mathcal{E}_2$. However, LEFTANCHOR is different from the sandwich problem as we know that the sandwich problem is satisfied by the ground truth UIG $\mathcal{G}_{\epsilon}^*$, and what we are interested in is whether a specific node $i$ can be a left anchor.

To address the challenge of finding the candidate set in polynomial time, we exploit the following topological property of $\mathcal{G}_{\epsilon}$ to obtain an approximation solution.
\begin{proposition}\label{prop2}
Given $\mathcal{G}_{\epsilon}$, an arm $i$ is sub-optimal if it is similar to two dissimilar arms, i.e., if there exist $j,k$, such that $(i,j),(i,k)\in\mathcal{E}_{\epsilon}^S$ but $(j,k)\in\mathcal{E}_{\epsilon}^{D}$, then $i\not\in\mathcal{B}$.
\end{proposition}

Based on this property, we develop the offline elimination step of LSDT-PSI with $O(K|\mathcal{E}_{\epsilon}^D|)$ complexity.

\begin{algorithm}
\caption*{~~{\bf{LSDT-PSI}} (Step 1): Offline Elimination}\label{EUCB-offline}
	\begin{algorithmic}
		\BState \textbf{Input}: $\mathcal{G}_{\epsilon}=(\mathcal{V},\mathcal{E}_{\epsilon}^S,\mathcal{E}_{\epsilon}^D)$.
		\vspace{.1cm}
		\BState \textbf{Output}: $\mathcal{B}_0$.
		\vspace{.1cm}
		\BState \textbf{Initialization}: $\mathcal{B}_0=\mathcal{V}$.
		\vspace{.1cm}
		\For{$i=1,2,...,K$}
		\vspace{.1cm}
			\State $\mathcal{B}_0\leftarrow \mathcal{B}_0\setminus\{i\}$ if there exist $j,k\in\mathcal{V}$ such that 
			\vspace{-.1cm}			
			$$(i,j), (i,k)\in\mathcal{E}_{\epsilon}^{S}, (j,k)\in\mathcal{E}^{D}_{\epsilon}.$$
			\vspace{-.5cm}
		\EndFor
	\end{algorithmic}
\end{algorithm}

It is clear that in general, $\mathcal{B}^*\subseteq\mathcal{B}\subseteq\mathcal{B}_0$. However, in certain scenarios, the partially revealed UIG provides sufficient topological information to identify the ground truth candidate set $\mathcal{B}^*$ obtained from the fully revealed UIG. We show that such information is fully exploited by the offline elimination step of LSDT-PSI to achieve the same performance as that of LSDT-CSI for the case of complete side information. 

Specifically, we make the following assumptions on $\mathcal{G}_{\epsilon}^*$ and its equivalence classes $\{\mathcal{B}_i^*\}_{i=1}^{m}$ assuming that the neighbor set of every arm $i\not\in\mathcal{B}^*$ is diverse enough. Without loss of generality, we assume an increasing order of the equivalence classes along the real line, i.e., $\forall 1\le i<j\le m$ and $\forall k_i\in\mathcal{B}_{i}^*,k_j\in\mathcal{B}_j^*$, we have $\mu_{k_i}<\mu_{k_j}$. Note that $\mathcal{B}^*=\mathcal{B}_1^*\cup\mathcal{B}^*_m$.

\begin{assumption}\label{diversity1}
For every $1<i<m$, assume that there exist $j,k$ such that $j<i<k$ and $\mathcal{B}_{j}^*,\mathcal{B}_{k}^{*}$ are connected to $\mathcal{B}_{i}^*$ but mutually disconnected in $\mathcal{G}_{\epsilon}^{*}$.\footnote{Two equivalence classes are connected if and only if at least one pair of arms from the two classes are adjacent in the UIG. It can be inferred from the equivalence relation that if there exists an adjacent arm pair from the two classes, all arm pairs are adjacent.}
\end{assumption}
\begin{assumption}\label{diversity2}
Assume that there exists a constant $\kappa>0$ and for every $i$, $|\mathcal{B}_i^*|\ge \kappa\log K$.
\end{assumption}

We further make a probabilistic assumption on the partial side information.
\begin{assumption}\label{observation}
The presence and the absence of an edge in the UIG $\mathcal{G}_{\epsilon}^*$ are revealed by the partial side information $\mathcal{E}_{\epsilon}^S$ and $\mathcal{E}_{\epsilon}^D$ independently with probabilities $p_S$  and $p_D$. Assume that $p_S^2p_D\ge 1-e^{-2/\kappa}$, where $\kappa$ is defined in Assumption \ref{diversity2}.
\end{assumption}
Note that as $\kappa$ increases, for every arm $i\not\in\mathcal{B}^*$, the number of dissimilar arm pairs that are similar to $i$ increases. Therefore, smaller probabilities of observing edges can still guarantee that arm $i$ is elilminated with high probability.

Based on these assumptions, we provide performance guarantee for the offline elimination step of LSDT-PSI through the following theorem. We also verify the results through numerical examples in Sec. \ref{sec:numerical}.

\begin{theorem}\label{sizeB0}
Given a UIG $\mathcal{G}^*_{\epsilon}$, under Assumptions \ref{diversity1}-\ref{observation}, with probability at least $1-\frac{1}{K^2}$, every arm $i\not\in\mathcal{B}^*$ is eliminated by the offline elimination step of LSDT-PSI and thus,
\begin{align}
\mathbb{E}_{\mathcal{E}_{\epsilon}^S,\mathcal{E}_{\epsilon}^D}\Big[\big|\mathcal{B}_0\big|\Big]=\big|\mathcal{B}^*\big|+o(1),
\end{align}
as $K\to\infty$, where $\mathcal{B}_0$ is the arm set remaining after the offline elimination step of LSDT-PSI.
\end{theorem}
\begin{proof}
See Appendix G in the supplementary material.
\end{proof}

\subsection{Online Aggregation}
Now we present the second step, the online learning procedure of LSDT-PSI. We first define a similarity graph $\mathcal{G}_{\epsilon}'=(\mathcal{V}',{\mathcal{E}_{\epsilon}^{S}}')$ restricted to the remaining arm set $\mathcal{B}_0$: $\mathcal{V}'=\mathcal{B}_0$ and ${\mathcal{E}^{S}_{\epsilon}}'=\{(i,j)\big|i,j\in \mathcal{B}_0, (i,j)\in\mathcal{E}^{S}_{\epsilon}\}.$
For every arm $i\in\mathcal{B}_0$, we define an exploration value $z_i\in[0,1]$, which measures the topological significance of node $i$ in the similarity graph $\mathcal{G}_{\epsilon}'$ and determines the frequency of playing arm $i$. Intuitively, a node with a higher degree has a higher exploration value since playing this node provides information about more (neighboring) nodes. Specifically, we define exploration values $\{z_i\}_{i\in\mathcal{B}_0}$ as the optimal solution to the following LP.
\vspace{-.1cm}
\begin{equation}
	\begin{aligned}
		\mathcal{P}_2:&~C_2=\min_{\{z_i\}_{i\in\mathcal{V}'}}\sum_{i\in\mathcal{V}'}z_i,\\
		s.t.&\sum_{j\in\mathcal{N}'[i]}z_j\ge 1,~\forall i\in\mathcal{V}',\\
		&z_i\ge 0,~\forall i\in\mathcal{V}',
	\end{aligned}
\end{equation}
where $\mathcal{N}'[i]$ is the set of neighbors of node $i$ in $\mathcal{G}_{\epsilon}'$ (including $i$). In the online learning procedure, the number of times arm $i$ is played is proportional to its exploration value $z_i$. Note that if at least $n_i$ plays are necessary to distinguish a suboptimal arm $i$ from the optimal one in the classic MAB problem, now if suffices to play only $z_in_i$ times by aggregating observations from every neighboring arm $j\in\mathcal{N}'[i]$ that is played $z_jn_i$ times. Note that $z_i\le 1,\forall i$ and $C_2$ is upper bounded by the size of the minimum dominating set of ${\mathcal{G}_{\epsilon}}'$.

We briefly summarize the second step of LSDT-PSI: the algorithm is played in epochs and during epoch $m$, arms are played up to $\tau_i(m)\sim \Theta(z_i\log T)$ times. Arms less likely to be optimal are eliminated at the end of every epoch and only two types of arms will be played in the next epoch: 1) non-eliminated arms and 2) arms with non-eliminated neighbors. After~a~sufficient~number~of~epochs,~only~arms~close~to~the~optimal~one~remain~and~we~use single arm indices for selection. Let $\bar{x}_i(m)$ be the average reward from arm $i$ up to epoch $m$.
		\begin{algorithm}
\caption*{~~{\bf{LSDT-PSI}} (Step 2): Online Aggregation}\label{EUCB-online}
	\begin{algorithmic}
		\BState \textbf{Input}: $\mathcal{G}'_{\epsilon}=(\mathcal{V}',{\mathcal{E}_{\epsilon}^S}')$, time horizon $T$, parameter $\lambda>0$.
		\vspace{.1cm}
		\BState {\bf{Initialization}}: Let $\tilde{\Delta}_0=1$, $\mathcal{S}_0=\mathcal{B}_0$, $\{z_i\}_{i\in\mathcal{V}'}$ be the solution to $\mathcal{P}_2$, $m_f=\min\left\{\left\lceil\log_2\left(\frac{8}{\sqrt{2\lambda}\epsilon}\right)\right\rceil,\left\lfloor\frac{1}{2}\log_2 \frac{T}{e}\right\rfloor\right\}$.\vspace{.1cm}
		\For{$m = 0,1,...,m_f$}
		\vspace{0.1cm}
			\If{$|\mathcal{B}_m|=1$} Play $i\in \mathcal{B}_m$ until time $T$.
			\vspace{0.1cm}
			\Else \For{each arm $i\in \mathcal{S}_m$}
				\vspace{0.1cm}
					\State Play arm $i$ until
					$
						\tau_i(m)=\left\lceil\frac{\lambda z_i\log(T\tilde{\Delta}_m^2)}{\tilde{\Delta}_m^2}\right\rceil.
					$
				\EndFor
				\State Let $\mathcal{B}_{m+1}=\mathcal{B}_m$.
				\vspace{0.1cm}
				\For{each arm $i\in \mathcal{B}_m$}
				\vspace{0.1cm}
					\State $\mathcal{B}_{m+1}\leftarrow \mathcal{B}_{m+1}\setminus\{i\}$ if
					\begin{equation}\label{UCBLCB}
					\begin{aligned}
						&\frac{\sum_{j\in\mathcal{N}'[i]}\bar{x}_j(m)\tau_j(m)}{\sum_{j\in\mathcal{N}'[i]}\tau_j(m)}+\sqrt{\frac{\log(T\tilde{\Delta}_m^2)}{2\sum_{j\in\mathcal{N}'[i]}\tau_j(m)}}+\epsilon\le \\
						&\max_{k\in \mathcal{B}_m}\left\{\frac{\sum_{j\in\mathcal{N}'[k]}\bar{x}_j(m)\tau_j(m)}{\sum_{j\in\mathcal{N}'[k]}\tau_j(m)}-\sqrt{\frac{\log(T\tilde{\Delta}_m^2)}{2\sum_{j\in\mathcal{N}'[k]}\tau_j(m)}}\right\}.
					\end{aligned}
					\end{equation}
				\EndFor
				\State Let $\mathcal{S}_{m+1}=\{i:\mathcal{N}'[i]\cap \mathcal{B}_{m+1}\neq\emptyset\}$.
				\vspace{.1cm}
				\State Let $\tilde{\Delta}_{m+1}=\tilde{\Delta}_m/2$.
			\EndIf
		\vspace{.1cm}	
		\EndFor		
		\For{$t=\sum_{i\in\mathcal{V}'}{\tau_i(m_f)}+1,...,T$}
		\vspace{.1cm}	
			\State Play arm $i^*_t=\argmax_{i\in \mathcal{B}_{m_f+1}}\bar{x}_i(t-1) + \sqrt{\frac{2\log (t-1)}{\tau_i(t-1)}}.$
		\EndFor
	\end{algorithmic}
\end{algorithm}
\vspace{-.3cm}
\subsection{Order Optimality}

The following theorem provides an upper bound on regret of LSDT-PSI for any given partially revealed UIG.

\begin{theorem}\label{upperboundpsi}
Given a partially revealed UIG $\mathcal{G}_{\epsilon}$. Assume that the reward distribution of reach arm is $\sigma=1/2$ sub-Gaussian\footnote{Certain sub-Gaussian distributions (e.g. Bernoulli distribution, uniform distribution on $[0,1]$) have parameters $\sigma=1/2$. See Sec. \ref{sec:discussion} for extensions to general $\sigma$.}. Let $\mathcal{Q}=\{i\in \mathcal{B}_0:\Delta_i>4\epsilon\}$. Then the regret of LSDT-PSI up to time $T$ is upper bounded by:
	
	\begin{equation}\label{ucbpsi}
		\begin{aligned}
			R(T)\le&\sum_{j\in \mathcal{B}_0\setminus(\mathcal{Q\cup A})}\Delta_j\max\left\{\frac{8\log T}{\Delta_j^2},\frac{32z_j\log(T\epsilon^2)}{\epsilon^2}\right\}+\\
			&\sum_{i\in \mathcal{Q}}\Delta_iz_i\frac{32\log(T\hat{\Delta}_i^2)}{\hat{\Delta}_i^2}+O(|\mathcal{V}'|),
		\end{aligned}
	\end{equation}
where $\hat{\Delta}_i=\max\{\min_{j\in\mathcal{N}'[i]}\Delta_j-3\epsilon,\epsilon\}$.
\end{theorem}

\begin{proof}
See Appendix H in the supplementary material.
\end{proof}

\begin{remark}
	For fixed $\Delta_i$, the regret of LSDT-PSI is of order
	\begin{align}
		O\Big((\gamma(\mathcal{G}_{\epsilon}')+|\mathcal{B}_0\setminus(\mathcal{Q}\cup\mathcal{A})|)\log T\Big),
	\end{align}
	as $T\to\infty$, where $\gamma(\mathcal{G}_{\epsilon}')$ is the size of the minimum dominating set of graph $\mathcal{G}_{\epsilon}'$ and $|\mathcal{B}_0\setminus(\mathcal{Q}\cup\mathcal{A})|$ is the number of sub-optimal arms that are $4\epsilon$-close to the optimal one. It is not difficult to see that as $\epsilon$ increases, $\gamma(\mathcal{G}_{\epsilon}')$ decreases and $|\mathcal{B}_0\setminus(\mathcal{Q}\cup\mathcal{A})|$ increases. For an appropriate $\epsilon$, a sublinear scaling of regret in terms of the number of arms can be achieved.
\end{remark}

Recall that in Theorem \ref{sizeB0}, we show that under certain assumptions, the offline elimination step of LSDT-PSI achieves the same performance as LSDT-CSI for the case of complete side information. The following corollary further shows the order optimality of LSDT-PSI in terms of both $K$ and $T$.

\begin{corollary}\label{optpsi}
	Assume that Assumptions \ref{diversity1}-\ref{observation} hold and $\Delta_i>4\epsilon,\forall i\in\mathcal{B}_{i_{\min}}^*$. For fixed $\Delta_i,p_S,p_D$, the expectation of regret of LSDT-PSI taken over random realizations of the partial side information $\mathcal{E}_{\epsilon}^S$, $\mathcal{E}_{\epsilon}^D$ is upper bounded as follows:
\begin{align}
\mathbb{E}_{\mathcal{E}_{\epsilon}^S,\mathcal{E}_{\epsilon}^D}[R(T)]\le O\Big((1+|\mathcal{B}_{i_{\max}}^*\setminus\mathcal{A}|)\log T\Big),
\end{align}
as $T\to\infty$, which matches the lower bound on regret for the case of complete side information established in Theorem \ref{lowerbound}.
\end{corollary}
\begin{proof}
See Appendix I in the supplementary material.
\end{proof}

\section{Extensions}\label{sec:discussion}
In this section, we discuss extensions of the proposed policies: LSDT-CSI and LSDT-PSI as well as their regret analysis to cases with disconnected UIGs and other reward distributions. We also discuss the extension of applying Thompson Sampling techniques to the LSDT learning structure.

\vspace{-.3cm}
\subsection{Extensions to disconnected UIG}
Suppose that the UIG $\mathcal{G}^*_{\epsilon}$ has $M$ ($M>1$) connected components. It is not difficult to see that every connected component of $\mathcal{G}_{\epsilon}^*$ is still a UIG and the set of left anchors of $\mathcal{G}_{\epsilon}^*$ is the union of left anchors of all components. Therefore, in the case of complete side information, the offline elimination step of LSDT-CSI outputs at most $2M$ equivalences classes and the second step of LSDT-CSI can be directly applied by maintaining a class index for every equivalence class as defined in (\ref{Hindex}). Moreover, by extending the regret analysis of LSDT-CSI in Theorem \ref{thmupperbound} as well as the lower bound on regret for uniformly good policies in Theorem \ref{lowerbound} to the disconnected case, we can show that LSDT-CSI achieves an order optimal regret, i.e.,
\begin{align}
	R(T)\sim\Theta\Big((M+|\mathcal{B}^*_{i_{\max}}\setminus\mathcal{A}|)\log T\Big),
\end{align}
as $T\to\infty$. In the extreme case when $M=K$ (e.g., $\epsilon\to 0$), LSDT-CSI degenerates to the classic UCB policy and $R(T)\sim\Theta(K\log T)$.

In the case of partial side information, the LSDT-PSI policy along with its regret analysis applies to any partially revealed UIG without assumptions on the connectivity of the graph. The upper bound on regret in Theorem \ref{upperboundpsi} still holds when $\mathcal{G}^*_{\epsilon}$ has $M$ connected components. In the extreme case where $M=K$, the size of the minimum dominating set of the similarity graph $\mathcal{G}_{\epsilon}'$ equals $K$ and thus, $R(T)\sim O(K\log T)$.

To show the order optimality of LSDT-PSI in the disconnected case, we need certain modifications on the assumptions of the UIG. We consider every connected component $m$ of $\mathcal{G}^*_{\epsilon}$ with $\ell$ equivalence classes $\{\mathcal{B}_{i}^{*^{(m)}}\}_{i=1}^{\ell}$. We assume that Assumptions \ref{diversity1} and \ref{diversity2} hold for every connected component and without loss of generality, we assume that the optimal arm $i_{\max}$ is in component $m=1$. Then under Assumption \ref{observation}, we can extend the regret analysis in Corollary \ref{optpsi} to the case where $\mathcal{G}^*_{\epsilon}$ has $M$ connected components. It can be shown that the expected regret of LSDT-PSI is upper bounded by 
\begin{align}
	O\Big((M+|\mathcal{B}^*_{i_{\max}}\setminus\mathcal{A}|)\log T\Big),
\end{align}
as $T\to \infty$, which matches the lower bound in the case of complete side information.
\vspace{-.3cm}
\subsection{Extensions to Other Distributions}
Recall that in the regret analysis of LSDT-CSI and LSDT-PSI, we assume sub-Gaussian reward distributions with parameter $\sigma=1$ (e.g., standard normal distribution) or $\sigma=1/2$ (e.g., Bernoulli distribution). We first discuss extensions to general sub-Gaussian distributions with arbitrary parameters $\sigma$. 

In LSDT-CSI, by replacing the second terms of the UCB indices defined in (\ref{Lindex}) and (\ref{Hindex}) by $\sqrt{\frac{\alpha\log t}{\tau_j(t)}}$ and $\sqrt{\frac{\alpha\log t}{\sum_{j\in\mathcal{B}_i^*}\tau_j(t)}}$ where $\alpha$ is an input parameter, the regret analysis in Theorem \ref{thmupperbound} still applies and the upper bound on regret is only affected up to a constant scaling factor, as long as $\alpha> 6\sigma^2$.  A similar extension also applies to LSDT-PSI if we change the second terms of the UCB indicies in (\ref{UCBLCB}) to $\sqrt{\frac{\beta\log (T\tilde{\Delta}_{m}^2)}{\sum_{j\in\mathcal{N}'[i]}\tau_j(m)}}$ where $\beta\ge 2\sigma^2$.

Furthermore, we can extend the results for sub-Gaussian reward distributions to other distribution types such as light-tailed and heavy-tailed distributions. There are standard techniques for such extensions by replacing the concentration result with the corresponding ones for light-tailed and heavy-tailed distributions (the latter also requires replacing sample means with truncated sample means). Similar extensions for classic MAB problems without side information are discussed in \cite{vakili2013achieving,vakili2013deterministic}. To illuminate the main ideas without too much technicality, most existing work assumes an even stronger assumption of bounded support in $[0,1]$ (see \cite{auer2002finite},\cite{garivier2011kl}, \cite{langford2008epoch}, etc.).
\vspace{-.3cm}
\subsection{Extensions to Thompson Sampling Techniques}
The two-step learning structure LSDT is in general independent of the specific arm selection rule adopted in the online learning step. We discuss here how Thompson Sampling (TS) techniques can be extended and incorporated into the basic structure with aggregation of reward observations. Specifically, in the case of complete side information, after reducing the action space to the candidate set via the offline step, we adopt a similar hierarchical online learning policy as that used in LSDT-CSI by maintaining two posterior distributions on the reward parameters, one at the equivalence class level, the other at the arm level. At each time, the policy first randomly selects an equivalence class according to its class-level probability of containing the optimal arm and then randomly draws an arm within the class according to its arm-level probability of being optimal. In the case of partial side information, similar to LSDT-PSI, an eliminative strategy is carried out to sequentially eliminate arms less likely to be optimal. At each time, an arm is randomly drawn according its arm-level posterior distribution of being optimal. The observation from the selected arm is also used to update higher level posterior distributions of its neighbors, which aggregate observations from all similar arms. According to the high level posterior distribution, the arm that is least likely to be optimal is eliminated if it has been explored for sufficient times. Simulation results in Appendix \ref{TSsimulation} show a similar performance gain by exploiting the side information on arm similarity and dissimilarity through the two-step learning structure when TS is incorporated in both cases. To achieve a full exploitation of the side information and establish the order optimality on regret, however, further studies are required.

\begin{figure}[t]
	\begin{center}
	\vspace{-.3cm}
	\begin{subfigure}[b]{0.25\textwidth}
	\begin{center}
		\includegraphics[width=\textwidth]{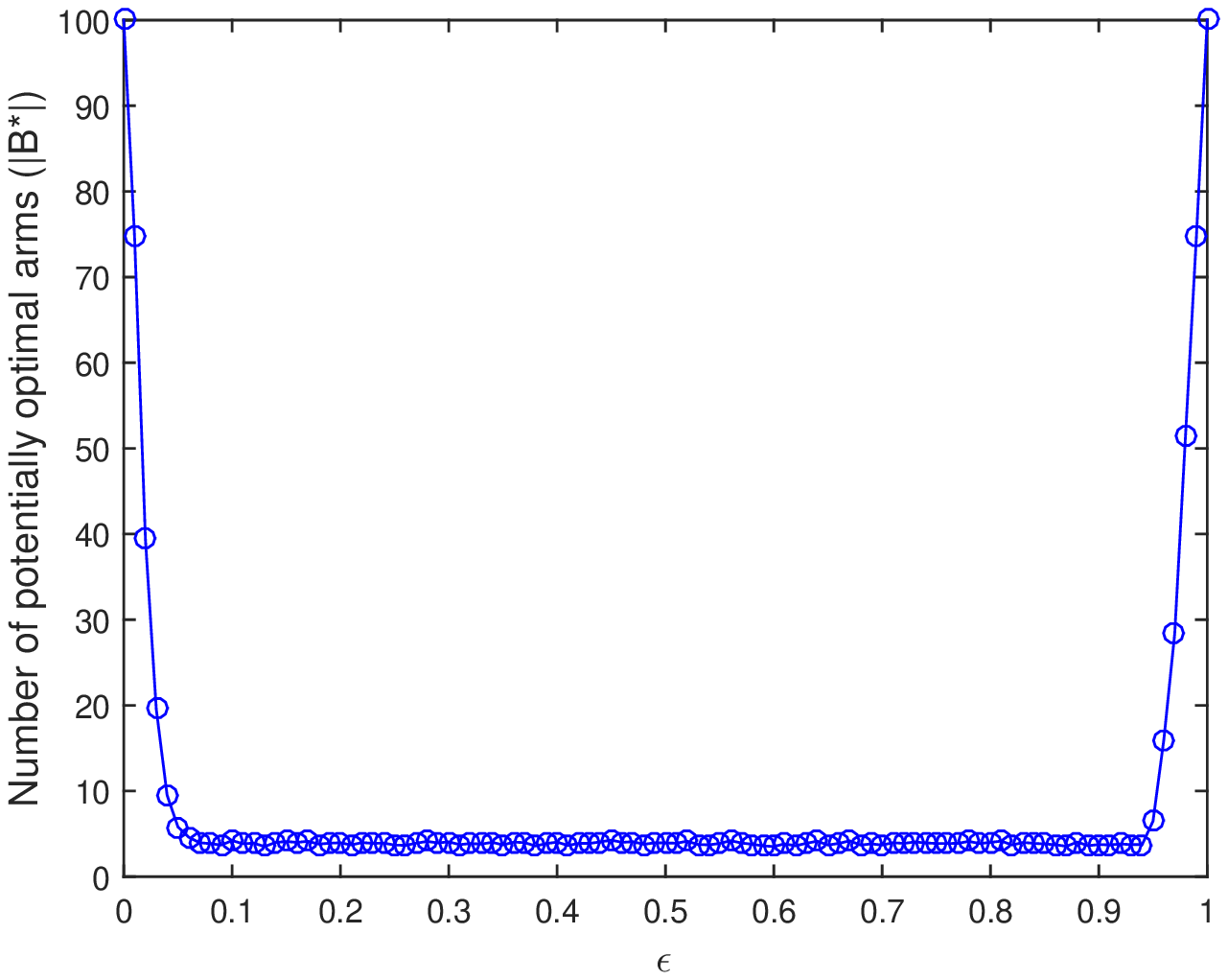}
		\caption{\small{ $|\mathcal{B}^*|$ v.s. $\epsilon$. }}
		\label{simulation1a}
	\end{center}
	\end{subfigure}%
	\hspace{-.3cm}
	\begin{subfigure}[b]{0.25\textwidth}
	\begin{center}
		\includegraphics[width=\textwidth]{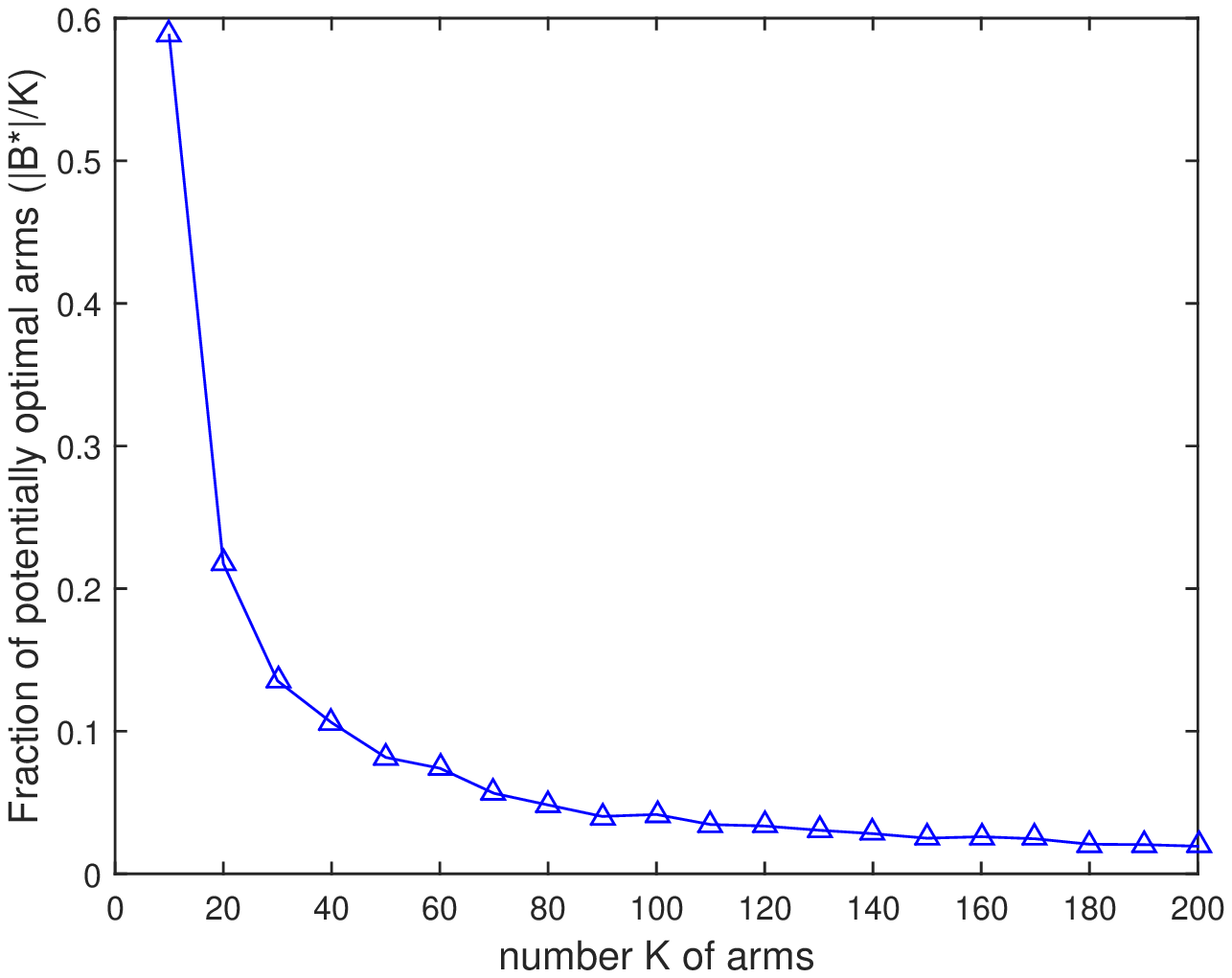}
		\caption{\small{ $\frac{|\mathcal{B}^*|}{K}$ v.s. $K$.}}
		\label{simulation1b}
	\end{center}
	\end{subfigure}
	\caption{\small{Reduction of the action space with complete side information: comparison between the size of the candidate set $|\mathcal{B}^*|$ and the number $K$ of arms. In (a), $K=100$, $|\mathcal{B}^*|\approx5$ when $\epsilon\in[0.1,0.9]$. In (b), $\epsilon=0.2$,$|\mathcal{B}^*|/K$ decreases as $K$ increases.}}
	\end{center}
	\vspace{-.5cm}
\end{figure}
\vspace{-.3cm}
\section{Numerical Examples}\label{sec:numerical}
In this section, we illustrate the advantages of our policies through numerical examples on both synthesized data and a real dataset in recommendation systems. All the experiments are run 100 times using a Monte Carlo method on MATLAB R2014b.
\vspace{-.5cm}
\subsection{Reduction of the action space}
\subsubsection{Complete Side Information}
We use two experiments to show how much the action space can be reduced by exploiting the complete side information. In the first experiment, we fix $K=100$ arms with mean rewards uniformly chosen from $(0,1)$ and let $\epsilon$ vary from $0$ to $1$. For every $\epsilon$, we obtain a UIG $\mathcal{G}_{\epsilon}^*$. We apply the offline elimination step of LSDT-CSI to $\mathcal{G}_{\epsilon}^*$ and compare the size of the candidate set $\mathcal{B}^*$ with $K$. In the second experiment, we fix $\epsilon=0.2$ and let $K$ increase from $10$ to $200$. We generate arms and UIGs in the same way as in the first experiment. We show how $|\mathcal{B}^*|/K$ varies as $K$ increases. The results are shown in Figs. \ref{simulation1a} and \ref{simulation1b}.

As we can see from Fig. \ref{simulation1a}, when $\epsilon$ is small ($\epsilon<0.1$), the graph is disconnected. As $\epsilon$ increases, the number of connected components decreases and thus, $|\mathcal{B}^*|$ decreases. When the graph is connected ($\epsilon>0.1$), the candidate set $\mathcal{B}^*$ only contains two equivalence classes and thus $|\mathcal{B}^*|$ is much smaller than $K$. When $\epsilon$ is large ($\epsilon>0.9$), the probability that the graph is complete increases as $\epsilon$ increases. In this case, the candidate set contains all the arms. Thus, $|B^*|$ increases to $K$ as $\epsilon$ grows to $1$. In Fig. \ref{simulation1b}, we notice that $\mathcal{B}^*$ has a diminishing cardinality compared with $K$. Since the mean rewards are uniformly chosen from $(0,1)$, the set of arms becomes denser on the interval $(0,1)$ as $K$ grows. It can be inferred from \cite{holst1980lengths} that the maximum distance $d$ between two consecutive points uniformly chosen from $(0,1)$ is in the order of $O(\frac{\log K}{\sqrt{K}})$ with probability $1-1/K$. If we choose $\epsilon=\frac{\rho\log K}{\sqrt{K}}$ for some $\rho>0$, $\mathcal{G}_{\epsilon}$ will be connected with high probability. Moreover, it can be shown that the cardinality of  $\mathcal{B}^*_{i_{\max}}$ ($\mathcal{B}^*_{i_{\min}}$) is smaller than the number of nodes whose distance to $i_{\max}$ ($i_{\min}$) is smaller than $d$. Therefore, it follows that the cardinality of the candidate set in this setting is smaller than $O(K^{1/2}\log K)$.
\vspace{-.4cm}
\subsubsection{Partial Side Information}
\begin{figure}
	\vspace{-.3cm}
	\begin{center}
	\begin{subfigure}[b]{0.25\textwidth}
	\begin{center}
		\includegraphics[width=\textwidth]{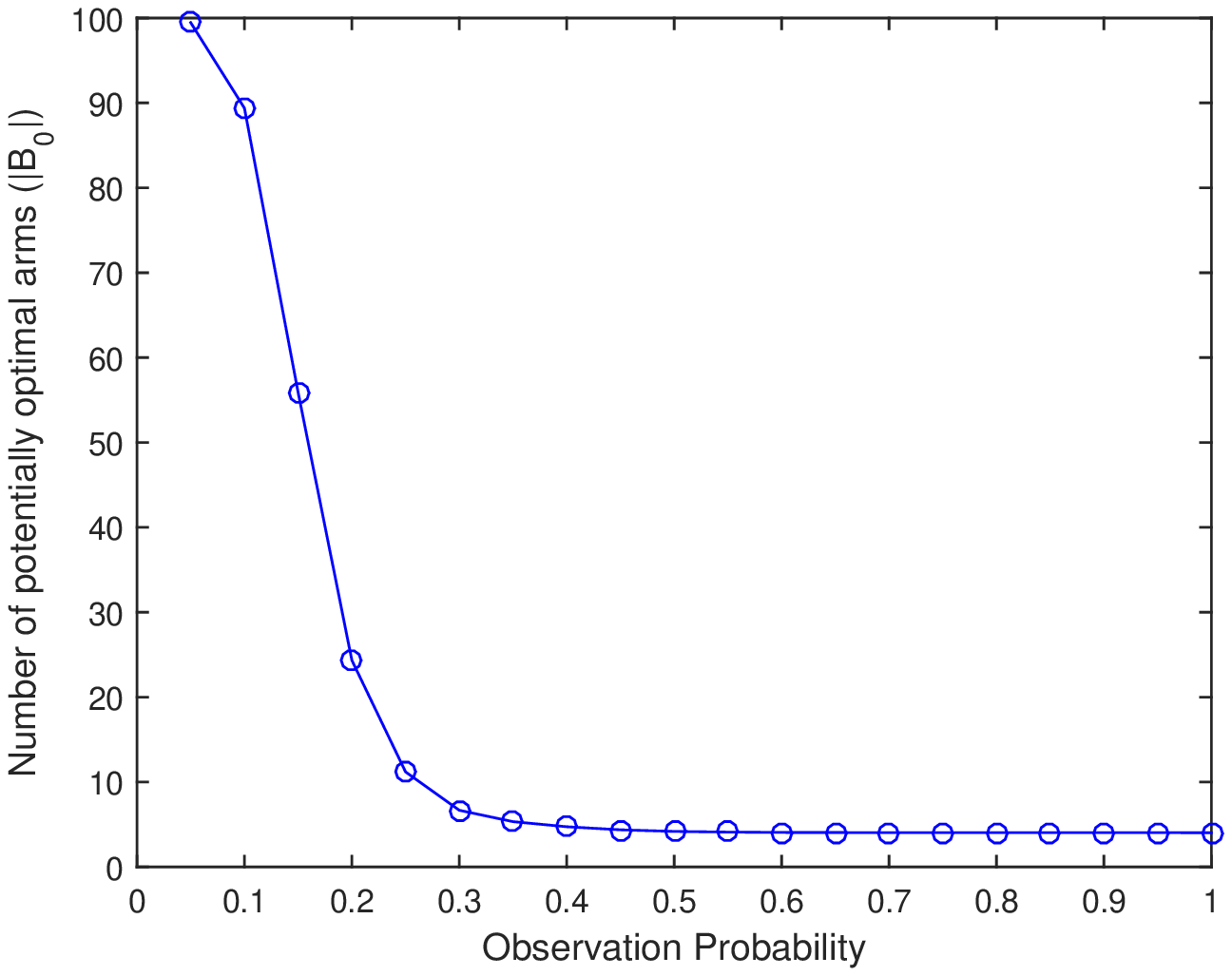}
		\caption{\small{$|\mathcal{B}_0|$ v.s. $p$. }}
		\label{simulation2a}
	\end{center}
	\end{subfigure}
	\hspace{-.4cm}
	\begin{subfigure}[b]{0.25\textwidth}
	\begin{center}
		\includegraphics[width=\textwidth]{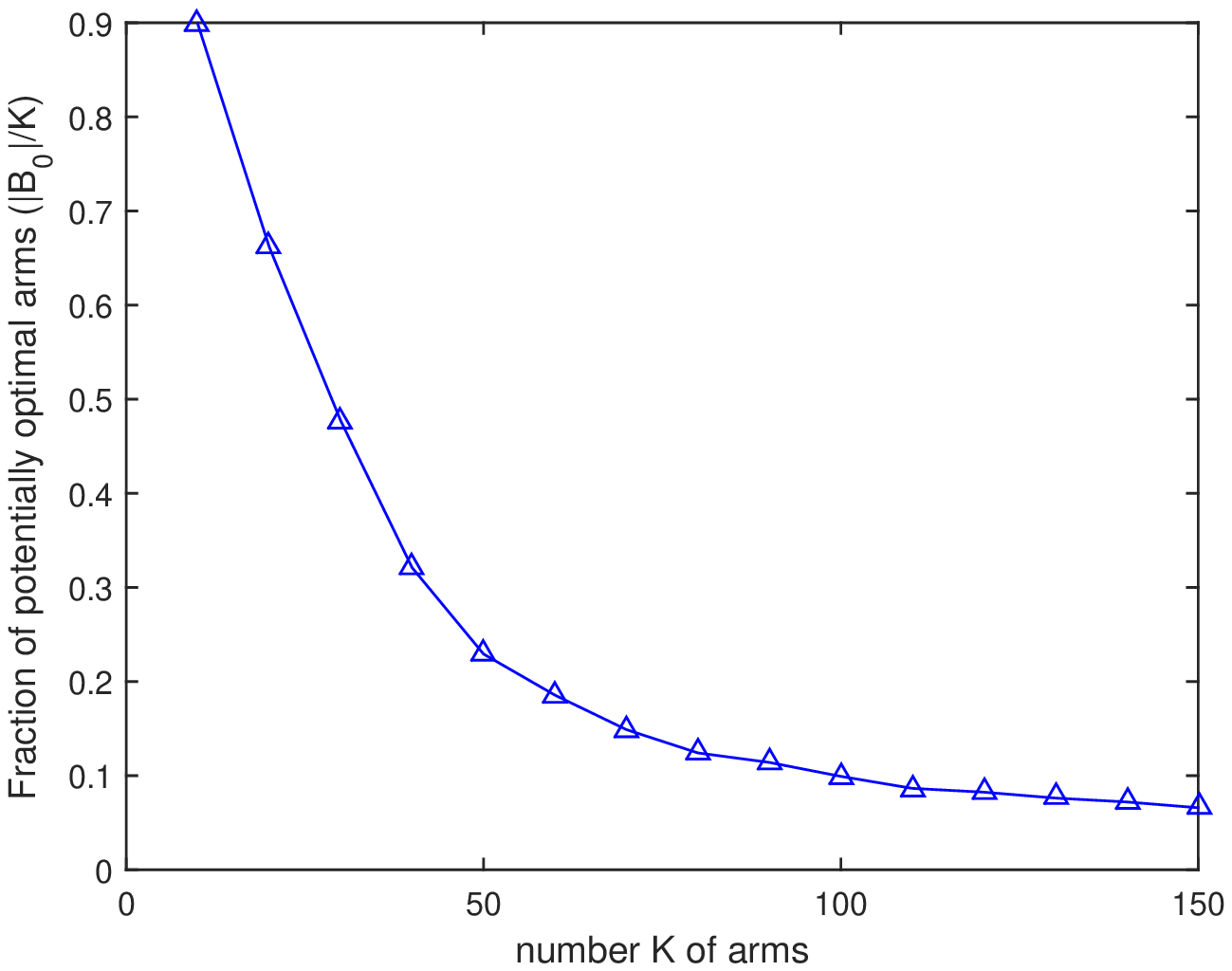}
		\caption{$\frac{|\mathcal{B}_0|}{K}$ v.s. $K$.}
		\label{simulation2b}
	\end{center}
	\end{subfigure}

	\caption{\small{Reduction of the action space with partial side information: comparison between the size of $|\mathcal{B}_0|$ and the number $K$ of arms. In (a), $K=100$, $\epsilon=0.2$, $|\mathcal{B}_0|$ decreases as $p$ grows to $1$. In (b), $\epsilon=0.2$, $p=0.5$, $|\mathcal{B}_0|/K$ decreases as $K$ increases.}}
	\end{center}
	\vspace{-.5cm}
\end{figure}
We use two other experiments to show the reduction of the action space with partial side information. In the first experiment, we fix $K=100$ arms with mean rewards uniformly chosen from $(0,1)$. We choose $\epsilon=0.2$ and obtain the UIG $\mathcal{G}_{\epsilon}^*$. We let $p_S=p_D=p$ vary from $0.1$ to $1$ and for every $p$, we observe the presence and the absence of edges in $\mathcal{G}_{\epsilon}^*$ independently with probability $p$. We apply the offline elimination step of LSDT-PSI on $\mathcal{G}_{\epsilon}$ and compare the size of the output set $\mathcal{B}_0$ with $K$. Note that when $p=1$, $\mathcal{G}^{*}_{\epsilon}$ is fully revealed and we use the offline elimination step of LSDT-CSI to obtain $\mathcal{B}^*$. In the second experiment, we fix $\epsilon=0.2,p_S=p_D=0.5$ and let $K$ increase from $10$ to $150$. We generate arms and side information graphs in the same way as in the first experiment and show how $|\mathcal{B}_0|/K$ varies as $K$ increases. The results of the two experiments are shown in Figs. \ref{simulation2a} and \ref{simulation2b} .

It can be seen from Fig. \ref{simulation2a} that as $p$ increases, $|\mathcal{B}_0|$ decreases to $|\mathcal{B}^*|$. Besides, when $p>0.5$, the performance of the offline elimination step of LSDT-PSI is as good as that of LSDT-CSI, which is optimal, i.e. only arms in $\mathcal{B}^*$ remain. Moreover, in Fig. \ref{simulation2b}, we see that $|\mathcal{B}_0|/K$ decreases as $K$ increases which indicates a diminishing cardinality of the reduced action space in terms of $K$.
\vspace{-.4cm}
\subsection{Regret on Randomly Generated Graphs}
\subsubsection{Complete Side Information}
We compare LSDT-CSI with existing algorithms on a set of randomly generated arms. We obtain the UIG $\mathcal{G}_{\epsilon}^*$ on $K=100$ nodes with means uniformly chosen from $[0.1,1]$ and $\epsilon = 0.1$. Every time an arm $i$ is played, a random reward is drawn independently from a Gaussian distribution with mean $\mu_i$ and variance $1$. We let $T$ vary from $10$ to $1000$ and compare the regret of LSDT-CSI and 4 baseline algorithms:
\begin{enumerate}[(i)]
	\item UCB1: classic UCB policy proposed in \cite{auer2002finite} assuming no relation among arms.
	\item TS: classic Thompson Sampling algorithm proposed in \cite{thompson1933likelihood} assuming Beta prior and Bernoulli likelihood on the reward model.
	\item CKL-UCB: proposed in \cite{magureanu2014lipschitz} for Lipschitz bandit exploiting only similarity relations.
	\item OSUB: proposed in \cite{combes2014unimodal} for unimodal bandits. Note that if the UIG $\mathcal{G}_{\epsilon}^*$ is connected, it satisfies the unimodal structure: for every sub-optimal arm $i$, there exists a path $P=(i_1=i,i_2,...,i_n=i_{\max})$ such that for every $t=1,...,n-1$, $\mu_{i_t}\le\mu_{i_{t+1}}$.
	\item OSSB: proposed in \cite{combes2017minimal} for general structured bandits. At each time, OSSB estimates the minimum number of times that every arm has to be played by solving a LP.
\end{enumerate}

The results shown in Fig. \ref{simulation3a} indicate that LSDT-CSI outperforms the baseline algorithms. In particular, when $T<K$, LSDT-CSI has already started to exploit the optimal arm while the other algorithms are still exploring. We~also~compare~LSDT-CSI~with~an~intuitive~algorithm~applying~UCB1~on~the~candidate~set~in~Fig.~\ref{simulation3b}. With the same setup, we see~performance~gain~due~to~the~online~step.

\begin{figure}
	\begin{center}
	\vspace{-.3cm}
	\begin{subfigure}[b]{0.25\textwidth}
	\begin{center}
		\includegraphics[width=1\textwidth]{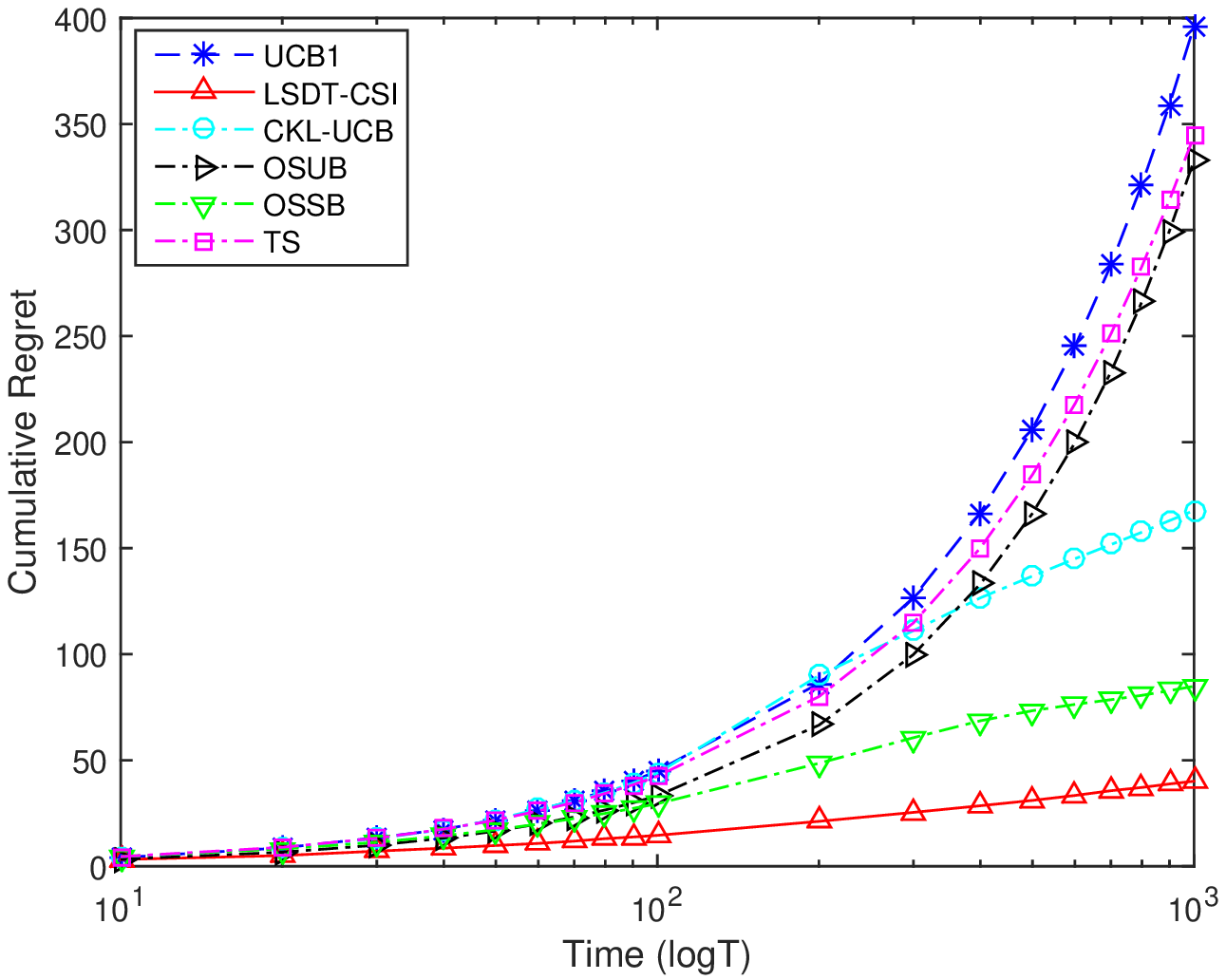}
		\caption{}
		\label{simulation3a}
	\end{center}
	\end{subfigure}
	\hspace{-.5cm}
	\begin{subfigure}[b]{0.25\textwidth}
	\begin{center}
		\includegraphics[width=1\textwidth]{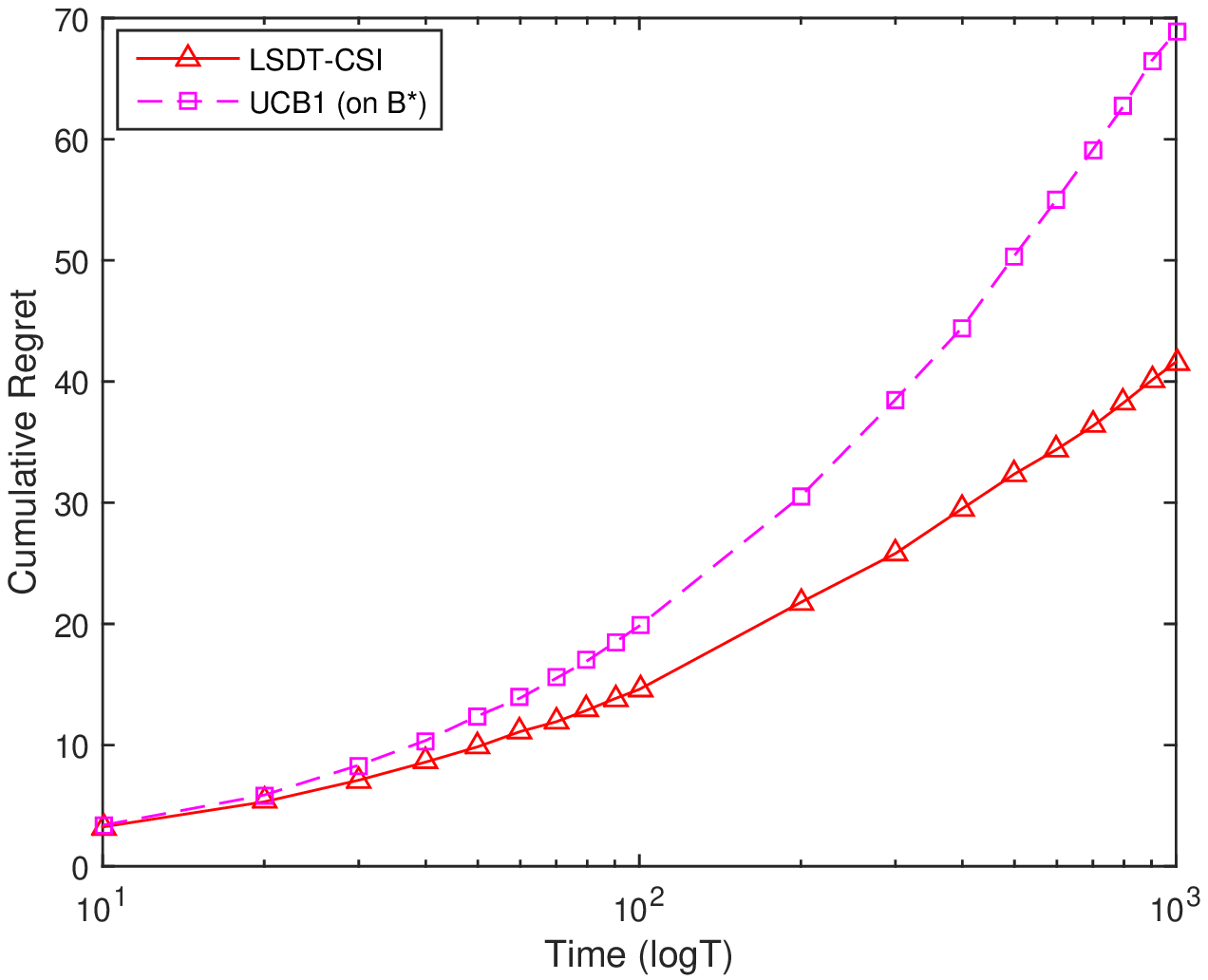}
		\caption{}
		\label{simulation3b}
	\end{center}
	\end{subfigure}
	\vspace{-.2cm}
	\caption{\small{Regret on randomly generated arms with complete side information: $K=100,\epsilon=0.1$. (a) Comparison with existing algorithms. (b) Comparison with an intuitive algorithm.}}
	\end{center}
	\vspace{-.6cm}
\end{figure}

We also evaluate the time complexity of the learning policies. Due to the page limit, we summarize the running times of LSDT-CSI and the other baseline algorithms in Table \ref{time_complete} in Appendix \ref{time}. It is shown that LSDT-CSI has a relatively low computation cost in contrast to algorithms with comparable performance, i.e., CKL-UCB and OSSB. 

\vspace{-.2cm}
\subsubsection{Partial Side Information}

We compare LSDT-PSI with existing algorithms. We obtain the UIG $\mathcal{G}^*_{\epsilon}$ on $K=100$ arms with means uniformly chosen from $[0.1,0.9]$ and $\epsilon=0.1$. We let $p_S=p_D=p=0.5$ and get the partially observed UIG $\mathcal{G}_{\epsilon}$ based on Assumption \ref{observation}. The random rewards for every arm $i$ are independently generated from a Bernoulli distribution with mean $\mu_i$. We consider $T\in[100,1000]$.

Given that finding the candidate set is NP-complete, the OSSB policy is not applicable since the LP is unspecified. Besides, OSUB is also inapplicable since $\mathcal{G}_{\epsilon}$ is not unimodal in general. Therefore, we only compare LSDT-PSI with three baseline algorithms: UCB1, TS and CKL-UCB. In LSDT-PSI, we choose the input parameter $\lambda=1/8$. Note that the choice of $\lambda$ does not affect the theoretical upper bound on regret. However, in practice, it is better to use a smaller $\lambda$ to avoid excessive plays of suboptimal arms. The simulation results shown in Fig \ref{simulation4a} indicates that LSDT-PSI outperforms the other two algorithms. Besides, similar to the case of complete side information, we compare LSDT-PSI with a heuristic algorithm applying UCB1 on $\mathcal{B}_0$ and a similar performance gain is observed in Fig. \ref{simulation4b}. Moreover, the computational efficiency of LSDT-PSI is also verified in Table \ref{time_partial} in Appendix \ref{time}.

\begin{figure}[t]
	\vspace{-.3cm}
	\begin{center}
	\begin{subfigure}[b]{0.25\textwidth}
	\begin{center}
		\includegraphics[width=\textwidth]{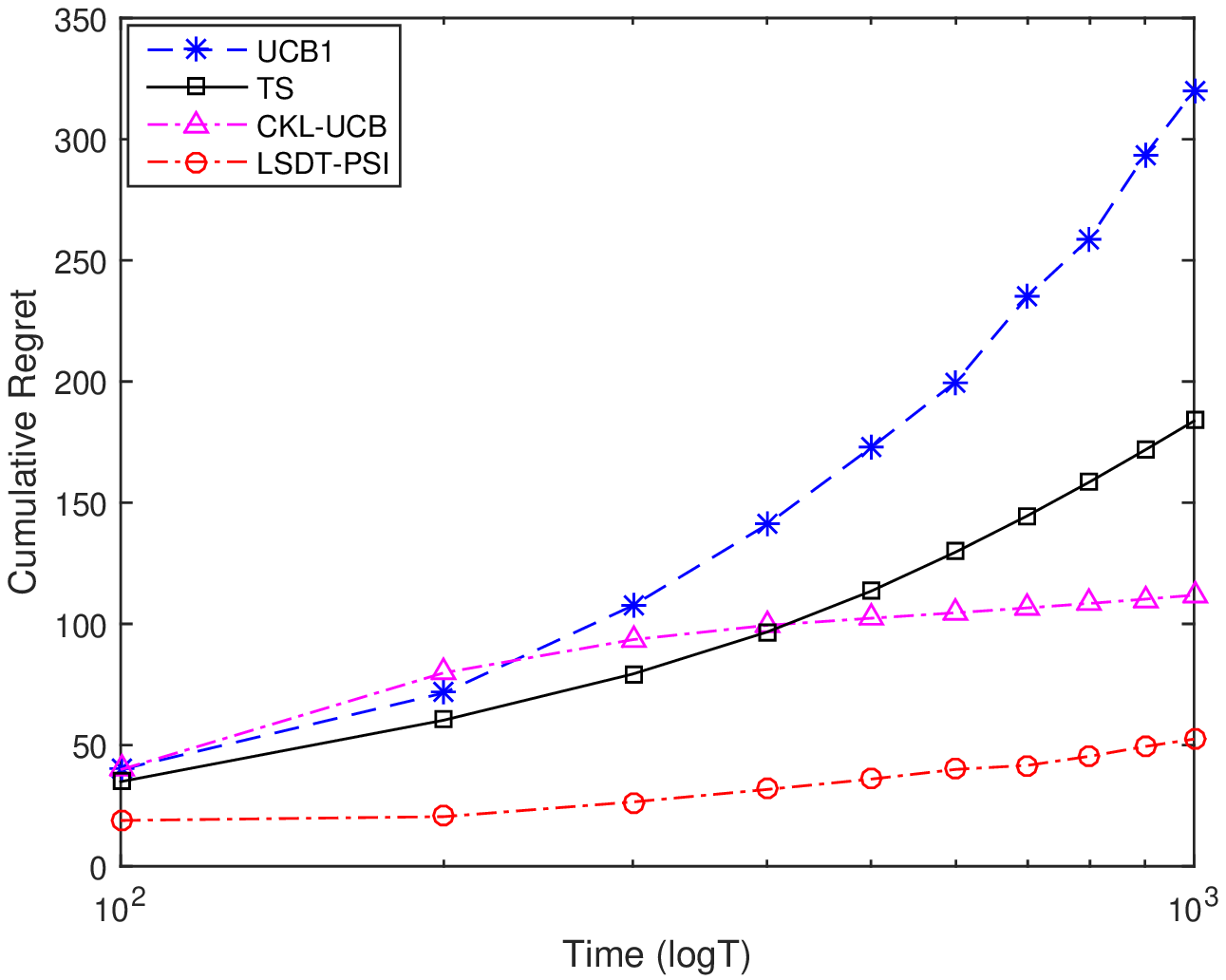}
		\caption{}
		\label{simulation4a}
	\end{center}
	\end{subfigure}
	\hspace{-.5cm}
	\begin{subfigure}[b]{0.25\textwidth}
	\begin{center}	
		\includegraphics[width=\textwidth]{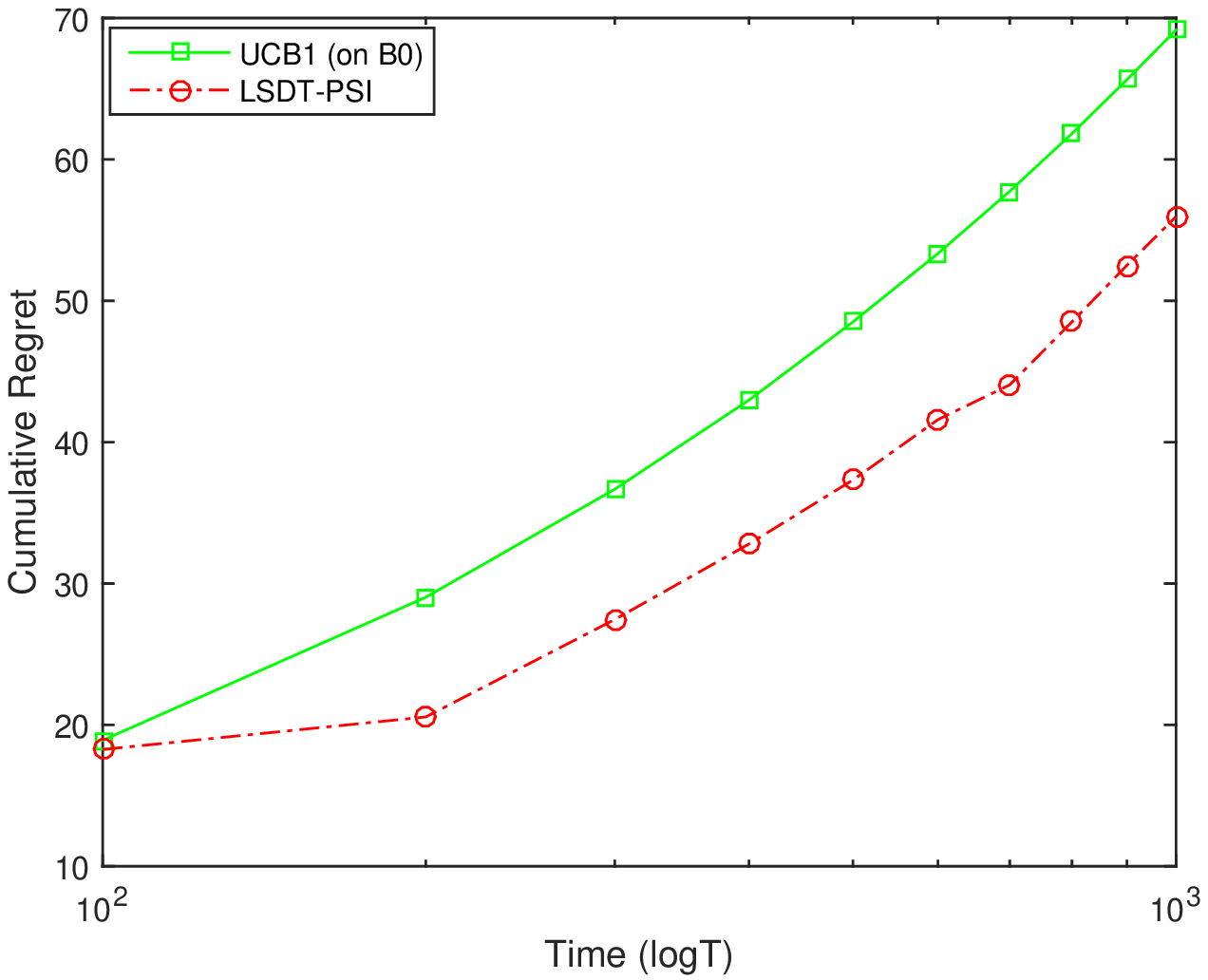}
		\caption{}
		\label{simulation4b}
	\end{center}
	\end{subfigure}
	\vspace{-.2cm}
	\caption{\small{Regret on randomly generated arms with partial side information: $K=100,\epsilon=0.1, p=0.5$. (a) Comparison with existing algorithms. (b) Comparison with a heuristic algorithm.}}
	\end{center}
	\vspace{-.6cm}
\end{figure}
\vspace{-.35cm}
\subsection{Online Recommendation Systems}

In this subsection, we apply LSDT-PSI to a problem in online recommendation systems. We test our policy on a dataset from Jester, an online joke recommendation and rating system \cite{goldberg2001eigentaste}, consisting of 100 jokes and 25K users and every joke has been rated by at least 34\% of the entire population.\footnote{Available on http://eigentaste.berkeley.edu/dataset/.}. Ratings are real values between $-10.00$ and $10.00$. In the experiment, we recommend a joke (modeled as an arm) to a new user at each time and observe the rating, which corresponds to playing an arm and receiving the reward. Note that although different users have different preference towards items, every item exhibits certain internal quality that is represented by the mean reward, i.e., the average rating from all users. The variations of ratings from different users correspond to the randomness of rewards. Notice that the algorithms we propose work for any reward distribution as long as it is sub-Gaussian, Jester is a suitable dataset for the purpose of evaluating the performance of our algorithms since any distribution with bounded support is sub-Gaussian. In accordance with the assumptions of the policy, all ratings are normalized to $[0,1]$.

To test our policy using side information, we partition the dataset into a training set ($5\%$ or $10\%$ of the users) and a test set (20K users). We obtain the partially revealed UIG from the training set as follows: we estimate the distance between two jokes $i,j$ by calculating the difference between their average ratings from users in the training set who have rated both jokes. We define a confidence parameter $\alpha>0$. If the distance between $(i,j)$ is larger than $(1+\alpha)\epsilon$, we add $(i,j)$ to $\mathcal{E}_{\epsilon}^D$. Otherwise if the distance is smaller than $(1-\alpha)\epsilon$, we add $(i,j)$ to $\mathcal{E}_{\epsilon}^S$. It is clear that there exist certain pairs of arms whose relations are unknown. We let $\alpha=0.2$ if the size of the training set is $2\%$ of the entire dataset or $\alpha = 0.1$ if the size of the training set is $5\%$. Note that as the size of the training set increases, the estimation of distances between jokes becomes more accurate and thus, the confidence parameter can be smaller. As a consequence, the number of joke pairs whose relations are known increases. For the hyper-parameter $\epsilon$, we use an iterative approach to find the best $\epsilon$ that minimize the size of $\mathcal{B}_0$, i.e., the set of arms that need to be explored. Intuitively, as $\epsilon$ increases, $|\mathcal{B}_0|$ first decreases since the side information graph becomes more connected and more similarity relations can be observed. Therefore, the probability of eliminating sub-optimal arms by the offine step becomes higher. When $\epsilon$ is large, the graph approaches a complete graph and less dissimilarity relations are observed. As a consequence, the probability of eliminating sub-optimal arms decreases and thus $|\mathcal{B}_0|$ increases. A similar tendency of variation can be observed on the overall regret performance of the learning policy. Based on this, the iterative approach starts from a small $\epsilon(0)$ (i.e., $0.01$) at time $t=0$ and find $\mathcal{B}_0(0)$. It keeps doubling the value of $\epsilon$ at each step until time $t$ when $|\mathcal{B}_0(t)|>|\mathcal{B}_0(t-1)|$. Then a binary search method is applied to find the best $\epsilon^*$ (with resolution $0.01$, i.e., the minimum increment of $\epsilon$) between $\epsilon(t-1)$ and $\epsilon(t)$ that achieves the minimum $|\mathcal{B}_0|$.
\begin{figure}[t]
	\vspace{-.3cm}
	\begin{center}
	\includegraphics[width=0.42\textwidth]{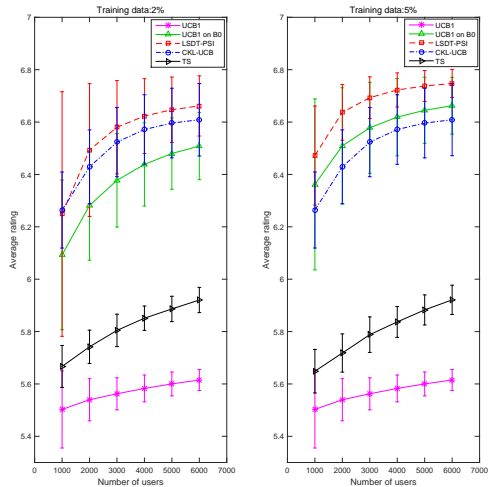}
	\vspace{-.5cm}
	\caption{\small{Joke recommendation on Jester.}}
	\vspace{-.8cm}
	\label{simulation5}
	\end{center}
\end{figure}

We use an unbiased offline evaluation method introduced in \cite{li2010contextual} and \cite{li2011unbiased} to evaluate algorithms including LSDT-PSI, UCB1, TS, CKL-UCB and UCB1 on $\mathcal{B}_0$, on the test set. Fig. \ref{simulation5} shows the average rating per user with confidence intervals (scaled back to $[0,10.00]$) of every policy. Note that CKL-UCB needs to estimate the KL-divergence between two distributions. Since the distribution type in the real dataset is unknown, we can only use $\Delta^2$ to approximate the KL-divergence where $\Delta$ is the distance between the average ratings. For LSDT-PSI, we choose the input parameter $\lambda=1/32$. Simulation results in Fig. \ref{simulation5} show that LSDT-PSI has the best performance with relatively small variations. Besides, the effectiveness of the adaptive approach on selecting the hyper-parameter $\epsilon$ is verified. Moreover, it can be observed that as the size of the training data increases, the performance of LSDT-PSI and UCB1 on $\mathcal{B}_0$ get improved since more side information is available.
\vspace{-.3cm}
\section{Conclusion}\label{sec:conclusion}

We studied a stochastic multi-armed bandit problem with side information on the similarity and dissimilarity across arms. The similarity-dissimilarity structure is represented by a UIG where every node represents an arm and the presence (absence) of an edge between two nodes represents similarity (dissimilarity) of their mean rewards. We considered two settings of complete and partial side information based on whether the UIG is fully revealed and proposed a general two-step learning structure: LSDT consisting of an offline reduction of the action space to the candidate set and online aggregation of observations from similar arms. In the case of complete side information, we showed that the candidate set can be identified by a BFS-based algorithm in polynomial time and the proposed learning policy LSDT-CSI achieves order optimal regret in terms of both the size of the action space and the time horizon. In the case of partial side information, we showed that finding the candidate set is NP-complete and proposed an approximation algorithm to reduce the action space with polynomial time complexity. We proved that under certain probabilistic assumptions on the side information, the approximation algorithm achieves the same performance as that in the case of complete side information and the proposed learning policy LSDT-PSI is order optimal.

For future directions, it will be interesting to consider different probabilistic models of the side information. It is reasonable to assume that if the difference between the mean rewards of two arms is smaller (larger), their similarity (dissimilarity) relation is more likely to be revealed. In addition, it is worth investigating a case with spurious relations across arms (e.g., the side information indicates that two arms are close in their mean rewards but actually are not). Besides, while we assumed a single hyper-parameter $\epsilon$ to quantize the similarity and dissimilarity relations across arms, a more general setting is to consider distinct parameters $\epsilon_{i,j}$ to characterize different similarity and dissimilarity levels between different arm pairs.
\vspace{-.3cm}
\bibliographystyle{IEEEtran}
\bibliography{MABonPartialUIG}


%

\clearpage
\onecolumn
\appendices
\section{Additional Numerical Results}
\subsection{LSDT with Thompson Sampling Techniques}\label{TSsimulation}
As discussed in Sec. 5.3, we use numerical examples to show the performance of applying TS techniques in the two-step learning structure: LSDT. We adopt the same experiment setup with that in the simulation of regret analysis on randomly generated graphs with complete side information (Sec. 6.2.1) and compare LSDT-TS (CSI) (applying TS in LSDT learning structure in the case of complete side information, which is introduced in Sec. 5.3) with classic TS that ignores side information. The results in Fig. \ref{simulationts1} verify the advantage of our two-step learning structure, which fully exploits the topological structure of the side information graph. Besides, we also compare LSDT-TS (CSI) with another heuristic algorithm, which simply applies classic TS on the reduced action space $\mathcal{B}^*$ without aggregation observations from similar arms in the second step of online learning. The results in Fig. \ref{simulationts2} further indicates that the online aggregations step in the two-step learning structure improves the performance.

\begin{figure}[h!]
	\begin{center}
	\begin{subfigure}[b]{0.5\textwidth}
	\begin{center}
		\includegraphics[width=0.6\textwidth]{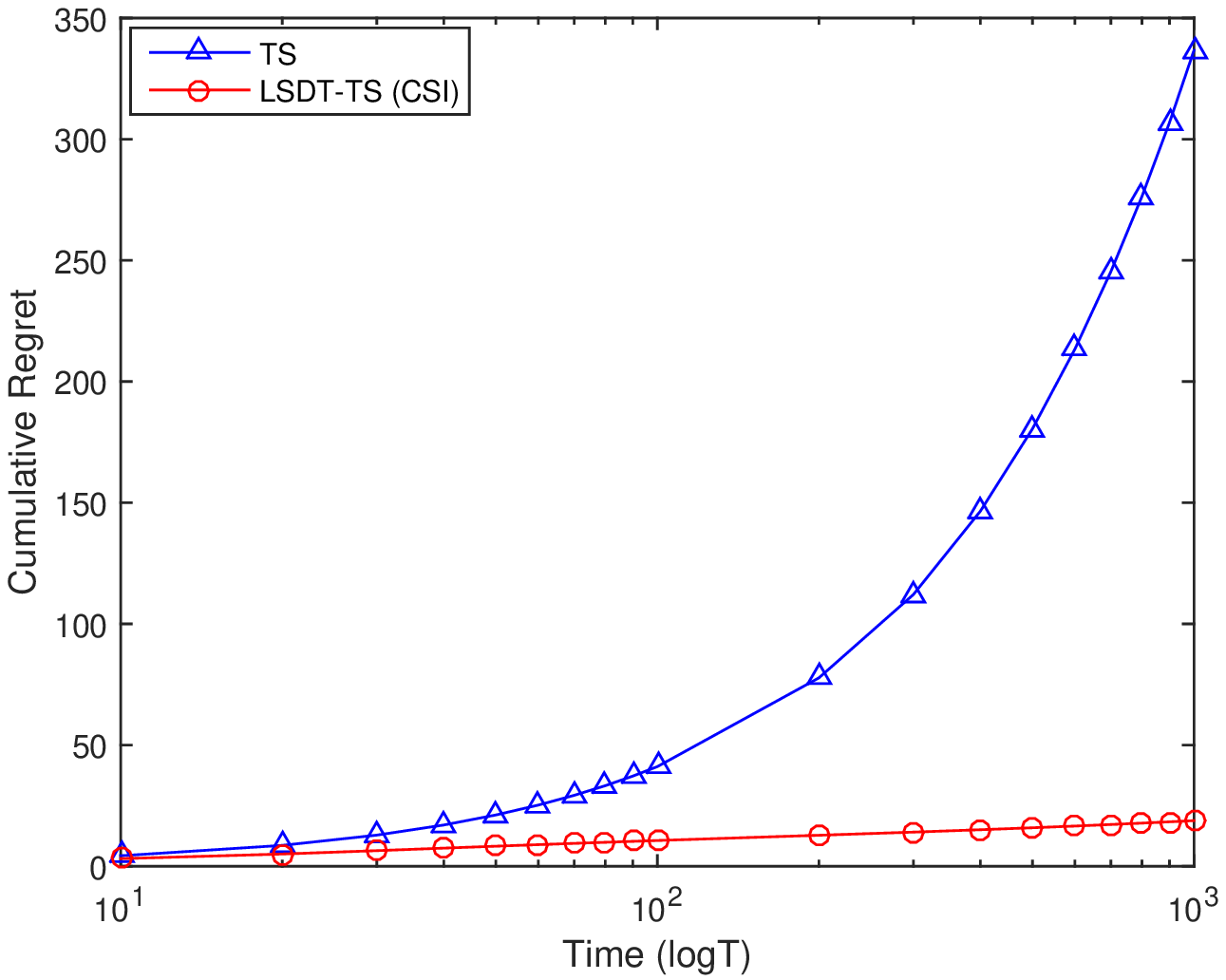}
		\caption{\small{Comparison with classic TS.}}
		\label{simulationts1}
	\end{center}
	\end{subfigure}
	\hspace{-1cm}
	\begin{subfigure}[b]{0.5\textwidth}
	\begin{center}	
		\includegraphics[width=0.6\textwidth]{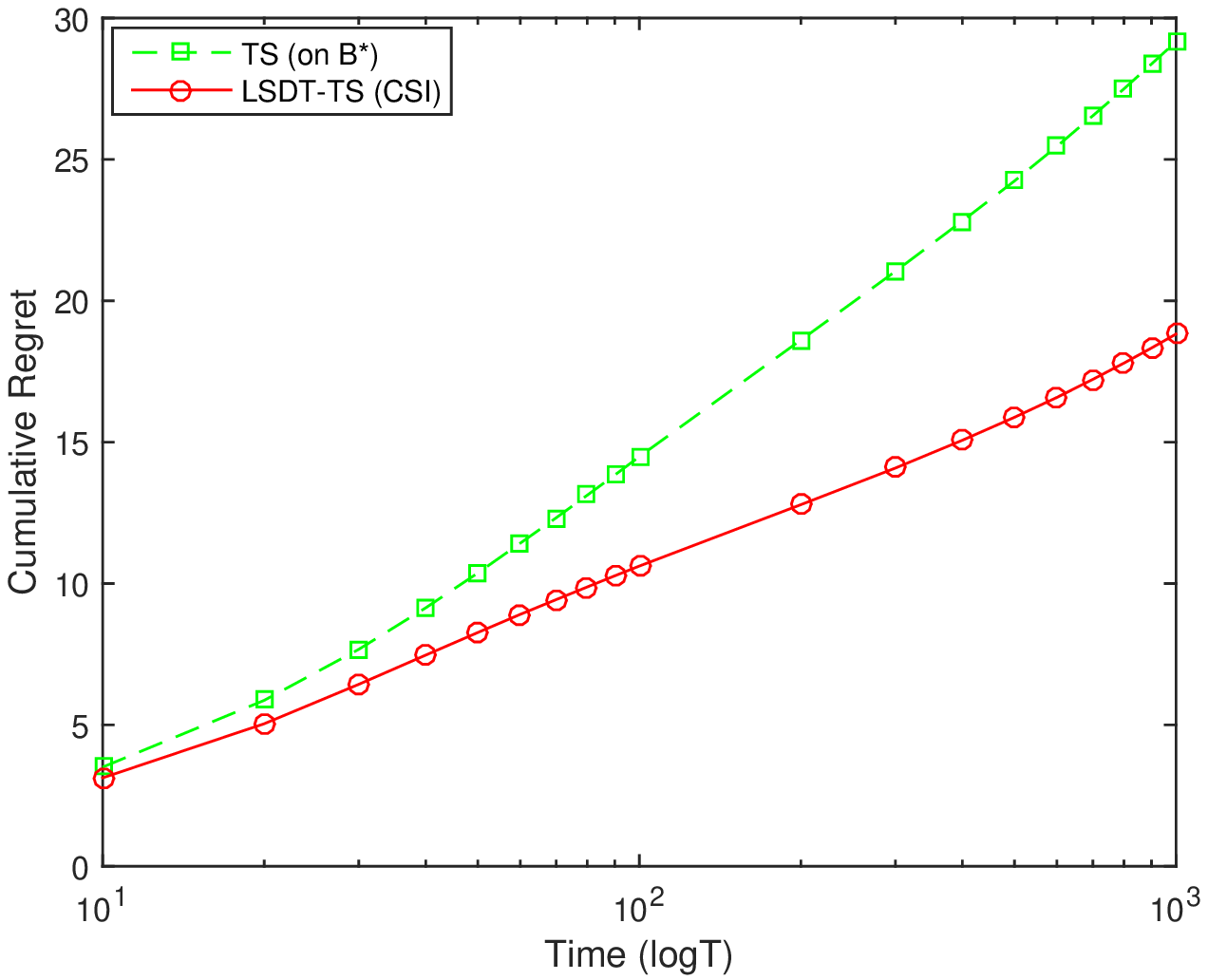}
		\caption{\small{Comparison with a heuristic algorithm.}}
		\label{simulationts2}
	\end{center}
	\end{subfigure}
	
	\caption{\small{Regret on randomly generated arms with complete side information: $K=100,\epsilon=0.1$.}}
	\end{center}
	\vspace{-.4cm}
\end{figure}
In the case of partial side information, we conduct an experiment similar with that in Sec. 6.2.2 to evaluate the performance of LSDT-TS (PSI), which applies TS in LSDT learning structure in the case of partial side information as discussed in Sec. 5.3. We compare LSDT-TS (PSI) with classic TS ignoring side information and another heuristic algorithm applying TS on the reduced action space without online aggregation. The results are shown in Fig. \ref{simulationtspartial} and the performance gain through both offline and online steps of LSDT is verified.

\begin{figure}[h!]
	\begin{center}
	\begin{subfigure}[b]{0.5\textwidth}
	\begin{center}
		\includegraphics[width=0.6\textwidth]{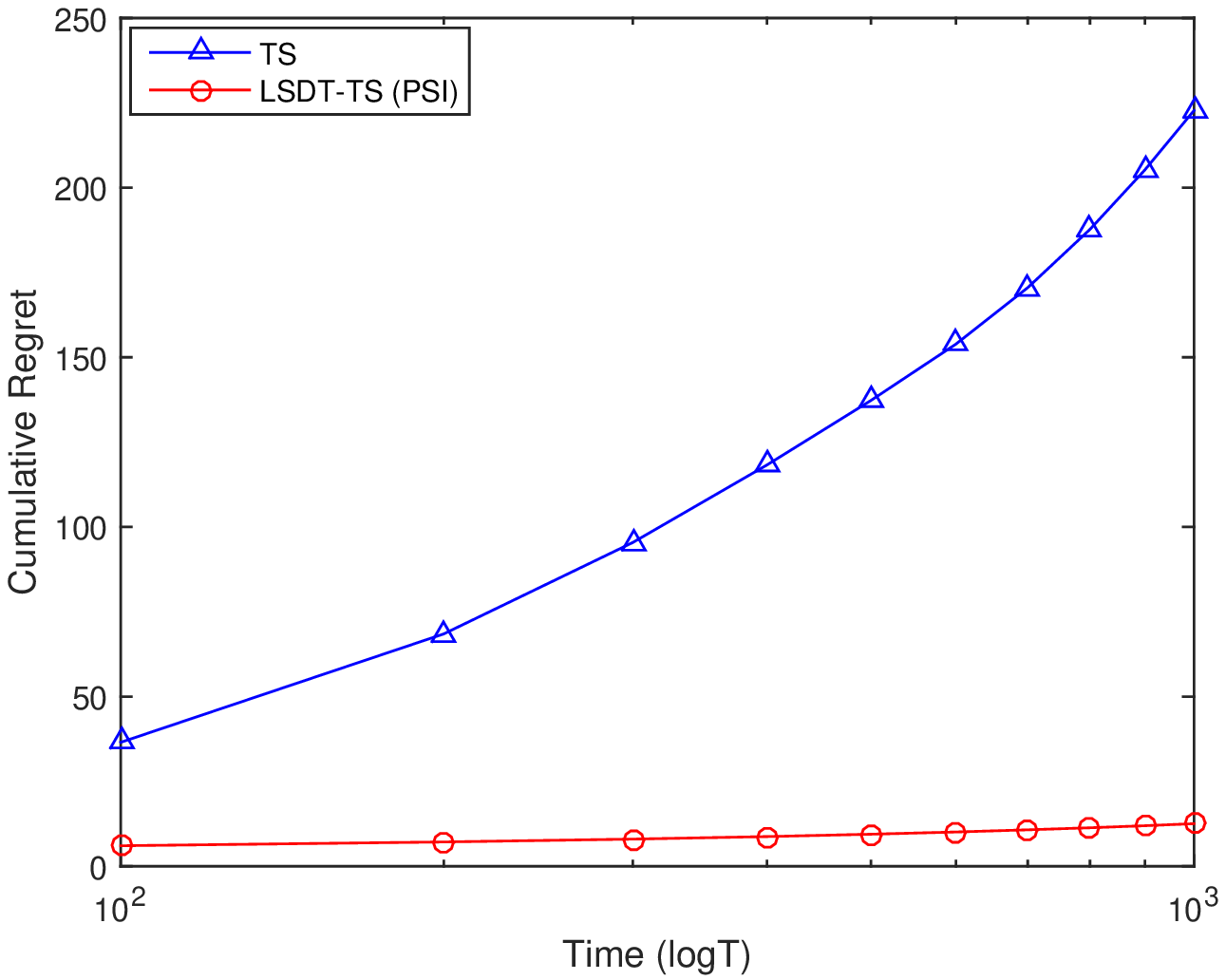}
		\caption{\small{Comparison with classic TS.}}

	\end{center}
	\end{subfigure}
	\hspace{-1cm}
	\begin{subfigure}[b]{0.5\textwidth}
	\begin{center}	
		\includegraphics[width=0.6\textwidth]{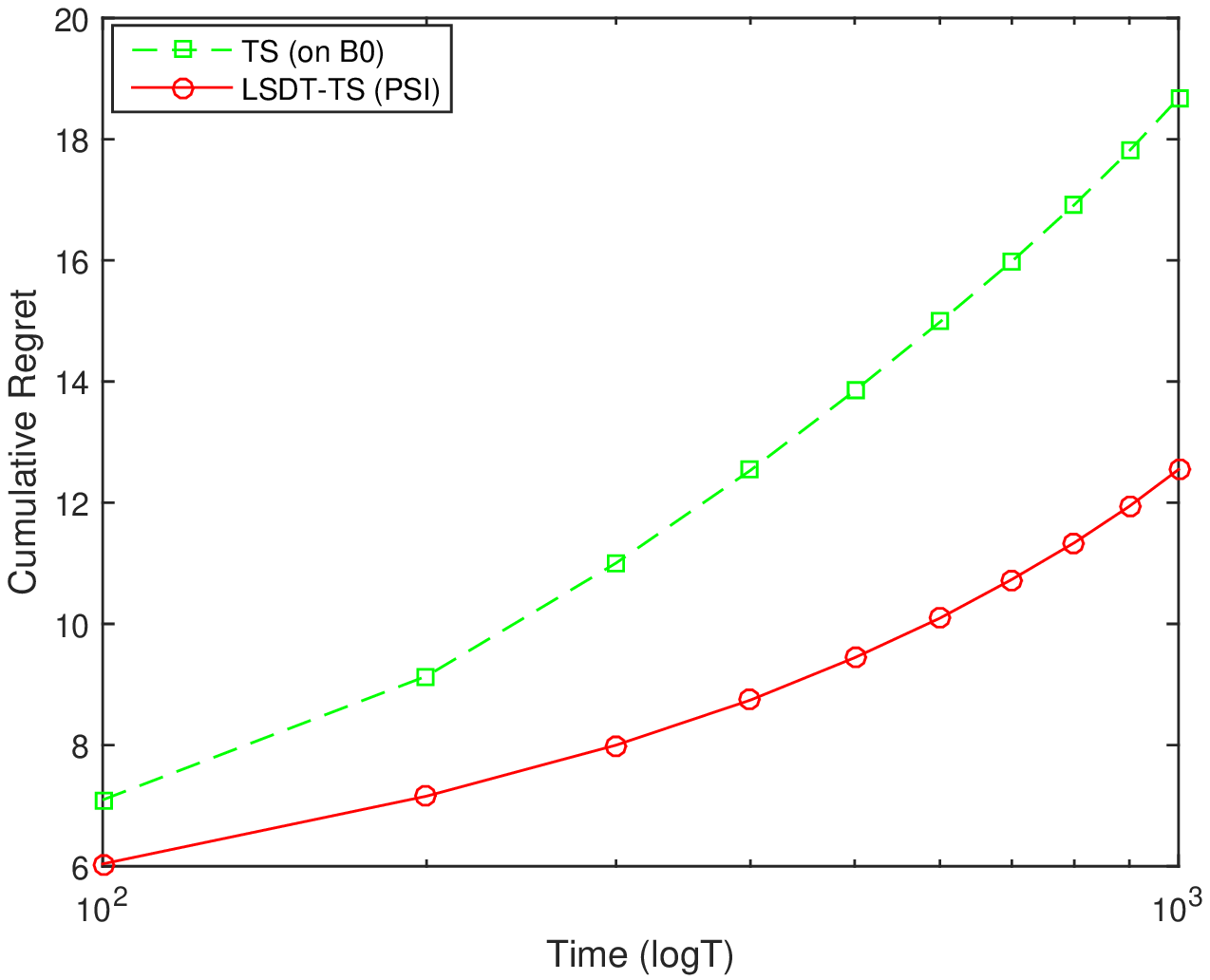}
		\caption{\small{Comparison with a heuristic algorithm.}}

	\end{center}
	\end{subfigure}
	
	\caption{\small{Regret on randomly generated arms with partial side information: $K=200,\epsilon=0.1,p=0.5$.}}\label{simulationtspartial}
	\end{center}
	
	\vspace{-.4cm}
\end{figure}
\subsection{Comparison of Running Times}\label{time}
We compare the running time of LSDT-CSI as well as the baseline algorithms in Table \ref{time_complete} for the case of complete side information. It is not difficult to see that LSDT-CSI has a relatively low computation complexity in contrast to algorithms with comparable performance, i.e., CKL-UCB and OSSB. Note that CKL-UCB and OSSB are time consuming since they have to solve an optimization problem at each time step. Besides, it can be seen that the time complexity of the offline reduction step is not too high to be applied.
\begin{table*}[h!]
\begin{center}
\begin{tabular}{|c|c|c|c|c|c|c|c|c|}
\hline
Algorithm  & UCB1 & TS & CKL-UCB & OSUB  & OSSB    & LSDT-CSI (offline) & LSDT-CSI (online) & UCB1 on $\mathcal{B}^*$ \\
\hline
Running Time (ms) & 9.7 & 37.6 & 849.5   & 880.3 & $3.3\times 10^5$ & 14.1               & 18.6              & 9.1\\      
\hline
\end{tabular}
\end{center}

\caption{Running times in the case of complete side information.}\label{time_complete}
\end{table*}

For the case of partial side information, we summarize the running times of LSDT-PSI and the other baseline algorithms in Table \ref{time_partial}. Note that the running times of UCB1 and TS are smaller than LSDT-PSI since they ignore the similarity-dissimilarity relations across arms and have worse performance. When compared with CKL-UCB, which achieves a comparable performance by exploiting the similarity relations, LSDT-PSI has a smaller computation complexity.
\begin{table*}[h!]
\begin{center}
\begin{tabular}{|c|c|c|c|c|c|c|}
\hline
Algorithm  & UCB1 & TS & CKL-UCB &  LSDT-PSI (offline) & LSDT-CSI (online) & UCB1 on $\mathcal{B}_0$ \\
\hline
Running Time (ms) & 10.3  & 38.3   & 354.2 & 12.6 & 161.4               & 10.1            \\      
\hline
\end{tabular}
\end{center}

\caption{Running times in the case of partial side information.}\label{time_partial}
\vspace{-.3cm}
\end{table*}

\section{Proof of Theorem 1}\label{pfcandidateset}

	We first show that $\mathcal{B}^*_{i_{\max}}\cup\mathcal{B}_{i_{\min}}^*\subseteq \mathcal{B}^*$. Clearly $i_{\max}\in\mathcal{B}^*$. For each $j\in\mathcal{B}^*_{i_{\max}}$, $\mathcal{N}[j]=\mathcal{N}[i_{\max}]$. Thus if we construct a new set of mean rewards $(\mu_1',...,\mu_K')$ where the mean values of $j$ and $i_{\max}$ get switched and the others remain the same, the UIG $\mathcal{G}_{\epsilon}^*$ remains unchanged. Thus, $j\in\mathcal{B}^*$. Similar~result~holds~for $\mathcal{B}^*_{i_{\min}}$.~Therefore $\mathcal{B}^*_{i_{\max}}\cup\mathcal{B}^*_{i_{\min}}\subseteq \mathcal{B}^*$.
	
	Next, we show that $\mathcal{B}^*\subseteq\mathcal{B}^*_{i_{\max}}\cup\mathcal{B}^*_{i_{\min}}$. For each $j\not\in\mathcal{B}^*_{i_{\max}}\cup\mathcal{B}^*_{i_{\min}}$, consider two cases:
	\begin{enumerate}
		\item $j\in\mathcal{N}[i_{\max}]\cup\mathcal{N}[i_{\min}]$: without loss of generality, assume that $j\in\mathcal{N}[i_{\max}]$. Since $j\not\in\mathcal{B}^*_{i_{\max}}$, there exists an arm $k$ such that $k\in\mathcal{N}[j]$ but $k\not\in\mathcal{N}[i_{\max}]$. Now suppose there exists an assignment of mean rewards $(\mu_1',...,\mu_K')$ conforming to $\mathcal{G}_{\epsilon}^*$ such that arm $j$ is optimal, then $\mu_{k}',\mu_{i_{\max}}'\in (\mu_j'-\epsilon,\mu_j']$ and thus, arm $k$ and $i_{\max}$ are neighbors. This contradicts the assumption that $k\not\in\mathcal{N}[i_{\max}]$. Hence, there doesn't exists a set of mean rewards conforming to $\mathcal{G}_{\epsilon}^*$ where $j$ is optimal. Thus $j\not\in\mathcal{B}^*$. Similar result holds for the case when $j\in\mathcal{N}[i_{\min}]$.
		\item $j\not\in\mathcal{N}[i_{\max}]\cup\mathcal{N}[i_{\min}]$: define 
		\begin{align}k_1=\argmin_{k\not\in\mathcal{N}[j],\mu_k>\mu_j}\mu_k,\\
		k_2=\argmax_{k\not\in\mathcal{N}[j],\mu_k<\mu_j}\mu_k.
		\end{align} 
		Notice that $k_1,k_2$ are not neighbors. However, since the component is connected, $k_1,k_2$ must connect with arms in $\mathcal{N}[j]$. Now suppose there exists an assignment of mean rewards $(\mu_1',...,\mu_K')$ conforming to $\mathcal{G}_{\epsilon}^*$ such that $j$ is optimal, then $\mu_{k_1}',\mu_{k_2}'\in(\mu_j'-2\epsilon,\mu_j'-\epsilon]$. This contradicts the assumption that $k_1,k_2$ are not neighbors. Thus, $j\not\in\mathcal{B}^*$.
	\end{enumerate}
	Therefore, we have that if $j\not\in\mathcal{B^*}_{i_{\max}}\cup\mathcal{B}^*_{i_{\min}}$, then $j\not\in\mathcal{B}^*$. This implies that $\mathcal{B}^*\subseteq\mathcal{B}^*_{i_{\max}}\cup\mathcal{B}^*_{i_{\min}}$. In summary, 
	\begin{align}\mathcal{B}^*=\mathcal{B}^*_{i_{\max}}\cup\mathcal{B}^*_{i_{\min}}.
	\end{align}

\section{Proof of Theorem 2}\label{pfthmupperbound}

When $\mathcal{G}^*_{\epsilon}$ is connected, $\mathcal{B}^*=\mathcal{B}^*_{i_{\min}}\cup\mathcal{B}^*_{i_{\max}}$ where $\mathcal{B}^*_{i_{\min}}$ and $\mathcal{B}^*_{i_{\max}}$ are disjoint if $\mathcal{G}^*_{\epsilon}$ is not complete. We upper bound the number of times that arms in $\mathcal{B}^*_{i_{\min}}$ have been played up to time $T$. Let $\tau_{\mathcal{B}^*_{i_{\min}}}(T)=\sum_{j\in\mathcal{B}^*_{i_{\min}}}\tau_j(T)$, $\tau_{\mathcal{B}^*_{i_{\max}}}(T)=\sum_{j\in\mathcal{B}^*_{i_{\max}}}\tau_j(T)$. Let $c_{t,s}=\sqrt{(8\log t)/s}$. Let $\pi_t$ be the arm selected at time $t$ and $\mathbb{I}\{\cdot\}$ be the indicator function. Let $\ell>|\mathcal{B}^*_{i_{\min}}|$ be an arbitrary integer, then with $H_i(t)$ defined in (\ref{Hindex}),

\begin{equation}\label{pfthm2eq1}
	\begin{aligned}	
	&\mathbb{E}[\tau_{\mathcal{B}^*_{i_{\min}}}(T)]= \mathbb{E}\left[|\mathcal{B}^*_{i_{\min}}|+\sum_{t=|\mathcal{B}^*|+1}^{T}\mathbb{I}\{\pi_{t}\in\mathcal{B}^*_{i_{\min}}\}\right]\\
	\le& \ell+\sum_{t=|\mathcal{B}^*|+1}^{T}\mathbb{P}\left(\pi_t\in\mathcal{B}^*_{i_{\min}},\tau_{\mathcal{B}^*_{i_{\min}}}(t-1)\ge \ell\right)\\
	\le&\ell +\sum_{t=|\mathcal{B}^*|+1}^{T}\mathbb{P}\Big(H_{i_{\min}}(t-1)\ge H_{i_{\max}}(t-1),\tau_{\mathcal{B}^*_{i_{\min}}}(t-1)\ge \ell\Big)\\	
	\le& \ell+\sum_{t=|\mathcal{B}^*|}^{T-1}\sum_{s=\ell}^{t}\sum_{r=1}^{t}\mathbb{P}\big(H_{i_{\min}}(t)\ge H_{i_{\max}}(t),\tau_{\mathcal{B}^*_{i_{\min}}}(t)=s,\tau_{\mathcal{B}^*_{i_{\max}}}(t)=r\big).\\
	\end{aligned}
\end{equation}
To upper bound each term on the RHS of the last inequality in (\ref{pfthm2eq1}), we consider
\begin{equation}\label{obj}
	\begin{aligned}
	&\mathbb{P}\Bigg(\frac{\sum_{j\in\mathcal{B}^*_{i_{\min}}}\tau_j(t)\bar{x}_{j}(t)}{s}+c_{t,s}\ge \frac{\sum_{j\in\mathcal{B}^*_{i_{\max}}}\tau_j(t)\bar{x}_{j}(t)}{r}+c_{t,r}\Bigg)\\
	\le&\mathbb{P}\Bigg(\frac{\sum_{j\in\mathcal{B}^*_{i_{\min}}}\tau_j(t)\bar{x}_{j}(t)}{s}\ge\frac{\sum_{j\in\mathcal{B}^*_{i_{\min}}}\tau_j(t)\mu_j}{s}+c_{t,s}\Bigg)\\
	&+\mathbb{P}\Bigg(\frac{\sum_{j\in\mathcal{B}^*_{i_{\max}}}\tau_j(t)\bar{x}_{j}(t)}{r}\le \frac{\sum_{j\in\mathcal{B}^*_{i_{\max}}}\tau_j(t)\mu_j}{r}-c_{t,r}\Bigg)\\
	&+\mathbb{P}\Bigg(\frac{\sum_{j\in\mathcal{B}^*_{i_{\max}}}\tau_j(t)\mu_j}{r}<\frac{\sum_{j\in\mathcal{B}^*_{i_{\min}}}\tau_j(t)\mu_j}{s}+2c_{t,s}\Bigg),
	\end{aligned}
\end{equation}
where $\tau_{\mathcal{B}^*_{i_{\min}}}(t)=s,\tau_{\mathcal{B}^*_{i_{\max}}}(t)=r$. The inequality holds because the event on the LHS indicates that at least one of the three events on the RHS happens. To upper bound the first term, let $Z_t=\sum_{j\in\mathcal{B}^*_{i_{\min}}}\mathbb{I}\{\pi_t=j\}X_j(t)$, where $X_j(t)$ is the random reward from arm $j$ at time $t$. Let $\nu_t=\sum_{j\in\mathcal{B}^*_{i_{\min}}}\mathbb{I}\{\pi_t=j\}\mu_j$. Note that if $\pi_t\not\in\mathcal{B}^*_{i_{\min}}$, $Z_t=\nu_t=0$. Consider the first term on the RHS of (\ref{obj}):
	\begin{equation}\label{markov}
	\begin{aligned}
		&\mathbb{P}\left(\frac{\sum_{\tau=1}^{t}(Z_{\tau}-\nu_{\tau})}{s}\ge\sqrt{\frac{8\log t}{s}},\tau_{\mathcal{B}^*_{i_{\min}}}(t)=s\right)		\le &\mathbb{P}\left(\mathbb{I}\{\tau_{\mathcal{B}^*_{i_{\min}}}(t)=s\}\cdot e^{\lambda\sum_{\tau=1}^{t}(Z_{\tau}-\nu_{\tau})}\ge e^{\lambda\sqrt{8s\log t}}\right),
	\end{aligned}
	\end{equation}
	Using the Markov inequality, we have
	\begin{equation}\label{markov2}	
	\begin{aligned}
		&\mathbb{P}\left(\mathbb{I}\{\tau_{\mathcal{B}^*_{i_{\min}}}(t)=s\}\cdot e^{\lambda\sum_{\tau=1}^{t}(Z_{\tau}-\nu_{\tau})}\ge e^{\lambda\sqrt{8s\log t}}\right)\le &e^{-\lambda\sqrt{8s\log t}}\cdot \mathbb{E}\left[\mathbb{I}\{\tau_{\mathcal{B}^*_{i_{\min}}}(t)=s\}e^{\lambda\sum_{\tau=1}^{t}(Z_{\tau}-\nu_{\tau})}\right].
	\end{aligned}
	\end{equation}
	Let $\mathcal{F}_{t}=\sigma(Z_1,...,Z_{t})$ be a filtration on the observation history, $Y_t=\mathbb{I}\{\pi_t\in\mathcal{B}^*_{i_{\min}}\}$; clearly $Y_t\in\mathcal{F}_{t-1}$. Let $S_t=\sum_{\tau=1}^{t}Y_{\tau}$, $G_t=e^{\lambda\sum_{\tau=1}^{t}(Z_{\tau}-\nu_{\tau})}$ (note that $G_0=1$ and $S_0=0$). We show that $\left\{G_t/e^{\frac{1}{2}\lambda^2S_t}\right\}_t$ is a submartingale.
	Consider
	\begin{align}
		\mathbb{E}\left[\frac{G_t}{e^{\frac{1}{2}\lambda^2S_t}}\Bigg|\mathcal{F}_{t-1},Y_t=1\right]=&\frac{G_{t-1}\mathbb{E}\left[e^{\lambda(Z_t-\nu_t)}\Big|\mathcal{F}_{t-1},Y_t=1\right]}{e^{\frac{1}{2}\lambda^2(S_{t-1}+1)}}\le\frac{G_{t-1}}{e^{\frac{1}{2}\lambda^2(S_{t-1}+1)}}e^{\frac{1}{2}\lambda^2}=\frac{G_{t-1}}{e^{\frac{1}{2}\lambda^2S_{t-1}}},\label{submartingale1}
	\end{align}
	and
	\begin{align}
		\mathbb{E}\left[\frac{G_t}{e^{\frac{1}{2}\lambda^2S_t}}\Bigg|\mathcal{F}_{t-1},Y_t=0\right]=&\frac{G_{t-1}\mathbb{E}\left[e^{\lambda(Z_t-\nu_t)}\Big|\mathcal{F}_{t-1},Y_t=0\right]}{e^{\frac{1}{2}\lambda^2S_{t-1}}}=\frac{G_{t-1}}{e^{\frac{1}{2}\lambda^2S_{t-1}}}.\label{submartingale2}
	\end{align}
	Note that the inequality in (\ref{submartingale1}) holds because given $\mathcal{F}_{t-1}$, $\pi_t$ is fixed and thus $Z_t=X_{\pi_t}(t)$ which is a sub-Gaussian random variable. Equation (\ref{submartingale2}) holds because given $Y_t=0$, $Z_t=\nu_t=0$. Therefore, $\left\{G_t/e^{\frac{1}{2}\lambda^2S_t}\right\}_t$ is a submartingale and
	\begin{align}
		\mathbb{E}\left[\frac{G_t}{e^{\frac{1}{2}\lambda^2S_t}}\right]\le \mathbb{E}\left[\frac{G_0}{e^{\frac{1}{2}\lambda^2S_0}}\right]=1.
	\end{align}
	Moreover, we have
	\begin{align}
		\mathbb{E}\left[\mathbb{I}\{S_t=s\}\frac{G_t}{e^{\frac{1}{2}\lambda^2S_t}}\right]\le\mathbb{E}\left[\frac{G_t}{e^{\frac{1}{2}\lambda^2S_t}}\right]\le 1,
	\end{align}
	and thus
	\begin{align}
		\mathbb{E}\left[\mathbb{I}\{S_t=s\}G_t\right]\le e^{\frac{1}{2}\lambda^2s}.
	\end{align}
	Applying this to (\ref{markov2}) and choosing $\lambda=\frac{\sqrt{8s\log t}}{s}$, we have
	\begin{align}
		&\mathbb{P}\left(\mathbb{I}\{\tau_{\mathcal{B}^*_{i_{\min}}}(t)=s\}\cdot e^{\lambda\sum_{\tau=1}^{t}(Z_{\tau}-\nu_{\tau})}\ge e^{\lambda\sqrt{8s\log t}}\right)\le e^{\frac{1}{2}\lambda^2s-\lambda\sqrt{8s\log t}}=e^{-4\log t}=t^{-4}.
	\end{align}
	Similarly, the second term can also be upper bounded by $t^{-4}$. For the third term, let 
\begin{align}
\ell\ge\frac{32\log T}{(\min_{j\in\mathcal{B}^*_{i_{\max}}}\mu_j-\max_{j\in\mathcal{B}^*_{i_{\min}}}\mu_j)^2}.
\end{align}
Then, since $s\ge\ell$, $t\le T$, we have
\begin{equation}
\begin{aligned}
	&\frac{\sum_{j\in\mathcal{B}^*_{i_{\max}}}n_j\mu_j}{r}-\frac{\sum_{j\in\mathcal{B}^*_{i_{\min}}}n_j\mu_j}{s}-2c_{t,s}\ge\min_{j\in\mathcal{B}^*_{i_{\max}}}\mu_j-\max_{j\in\mathcal{B}^*_{i_{\min}}}\mu_j-\sqrt{\frac{32\log t}{s}}\ge 0.
\end{aligned}
\end{equation}
Therefore, if we choose $\ell=\lceil\frac{32\log T}{(\min_{j\in\mathcal{B}^*_{i_{\max}}}\mu_j-\max_{j\in\mathcal{B}^*_{i_{\min}}}\mu_j)^2}\rceil$, we get
\begin{equation}
	\begin{aligned}
		&\mathbb{E}[\tau_{\mathcal{B}^*_{i_{\min}}}(T)]\le \ell+\sum_{t=|\mathcal{B}^*|}^{T-1}\sum_{s=1}^{t}\sum_{r=1}^{t}2t^{-4}\\
		\le& \frac{32\log T}{(\min_{j\in\mathcal{B}^*_{i_{\max}}}\mu_j-\max_{j\in\mathcal{B}^*_{i_{\min}}}\mu_j)^2}+O(1)\\
		=& \frac{32\log T}{(\min_{j\in\mathcal{B}^*_{i_{\min}}}\Delta_j-\max_{j\in\mathcal{B}^*_{i_{\max}}}\Delta_j)^2}+O(1).
	\end{aligned}
\end{equation}
Now we upper bound the number of times that arms in $\mathcal{B}^*_{i_{\max}}$ have been played up to time $T$. For each $i\in\mathcal{B}^*_{i_{\max}}\setminus \mathcal{A}$,
\begin{equation}
	\begin{aligned}
		&\mathbb{E}[\tau_i(T)]=\mathbb{E}\left[1+\sum_{t=|\mathcal{B}^*|+1}^{T}\mathbb{I}\{\pi_t=i\}\right]\\
		\le &\ell +\sum_{t=|\mathcal{B}^*|+1}^{T}\mathbb{P}\left(\pi_t=i,\tau_i(t-1)\ge\ell\right)\\
		\le &\ell +\sum_{t=|\mathcal{B}^*|+1}^{T}\mathbb{P}\left(L_i(t-1)\ge L_{i_{\max}}(t-1),\tau_i(t-1)\ge\ell\right).
	\end{aligned}
\end{equation}
Using an argument similar to that for $\tau_{\mathcal{B}^*_{i_{\min}}}$, we get
\begin{align}
	\mathbb{E}[\tau_i(T)]\le \frac{32\log T}{\Delta_i^2}+O(1).
\end{align}
Therefore, we get the upper bound on regret of LSDT-CSI in (\ref{upperbound}) if $\mathcal{G}^*_{\epsilon}$ is connected but not complete.

\section{Proof of Theorem 3}\label{pflowerbound}
The basic structure of the proof follows that in \cite{lai1985asymptotically} and \cite{buccapatnam2014stochastic}. For every suboptimal arm $i~(\mu_i<\mu_{i_{\max}})$, we construct a new set of reward distributions with parameters
$\bm{\theta}^{(i)}=(\theta_1^{(i)},\theta_2^{(i)},...,\theta_K^{(i)})$ and means $\bm{\mu}^{(i)}=(\mu_1^{(i)},\mu_2^{(i)},...,\mu_K^{(i)})$ such that $\mu_{i}^{(i)}=\max_{j\in\mathcal{V}}\mu_j^{(i)}$. Then we can generate a new graph $\mathcal{G}_{\epsilon}^{(i)}=(\mathcal{V}^{(i)},\mathcal{E}^{(i)})$ where $\mathcal{V}^{(i)}$  is the set of new arms, and $(u,v)\in \mathcal{E}^{(i)}$ if and only if $|\mu_u^{(i)}-\mu_v^{(i)}|<\epsilon$. 

To establish the relationship between the new problem and the original one, we need to retain the same graph connectivity. Since $\mathcal{B}^*$ is the set of arms that could potentially be optimal given $\mathcal{G}_{\epsilon}^*$, we could only construct for each $i\in\mathcal{B}^*\setminus\mathcal{A}$ a set of new reward distributions with parameters $\theta^{(i)}$ such that arm $i$ is optimal. Thus, for each $i\in\mathcal{B}^*\setminus\mathcal{A}$, consider $\theta^{(i)}$ with mean rewards $\mu^{(i)}$ satisfying:
\begin{enumerate}
	\item If $i\in \mathcal{B}^*_{i_{\max}}\setminus\mathcal{A}$: $\mu_i^{(i)}=\mu_{i_{\max}}+\eta$, $\mu_j^{(i)}=\mu_j,\forall j\neq i$.
	\item If $i\in \mathcal{B}^*_{i_{\min}}$: $\mu_i^{(i)}=\mu(\theta_i')+\eta$, $\mu_j^{(i)}=\mu(\theta_j'),\forall j\neq i,$ where $\mu(\theta_i'),\mu(\theta_j')$ are defined as
	\begin{equation}\label{thetaj}
		\mu(\theta_j')=\left\{
		\begin{aligned}
			&\mu_j,&\forall j\in \mathcal{B}_{i_{\max}}^*,\\
			&\mu_{i_{\max}}+\min_{k\in \mathcal{B}^*_{i_{\max}}}\mu_k-\mu_{i_{\min}},&\forall j\in \mathcal{B}^*_{i_{\min}},\\
			&\mu_{i_{\max}}+\min_{k\in \mathcal{B}^*_{i_{\max}}}\mu_k-\mu_{j},&\forall j\not\in \mathcal{B}^*.
		\end{aligned}
		\right.
	\end{equation}
\end{enumerate}
One can check that in both cases, $\mathcal{G}^{(i)}_{\epsilon}$ and $\mathcal{G}^*_{\epsilon}$ have the same connectivity if 
\begin{align}
\eta<\epsilon-\max\big\{\mu_{i_{\max}}-\min_{j\in\mathcal{N}[i_{\max}]}\mu_j,\max_{j\in\mathcal{N}[i_{\min}]}\mu_j-\mu_{i_{\min}}\big\}.
\end{align}

Then we define the log-likelihood ratio between the observations from two sets of arms with distribution parameters $\theta=(\theta_1,...,\theta_K)$ and $\theta^{(i)}=(\theta_1^{(i)},...,\theta_K^{(i)})$ up to time $T$ under any uniformly good policy $\pi$ as
\begin{align}
	\mathcal{L}^{(i)}(T)=\sum_{j\in\mathcal{V}}\sum_{s=1}^{\tau_j(T)}\log\left(\frac{f(X_{j,s};\theta_j)}{f(X_{j,s};\theta_j^{(i)})}\right),
\end{align}
where $\tau_j(T)$ is the number of times arm $j$ has been played by policy $\pi$ up to time $T$ and $X_{j,s}$ is the reward obtained when arm $j$ is played for the $s$-th time. We show that it is unlikely to have
\begin{align}
	\sum_{j\in\mathcal{V}}\tau_j(T)I(\theta_j||\theta_j^{(i)})\le(1-\gamma)\log T,
\end{align}
under two separate cases: $\mathcal{L}^{(i)}(T)\le(1-\delta)\log T$ and $\mathcal{L}^{(i)}(T)>(1-\delta)\log T$ where $\delta,\gamma>0$ are determined later.
\begin{enumerate}
	\item If $\mathcal{L}^{(i)}(T)\le(1-\delta)\log T$: by the uniform goodness of policy $\pi$, we have
	\begin{equation}
		\begin{aligned}
		&{\mathbb{P}}_{\bm{\theta}^{(i)}}\left\{\sum_{j\in\mathcal{V}}\tau_j(T)I(\theta_j||\theta_j^{(i)})\le(1-\gamma)\log T\right\}\\
		\le&{\mathbb{P}}_{\bm{\theta}^{(i)}}\left\{\tau_i(T)I(\theta_i||\theta_i^{(i)})\le (1-\gamma)\log T\right\}\\
		=&{\mathbb{P}}_{\bm{\theta}^{(i)}}\left\{T-\tau_i(T)\ge T-\frac{(1-\gamma)\log T}{I(\theta_i||\theta_i^{(i)})}\right\}\\
		\le&\frac{\mathbb{E}_{\bm{\theta}^{(i)}}[T-\tau_i(T)]}{ T-\frac{(1-\gamma)\log T}{I(\theta_i||\theta_i^{(i)})}}=o(T^{\alpha-1}),
		\end{aligned}
	\end{equation}
	for all $\alpha>0$ as $T\to\infty$. 
	
	We let 
	\begin{align}	
	H=\Bigg\{\sum_{j\in\mathcal{V}}\tau_j(T)I(\theta_j||\theta_j^{(i)})\le(1-\gamma)\log T,\mathcal{L}^{(i)}(T)\le(1-\delta)\log T\Bigg\}.
	\end{align}
	By a change of measure from $\mathbb{P}_{\bm{\theta}^{(i)}}$ to $\mathbb{P}_{\bm{\theta}}$, we have
	\begin{equation}
		\begin{aligned}
			{\mathbb{P}}_{\bm\theta}\{{H}\}&\le\int_{H}\textrm{d}P_{\bm{\theta}}=\int_{H}\exp\left(\mathcal{L}^{(i)}(T)\right)\textrm{d}P_{\bm{\theta}^{(i)}}\le T^{1-\delta}o(T^{\alpha-1})=o(1),
		\end{aligned}
	\end{equation}
	for all $\delta>0$ as $T\to\infty$ if we choose $\alpha<\delta$.
	\item If $\mathcal{L}^{(i)}(T)>(1-\delta)\log T$: by the strong law of large numbers, as $t\to\infty$, we have
	\begin{align}
		\frac{1}{t}\sum_{s=1}^{t}\log\left(\frac{f(X_{j,s};\theta_j)}{f(X_{j,s};\theta_j^{(i)})}\right)\rightarrow I(\theta_j||\theta_j^{(i)})~ \textrm{almost surely}.
	\end{align}
	Rewrite $\mathcal{L}^{(i)}(T)$ as
	\begin{align}
		\mathcal{L}^{(i)}(T)=\sum_{j\in\mathcal{V}}\tau_j(T)\frac{1}{\tau_j(T)}\sum_{s=1}^{\tau_j(T)}\log\left(\frac{f(X_{j,s};\theta_j)}{f(X_{j,s};\theta_j^{(i)})}\right)
	\end{align} 
	and then
	\begin{equation}
		\begin{aligned}
			&{\mathbb{P}}_{\bm\theta}\Bigg\{\sum_{j\in\mathcal{V}}\tau_j(T)I(\theta_j||\theta_j^{(i)})\le(1-\gamma)\log T,L^{(i)}(T)>(1-\delta)\log T\Bigg\}\\
			=&{\mathbb{P}}_{\bm\theta}\Bigg\{\sum_{j\in\mathcal{K}}\tau_j(T)I(\theta_j||\theta_j^{(i)})\le(1-\gamma)\log T,\sum_{j\in\mathcal{K}}\tau_j(T)\frac{1}{\tau_j(T)}\sum_{s=1}^{\tau_j(T)}\log\left(\frac{f(X_{j,s};\theta_j)}{f(X_{j,s};\theta_j^{(i)})}\right)>(1-\delta)\log T\Bigg\}\\
			=&o(1),
		\end{aligned}
	\end{equation}
as $T\to\infty$ if we choose $\gamma>\delta$.
\end{enumerate}

Now we have proved that for all $i\in \mathcal{B}^*\setminus \mathcal{A}$, we have
\begin{align}
	\sum_{j\in\mathcal{V}}\frac{\mathbb{E}[\tau_j(T)]}{\log T}I(\theta_j||\theta_j^{(i)})\ge 1.
\end{align}
To be specific, 
\begin{enumerate}
	\item If $i\in \mathcal{B}^*_{i_{\max}}\setminus \mathcal{A}$, let $\eta\to 0$, we have
	\begin{align}
		\mathbb{E}[\tau_i(T)]\ge \frac{\log T}{I(\theta_i||\theta_{i_{\max}})},
	\end{align}
	\item If $i\in\mathcal{B}^*_{i_{\min}}$, let $\eta\to 0$, we have
	\begin{align}
		\sum_{j\not\in\mathcal{B}^*_{i_{\max}}}\mathbb{E}[\tau_j(T)]I(\theta_j||\theta_j')\ge \log T.
	\end{align}
\end{enumerate}
Therefore, the optimal constant in front of $\log T$ is the solution to the linear program $\mathcal{P}_1$:
\begin{equation}\label{lowerboundLP}
		\begin{aligned}
			\mathcal{P}_1:C_1&=\min_{\{\tau_i\}_{i\in\mathcal{V}}}\sum_{i\in\mathcal{V}}\Delta_i\tau_i,\\
			s.t.&\sum_{j\not\in \mathcal{B}^*_{i_{\max}}}\tau_jI(\theta_j||\theta_j')\ge 1,\\
			&~~~\tau_i\ge\frac{1}{I(\theta_i||\theta_{i_{\max}})},~~\forall i\in \mathcal{B}^*_{i_{\max}}\setminus\mathcal{A},\\
			&~~~\tau_i\ge 0,~~~~~~~~~~~~~~~\forall i\in\mathcal{V}.
		\end{aligned}
	\end{equation}
	where $\theta_j'$ is the parameter of the density function $f(x_j;\theta_j')$ whose mean value $\mu(\theta_j')$ is defined in (\ref{thetaj}).

In light of the LP $\mathcal{P}_1$, each sub-optimal arm in $\mathcal{B}^*_{i_{\max}}$ has to be played $\Omega(\log T)$ times to be distinguished from the optimal one. Moreover, the total number of times that arms in $\mathcal{V}\setminus\mathcal{B}^*_{i_{\max}}$ are played should be at least $\Omega(\log T)$. Thus if we consider the regret order in terms of the number of arms and the time length, we can conclude that for fixed $\Delta_i$, $I(\theta_i||\theta_i')$ and $I(\theta_i||\theta_{i_{\max}})$, the regret for any uniformly good policy is of order $${\Omega}\Big((1+|\mathcal{B}^*_{i_{\max}}\setminus\mathcal{A}|)\log T\Big),$$
	as $T\to\infty$, which matches the upper bound on regret of LSDT-CSI. Therefore, LSDT-CSI is order optimal.
\section{CONSISTENT-NAE-3SAT and Proof of NP-Completeness}\label{pflm1}
We first give the definition of CONSISTENT-NAE-3SAT and then show that it is NP-complete.
\begin{itemize}
\item[]\emph{{CONSISTENT-NAE-3SAT}}
\item[]\emph{{[INPUT]}:  A not-all-equal satisfiable 3-SAT instance: there exists a truth assignment such that every clause contains one or two true literals .}
\item[]\emph{{[QUESTION]}: Does there exist a consistent truth assignment, i.e., every clause contains exactly one true literal OR every clause contains exactly two true literals? }
\end{itemize}

The problem is clearly in NP since a given truth assignment can be verified in polynomial time. To show the NP-completeness, we first show that 1-CONSISTENT-NAE-3SAT is NP-complete. Note that 1-CONSISTENT-NAE-3SAT asks if there exists a truth assignment such that every clause has exactly 1 true literal given true instance of NAE-3SAT.

It is clear that 1-CONSISTENT-NAE-3SAT is in NP. We give a reduction from 1-IN-3SAT, a known NP-complete problem \cite{schaefer1978complexity}, as follows: given an instance of 1-IN-3SAT, for every clause $C_i=(x_{i,1},x_{i,2},x_{i,3})$, we construct three clauses in the corresponding 1-CONSISTENT-NAE-3SAT instance with two additional variables $a_i,b_i$:
$$
C_{i,1}=(x_{i,1},x_{i,2},a_i),~C_{i,2}=(x_{i,2},x_{i,3},b_i),~C_{i,3}=(a_i,b_i,x_{i,2}).
$$ 
This is clearly a polynomial time reduction. 

We first show that the 1-CONSISTENT-NAE-3SAT instance we constructed is not-all-equal (NAE) satisfiable, i.e., there exists a truth assignment such that every clause is satisfied and contains at most 2 true literals. For any arbitrary truth assignment of $(x_1,...,x_n)$, we can choose $(a_1,b_1,...,a_m,b_m)$ according to Table \ref{truthtable}. One can check that every clause is satisfied with at most 2 true literals. Therefore, the 1-CONSISTENT-NAE-3SAT instance is NAE satisfiable.
\begin{table}[h!]
\begin{center}
\begin{tabular}{ c  c  c | c  c }
$x_{i,1}$ & $x_{i,2}$ & $x_{i,3}$ & $a_i$ & $b_i$\\
\hline
0 & 0 & 0 & 1 & 1\\ 
0 & 0 & 1 & 1 & 0\\ 
0 & 1 & 0 & 0 & 0\\ 
1 & 0 & 0 & 0 & 1\\ 
0 & 1 & 1 & 0 & 0\\ 
1 & 0 & 1 & 1 & 1\\ 
1 & 1 & 0 & 0 & 0\\ 
1 & 1 & 1 & 0 & 0\\ 
\end{tabular}
\end{center}
\caption{Truth table for NAE-3SAT.}\label{truthtable}
\vspace{-0.5cm}
\end{table}

Now we assume that the original 1-IN-3SAT instance is satisfied by an assignment of $(x_1,...,x_n)$ with three cases:
\begin{enumerate}[(i)]
	\item only $x_{i,1}$ is true: let $a_i=0,b_i=1$;
	\item only $x_{i,2}$ is true: let $a_i=0,b_i=0$;
	\item only $x_{i,3}$ is true: let $a_i=1,b_i=0$.
\end{enumerate} 
It is clear that the 1-CONSISTENT-NAE-3SAT is satisfied by the assignment of $(x_1,...,x_n,a_1,b_1,...,a_m,b_m)$.

On the other hand, assume that the 1-CONSISTENT-NAE-3SAT instance is satisfied by an assignment of $(x_1,...,x_n,$ $a_1,b_1,...,a_m,b_m)$. Consider clause $C_{i,1}=(x_{i,1},x_{i,2},a_i)$:
\begin{enumerate}[(i)]
	\item only $x_{i,1}$ is true: $x_{i,2}=a_i=0$. It is clear that $b_i=1$ since $C_{i,3}$ is satisfied. Thus, $x_{i,3}=0$ since $C_{i,2}$ is satisfied with only one true literal ($b_i$). Therefore, we have $x_{i,1}=1,x_{i,2}=0,x_{i,3}=0$;
	\item only $x_{i,2}$ is true: since $C_{i,1},C_{i,2},C_{i,3}$ are all satisfied with only one true literal in each clause, we have $x_{i,1}=x_{i,3}=a_i=b_i=0$;
	\item only $a_i$ is true: $x_{i,1}=x_{i,2}=0$. since $C_{i,2},C_{i,3}$ are satisfied with only one true literal in each clause, we have $b_i=x_{i,2}=0$ and $x_{i,3}=1$.
\end{enumerate}
Therefore, every clause $C_i=(x_{i,1},x_{i,2},x_{i,3})$ in the original 1-IN-3SAT instance is satisfied with only one true literal. 

In summary, we have shown that the 1-IN-3SAT instance is satisfiable if and only if the corresponding 1-CONSISTENT-NAE-3SAT instance is satisfiable, which indicates the NP-completeness of 1-CONSISTENT-NAE-3SAT.

Finally, we show that CONSISTENT-NAE-3SAT (clearly in NP) is NP-complete via a reduction from 1-CONSISTENT-NAE-3SAT. Given an instance of 1-CONSISTENT-NAE-3SAT with $n$ variables $(x_1,...,x_n)$ and $m$ clauses $C_{1},...,C_m$, we add a new clause $C_0=(x_1,\bar{x}_1,0)$ and get an instance of CONSISTENT-NAE-3SAT with $n$ variables and $m+1$ clauses. This is clearly a polynomial reduction and there must exist a NAE satisfiable assignment for the new instance. Now we assume that the original 1-CONSISTENT-NAE-3SAT instance has a satisfiable assignment $(x_1,...,x_n)$, it follows immediately that the CONSISTENT-NAE-3SAT is also satisfied by the same assignment. On the other hand, we assume that CONSISTENT-NAE-3SAT is satisfied by a truth assignment $(x_1,...,x_n)$. Since $C_0$ is satisfied with exactly 1 true literal, so are the other clauses. Thus $(x_1,...,x_n)$ is a satisfiable assignment for the 1-CONSISTENT-NAE-3SAT instance. Therefore, we have shown that the 1-CONSISTENT-NAE-3SAT instance is satisfiable if and only if the corresponding CONSISTENT-NAE-3SAT instance is satisfiable.

In conclusion, CONSISTENT-NAE-3SAT is NP-complete.

\section{Proof of Theorem 4}\label{pfNPcomplete}
It is clear that LEFTANCHOR is in NP since given a graph, one can verify if it is  a UIG and if a specific node is a left anchor in polynomial time. Now we show the NP-completeness of LEFTANCHOR through a reduction from CONSISTENT-NAE-3SAT. The reduction is similar to the one used in proving the NP-completeness of the UIG Sandwich Problem in \cite{kaplan1994complexity}. 

Given an instance of CONSISTENT-NAE-3SAT, let $x_1,...,x_n$ be $n$ variables and $C_1,...,C_m$ be $m$ clauses where $C_i=(x_{i,1},x_{i,2},x_{i,3})$ and $x_{i,j}\in\{x_1,...,x_n,\bar{x}_1,...,\bar{x}_n\}$. For every variable $x_i$, we construct a variable gadget with 5 vertices $(x_i,x_i',p,\bar{x}_i',\bar{x}_i)$ in the LEFTANCHOR instance: add 4 type-S edges $(x_i,x_i'),(x_i',p),(p,\bar{x}_i'),(\bar{x}_{i}',\bar{x}_i)$ to $\mathcal{E}_1$ and 6 type-2 edges $(x_i,p)$, $(x_i,\bar{x}_i')$, $(x_i,\bar{x}_i)$, $(x_i',\bar{x}_i')$, $(x_i',\bar{x}_i)$, $(p,\bar{x}_{i})$ to $\mathcal{E}_2$ (see Figure \ref{variable}).

Moreover, for every clause $C_i=(x_{i,1},x_{i,2},x_{i,3})$, we construct a clause gadget with 6 vertices $(x_{i,1},x_{i,2},x_{i,3},v_{i,1},$ $v_{i,2},v_{i,3})$ in the LEFTANCHOR instance: add 3 type-S edges $(x_{i,j},v_{i,j}), j=1,2,3$ to $\mathcal{E}_1$ and 9 type-D edges $(v_{i,j},x_{i,k}),j\neq k$ and $(v_{i,j},v_{i,k}),j\neq k$ to $\mathcal{E}_2$ (see Figure \ref{clause}). Note that every vertex $x_{i,j}$ in the clause gadget belongs to one of the variable gadgets, we don't~create~additional~vertices.

\begin{figure}[t]
	\begin{center}
	\begin{subfigure}{0.5\textwidth}
		\begin{center}
		\includegraphics[width = 0.9\textwidth]{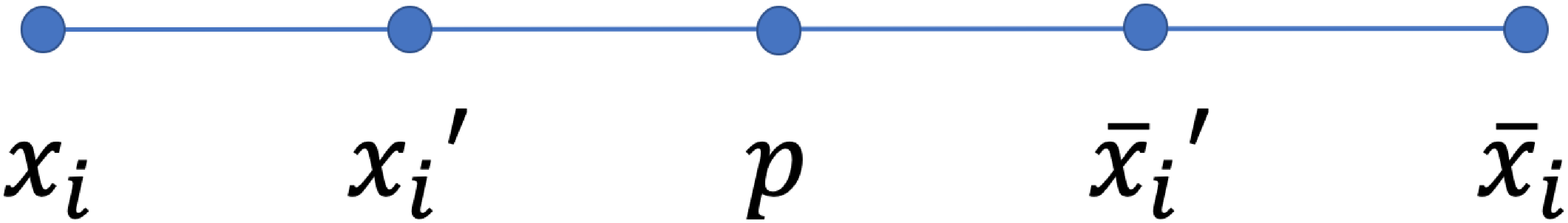}
		\vspace{-2.5cm}
		\caption{\small{Variable gadget}\label{variable}}
		\end{center}
	\end{subfigure}

	\begin{subfigure}{0.5\textwidth}
		\begin{center}
		\includegraphics[width=0.9\textwidth]{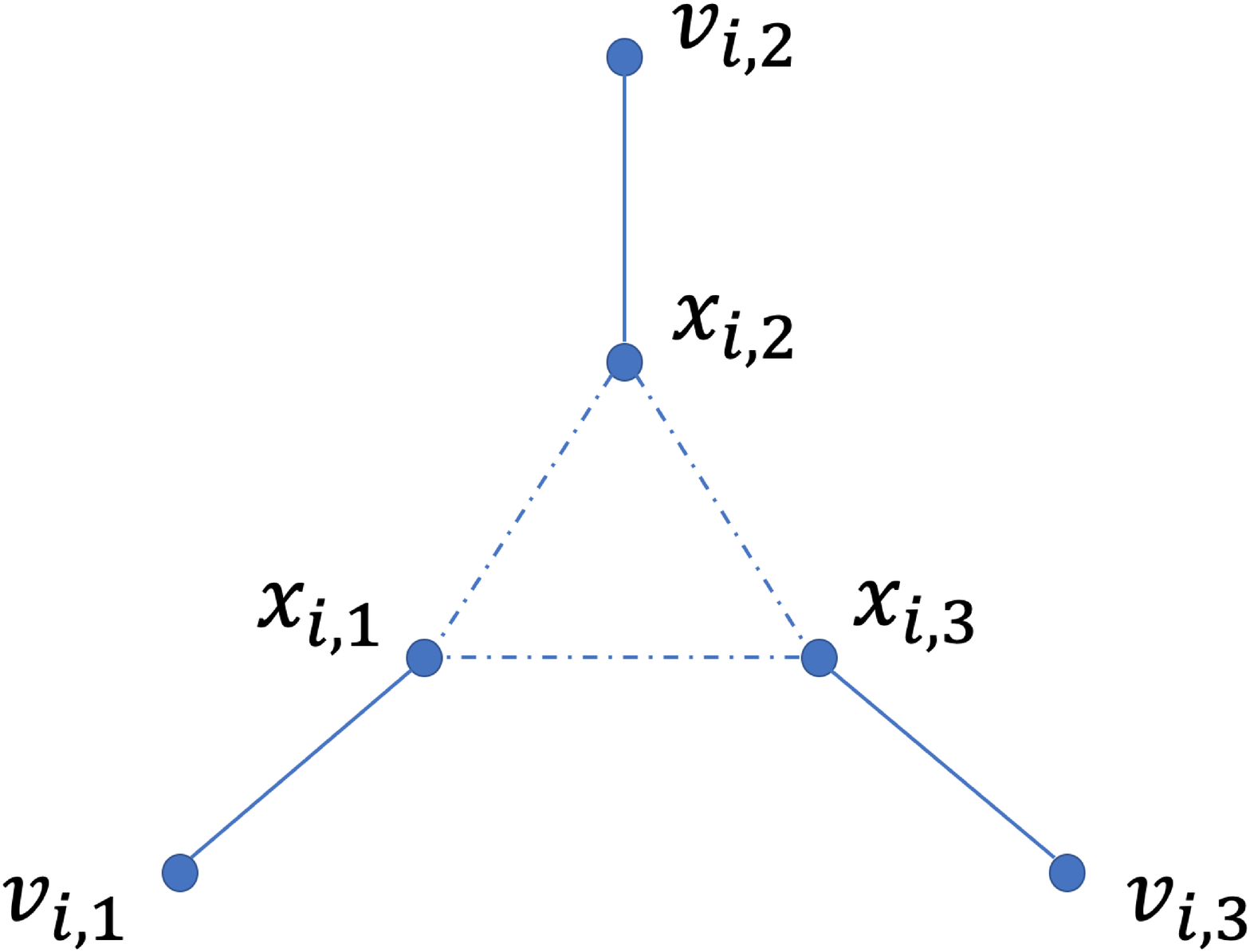}
		\vspace{-.5cm}
		\caption{\small{Clause gadget}\label{clause}}
		\end{center}
	\end{subfigure}
	\end{center}
	\caption{\small{Variable and clause gadgets: in each gadget, solid line edges represent type-S edges in $\mathcal{E}_1$, missing edges represent type-D edges in $\mathcal{E}_2$, dash line edges represent unknown edges in $(\mathcal{E}_1\cup\mathcal{E}_2)^C$.}}
	\vspace{-.5cm}
\end{figure}
In summary, there are $4n+3m+1$ vertices in the LEFTANCHOR instance, i.e.,
$$
\begin{aligned}
\mathcal{V}=&\{p\}\cup\{x_i,x_i',\bar{x}_i',\bar{x}_i|i=1,...,n\}\cup\{v_{i,1},v_{i,2},v_{i,3}|i=1,...,m\},
\end{aligned}
$$
$4n+3m$ type-S edges, i.e.,
$$
\begin{aligned}
\mathcal{E}_1=&\Big\{(x_i,x_i'),(x_i',p),(p,\bar{x}_i'),(\bar{x}_{i}',\bar{x}_i)|i=1,...,n\Big\}\cup\Big\{(x_{i,j},v_{i,j})|i=1,...,m,j=1,2,3\Big\},
\end{aligned}
$$
and $6n+9m$ type-D edges, i.e.,
$$
\begin{aligned}
&\mathcal{E}_2=\Big\{(x_i,p),(x_i,\bar{x}_i'), (x_i,\bar{x}_i), (x_i',\bar{x}_i'), (x_i',\bar{x}_i), (p,\bar{x}_{i})\Big|i=1,...,n\Big\}\\
&~~~~~~~~\cup\Big\{(v_{i,j},x_{i,k})\Big|i=1,...,m,j\neq k\Big\}\cup\Big\{(v_{i,j},v_{i,k})\Big|i=1,...,m,j\neq k\Big\}.
\end{aligned}
$$

Clearly the construction is done in polynomial time. Moreover, it is shown in \cite{kaplan1994complexity} that if there exists a truth assignment of $(x_1,...,x_n)$ such that every clause $C_i$ is satisfied with at most two true literals, there exists a UIG $\mathcal{G}'=(\mathcal{V},\mathcal{E}_3)$ such that $\mathcal{E}_{1}\subseteq\mathcal{E}_3$ and $\mathcal{E}_3\cap\mathcal{E}_2=\emptyset$. Now, let $x_1$ and $\bar{x}_1$ be two nodes that we want to decide if they can be left anchors. Then we get two corresponding instances of LEFTANCHOR for any given instance of CONSISTENT-NAE-3SAT. We need to show that the instance of CONSISTENT-NAE-3SAT is satisfiable if and only if at least one of the two corresponding instances of LEFTANCHOR is satisfiable, i.e., at least one of the two nodes $x_1$ and $\bar{x}_1$ can be a left anchor of a UIG $\mathcal{G}''=(\mathcal{V},\mathcal{E}_4)$ where $\mathcal{E}_1\subseteq\mathcal{E}_4$ and $\mathcal{E}_4\cap\mathcal{E}_2=\emptyset$.

We first assume that the CONSISTENT-NAE-3SAT instance is satisfied by a truth assignment of $(x_1,...,x_n)$. Suppose every clause has only 1 true literal and with out loss of generality, we assume $x_1=1$. We show that $x_1$ can be a left anchor of a UIG satisfying the constraints. We assign a unit length interval for every vertex in $\mathcal{V}$ as follows (see Figure \ref{intervalassing}): we let $I(p)=P$. For $i=1,...,n$, if $x_i=1$, we let $I(x_i)=A_i$, $I(x_i')=L$, $I(\bar{x}_i')=R$ and $I(\bar{x}_i)=B_i$; if $x_i=0$, we let $I(x_i)=B_i$, $I(x_i')=R$, $I(\bar{x}_i')=L$ and $I(\bar{x}_i)=A_i$. In other words, we put all the true (or false) literals to the left (or right) ``staircases'' and assign $x_i'$ and $\bar{x}_i'$ accordingly. For every clause $C_i,i=1,...,m$, let $x_{i,j}$ be the true literal in $C_i$,  then we let $I(v_{i,j})=I(x_{i,j})$. For the other two false literals $x_{i,k_1}$, $x_{i,k_2}$, the two corresponding intervals both have non-overlapping tails. Therefore, we can assign $I(v_{i,k_1})$ and $I(v_{i,k_2})$ extending from the respective tails. For example in Figure \ref{intervalassing}: consider a clause $(x_1,\bar{x}_2,\bar{x}_3)$, $A_1,B_2,B_3$ are assigned to the three literals and $V_1,V_2,V_3$ are assigned to the associated vertices $v_{i,1},v_{i,2},v_{i,3}$. One can easily check that the induced UIG from the interval assignment of vertices in $\mathcal{V}$ satisfies all the edge constraints and $x_1$ is a left anchor.
\begin{figure}
	\begin{center}
	\includegraphics[width = 0.6\textwidth]{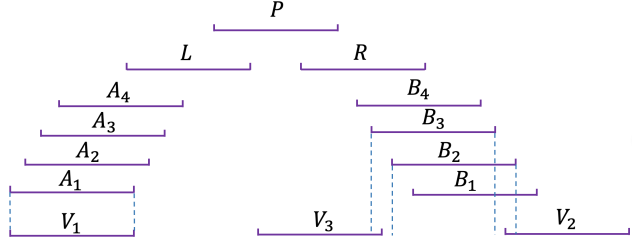}
	\vspace{-2.5cm}
	\caption{\small{Unit interval realization of a satisfiable instance of CONSISTENT-NAE-3SAT.}}\label{intervalassing}

	\end{center}
	
\end{figure}
Similarly, we can show that if $x_1=0$, then $\bar{x}_1$ can be a left anchor of a UIG in the LEFTANCHOR instance. Now suppose every clause has 2 true literals, we can use similar proof structure to show that if $x_1=0$, then $x_1$ can be a left anchor of a UIG satisfying the edges constraints; otherwise $\bar{x}_1$ can be a left anchor.

On the other hand, we assume that at least one of the two nodes $x_1$ and $\bar{x}_1$ can be a left anchor of a UIG $\mathcal{G}''=(\mathcal{V},\mathcal{E}_4)$ satisfying the edge constraints and we need to show that the original instance of CONSISTENT-NAE-3SAT is satisfiable. Without loss of generality, we assume that $x_1$ can be a left anchor. Let $I(v)$ be the unit interval assigned to vertex $v$ in the UIM of the UIG $\mathcal{G}''$. By changing scale and shifting, we can assume that every interval has length $1$ and $I(p)=[0,1]$. Consider for every variable gadget $i=1,...,n$, It is not difficult to see that $I(x_i)$ contains either $-1$ or $2$. We assign truth values to the variables as follows: let $x_i=1$ if $I(x_i)$ contains $-1$ and $x_i=0$ if $I(x_i)$ contains $2$. Now we show that every clause is satisfied with only 1 true literal OR every clause is satisfied with only 2 true literals by the truth assignment. 

Consider every clause gadget: we show that there is exactly one edge among $(x_{i,j},x_{i,k}),j\neq k$ that belongs to $\mathcal{E}_4$: 
\begin{enumerate}[(i)]
	\item if all three edges belong to $\mathcal{E}_4$, then $v_{i,1},v_{i,2},v_{i,3}$ form an asteroidal triple\footnote{An asteroidal triple in a graph is a triple of mutually non-adjacent nodes $i,j,k$ such that between any two of them, there exists a path avoiding the neighborhood of the third.}, which is forbidden in a UIG \cite{lekkeikerker1962representation};
	\item if exactly two edges belong to $\mathcal{E}_4$, e.g., $(x_{i,1},x_{i,2})$ and $(x_{i,1},x_{i,3})$, then $x_{i,1},x_{i,2},x_{i,3}$ and $v_{i,1}$ form a claw ($K_{1,3}$) which is also forbidden in a UIG \cite{lekkeikerker1962representation};
	\item $I(x_{i,j})$ contains either $-1$ or $2$. Hence, there always exist two intervals containing the same point, thus intersecting. Therefore, at least one edge belongs to $\mathcal{E}_4$.
\end{enumerate} 
Since there is exactly one edge among $(x_{i,j},x_{i,k}),j\neq k$ that belongs to $\mathcal{E}_4$, it follows that every clause has only $1$ or $2$ true literals. Furthermore, since $x_1$ is a left anchor of $\mathcal{G}''$, within every clause gadget containing $x_1$, the truth assignment of $x_{1}$ should be different from the other two variables: consider a clause gadget containing $x_1$, assume that $x_{i,k}$ has the same truth assignment as $x_1$ and $x_1$ is connected to $v_1$, then we have $(x_1,x_{i,k}),(x_1,v_1)\in\mathcal{E}_4$, but $(v_1,x_{i,k})\not\in\mathcal{E}_4$. This contradicts the assumption that $x_1$ is a left anchor since $(v_1,x_{i,k})$ should also belong to $\mathcal{E}_4$ if $x_1$ is a left anchor and $(x_1,x_{i,k}),(x_1,v_1)\in\mathcal{E}_4$. 

We first consider if $x_1=1$, then we claim that every clause has exactly one true literal. We prove by contradiction: assume that there exists a clause $C_i=(x_{i,1},x_{i,2},x_{i,3})$ with two true literals, e.g., $x_{i,1},x_{i,2}$. By the assignment of truth values, it is clear that $I(x_{i,1})=[l_1,r_1]$ and $I(x_{i,2})=[l_2,r_2]$ contains $-1$. Without loss of generality, we assume that $l_1<l_2$. Then we consider $I(v_{i,1})=[l_v,r_v]$: since $(v_{i,1},x_{i,1})\in\mathcal{E}_4$ and $(v_{i,1},x_{i,2})\not\in\mathcal{E}_4$, we have $l_1\le r_v<l_2$. Then it is not difficult to see that $I(v_{i,1})$ doesn't contain $-1$ and $l_v$ is smaller than the left coordinate of $I(x_1)$, i.e., $x_1$ is not a left anchor. Contradiction! Therefore, we have shown that every clause has exactly one true literal. On the other hand, if $x_1=0$, it can be shown similarly that every clause has exactly two true literals. In summary, we have shown that the original instance of CONSISTENT-NAE-3SAT is satisfiable if at least one of the two nodes $x_1$ and $\bar{x}_1$ can be a left anchor of the corresponding UIG. This completes the entire reduction and we conclude that LEFTANCHOR is NP-complete.

\section{Proof of Theorem 5}\label{pfsizeB0}
We first show that with probability at least $1-\frac{1}{K^2}$, every arm $i\not\in\mathcal{B}^*$ is eliminated by the offline elimination step of LSDT-PSI. Consider any $i\not\in\mathcal{B}^*$. Note that under Assumption \ref{diversity1}, there exists $j,k\in[m]$ s.t. $\forall u\in\mathcal{B}_j^{*},v\in\mathcal{B}_k^{*}$,
\begin{align}
&(u,i)\in\mathcal{E}_{\epsilon}^*,~(v,i)\in\mathcal{E}_{\epsilon}^*,~(u,v)\in\stcomp{\mathcal{E}_{\epsilon}^*}.
\end{align}
Let $N=\min\{|\mathcal{B}_j^*|,|\mathcal{B}_k^*|\}$. According to Assumption \ref{diversity2}, we have $N\ge\kappa\log K$. We select $\{u_1,u_2,...,u_N\}$ from $\mathcal{B}_j^*$ and $\{v_1,v_2,...,v_N\}$ from $\mathcal{B}_k^*$, then for $n=1,...,N$, define 
\begin{align}
&E_{n,1}=\Big\{(u_n,i)\in\mathcal{E}_{\epsilon}^S\Big\},\\
&E_{n,2}=\Big\{(v_n,i)\in\mathcal{E}_{\epsilon}^S\Big\},\\
&E_{n,3}=\Big\{(u_n,v_n)\in\mathcal{E}_{\epsilon}^D\Big\}.
\end{align} 
According to Assumption \ref{observation}, $\{E_{n,\ell}\}_{n=1,...,N,\ell=1,2,3}$ are independent and $\mathbb{P}(E_{n,1})=\mathbb{P}(E_{n,2})=p_S$, $\mathbb{P}(E_{n,3})=p_D$. Therefore, according to the offline elimination step of LSDT-PSI, the probability that arm $i$ is not eliminated is upper bounded as follows:
\begin{align}
\mathbb{P}(i\textrm{~is~not~eliminated from }\mathcal{B}_0)&\le\prod_{n=1}^{N}\left(1-\prod_{\ell=1}^{3}\mathbb{P}(E_{n,\ell})\right)=(1-p_S^2p_D)^N.
\end{align}
Since $N\ge\kappa\log K$ and according to Assumption \ref{observation}, $p_S^2p_D\ge 1-e^{-2/\kappa}$, we have
\begin{align}
(1-p_S^2p_D)^N\le (e^{-2/\kappa})^{\kappa\log K}\le \frac{1}{K^2}.
\end{align}
Moreover, we can show that
\begin{equation}
\begin{aligned}
\mathbb{E}_{\mathcal{E}_{\epsilon}^S,\mathcal{E}_{\epsilon}^D}\Big[\big|\mathcal{B}_0\big|\Big]&=\sum_{i=1}^{K}\mathbb{P}(i\textrm{~is~not~eliminated from }\mathcal{B}_0) \\
&= |\mathcal{B}^*| + \sum_{i\not\in\mathcal{B}^*}\frac{1}{K^2}\le |\mathcal{B}^*|+o(1),
\end{aligned}
\end{equation}
as $K\to\infty$.
\section{Proof of Theorem 6}\label{pfupperboundpsi}
The basic structure of the proof follows that in \cite{auer2010ucb} and \cite{buccapatnam2014stochastic}. Define $\mathcal{Q}=\{i\in \mathcal{V}':\Delta_i>4\epsilon\}$ (note that $\mathcal{V}'=\mathcal{B}_0$). For each $i\in\mathcal{Q}$, let
\begin{align}
	m_i=\min\left\{m\ge 0:2^{-m}<\frac{\sqrt{2\lambda}(\Delta_i-3\epsilon)}{4}\right\}.
\end{align}
One can easily verify that
\begin{align}
	\min\left\{\frac{1}{2},\frac{\sqrt{2\lambda}(\Delta_i-3\epsilon)}{8}\right\}\le 2^{-m_i}<\frac{\sqrt{2\lambda}(\Delta_i-3\epsilon)}{4},
\end{align}
and
\begin{align}
	\max_{i\in\mathcal{Q}}m_i\le \max\left\{1,\left\lceil\log_2\left(\frac{8}{\sqrt{2\lambda}\epsilon}\right)\right\rceil\right\}.
\end{align}

We first consider suboptimal arms in $\mathcal{Q}$ and analyze regret in the following cases:

\noindent (a) Some suboptimal arm $i\in\mathcal{Q}$ is not eliminated in round $m_i$ (or before) with an optimal arm $i_{\max}\in \mathcal{B}_{m_i}$.
	
	Consider $i\in\mathcal{Q}$, note that if
	\begin{equation}\label{pfthm5ieq1}
		\begin{aligned}
		\frac{\sum_{j\in\mathcal{N}'[i]}\bar{x}_j(m)\tau_j(m)}{\sum_{j\in\mathcal{N}'[i]}\tau_j(m)}\le&\frac{\sum_{j\in\mathcal{N}'[i]}\mu_j\tau_j(m)}{\sum_{j\in\mathcal{N}'[i]}\tau_j(m)}+\sqrt{\frac{\log(T\tilde{\Delta}_m^2)}{2\sum_{j\in\mathcal{N}'[i]}\tau_j(m)}},
		\end{aligned}
	\end{equation}
	and
	\begin{equation}\label{pfthm5ieq2}
		\begin{aligned}
		\frac{\sum_{j\in\mathcal{N}'[i_{\max}]}\bar{x}_j(m)\tau_j(m)}{\sum_{j\in\mathcal{N}'[i_{\max}]}\tau_j(m)}\ge&\frac{\sum_{j\in\mathcal{N}'[i_{\max}]}\mu_j\tau_j(m)}{\sum_{j\in\mathcal{N}'[i_{\max}]}\tau_j(m)}-\sqrt{\frac{\log(T\tilde{\Delta}_m^2)}{2\sum_{j\in\mathcal{N}'[i_{\max}]}\tau_j(m)}},
		\end{aligned}
	\end{equation}
	hold for $m=m_i$, then under the assumption that $i_{\max},i\in \mathcal{B}_{m_i}$, we have
	\begin{equation}
		\begin{aligned}
			\sqrt{\frac{\log(T\tilde{\Delta}_{m_i}^2)}{2\sum_{j\in\mathcal{N}'[i]}\tau_j(m_i)}}\le&\sqrt{\frac{\log(T\tilde{\Delta}_{m_i}^2)}{2\lambda\sum_{j\in\mathcal{N}'[i]}z_j\log(T\tilde{\Delta}_{m_i}^2)/\tilde{\Delta}_{m_i}^2}}
			\le \frac{\tilde{\Delta}_{m_i}}{\sqrt{2\lambda}}<\frac{\Delta_i-3\epsilon}{4},
		\end{aligned}
	\end{equation}	 
	\begin{equation}
		\begin{aligned}
			\sqrt{\frac{\log(T\tilde{\Delta}_{m_i}^2)}{2\sum_{j\in\mathcal{N}'[i_{\max}]}\tau_j(m_i)}}\le&\sqrt{\frac{\log(T\tilde{\Delta}_{m_i}^2)}{2\lambda\sum_{j\in\mathcal{N}'[i_{\max}]}z_j\log(T\tilde{\Delta}_{m_i}^2)/\tilde{\Delta}_{m_i}^2}}
			\le\frac{\tilde{\Delta}_{m_i}}{\sqrt{2\lambda}}<\frac{\Delta_i-3\epsilon}{4}.
		\end{aligned}
	\end{equation}	
	Thus,
	\begin{equation}
		\begin{aligned}
			&\frac{\sum_{j\in\mathcal{N}'[i]}\bar{x}_j(m_i)\tau_j(m_i)}{\sum_{j\in\mathcal{N}'[i]}\tau_j(m_i)}+\sqrt{\frac{\log(T\tilde{\Delta}_{m_i}^2)}{2\sum_{j\in\mathcal{N}'[i]}\tau_j(m_i)}}+\epsilon\\
			\le&\frac{\sum_{j\in\mathcal{N}'[i]}\mu_j\tau_j(m_i)}{\sum_{j\in\mathcal{N}'[i]}\tau_j(m_i)}+\frac{\Delta_i-3\epsilon}{2}+\epsilon\\
			\le&\mu_i+2\epsilon+\frac{\Delta_i-3\epsilon}{2}=\mu_{i_{\max}}-\epsilon-\frac{\Delta_i-3\epsilon}{2}\\
			\le&\frac{\sum_{j\in\mathcal{N}'[i_{\max}]}\mu_j\tau_j(m_i)}{\sum_{j\in\mathcal{N}'[i_{\max}]}\tau_j(m_i)}-2\sqrt{\frac{\log(T\tilde{\Delta}_{m_i}^2)}{2\sum_{j\in\mathcal{N}'[i_{\max}]}\tau_j(m_i)}}\\
			\le &\frac{\sum_{j\in\mathcal{N}'[i_{\max}]}\bar{x}_j(m_i)\tau_j(m_i)}{\sum_{j\in\mathcal{N}'[i_{\max}]}\tau_j(m_i)}-\sqrt{\frac{\log(T\tilde{\Delta}_{m_i}^2)}{2\sum_{j\in\mathcal{N}'[i_{\max}]}\tau_j(m_i)}}.
		\end{aligned}
	\end{equation}
	
	Therefore, arm $i$ will be eliminated in round $m_i$. Using Hoeffding's inequality, we know that for every $m=0,1,2,...,$
	\begin{align}\label{pfthm5ieq3}
		\mathbb{P}\{(\ref{pfthm5ieq1}) \textrm{ doesn't hold}\}\le\frac{1}{T\tilde{\Delta}_{m}^2},
	\end{align}
	\begin{align}\label{pfthm5ieq4}
		\mathbb{P}\{(\ref{pfthm5ieq2}) \textrm{ doesn't hold}\}\le\frac{1}{T\tilde{\Delta}_m^2}.
	\end{align}
	As a consequence, the probability that a suboptimal arm $i$ is not eliminated in round $m_i$ (or before) by an optimal arm is bounded by $2/(T\tilde{\Delta}_{m_i}^2)$ and thus, the regret contributed by case (a) is upper bounded by
	\begin{align}
		R_a(T)\le\sum_{i\in\mathcal{Q}}\frac{2\Delta_i}{\tilde{\Delta}_{m_i}^2}=O(|\mathcal{V}'|).
	\end{align}
	
\noindent (b) The last remaining optimal arm $i_{\max}$ is eliminated by some suboptimal arm $i$ in some round $m^*<m_f$.
	
	Note that if (\ref{pfthm5ieq1}) and (\ref{pfthm5ieq2}) hold at $m=m^*$, then
	\begin{equation}
		\begin{aligned}
			&\frac{\sum_{j\in\mathcal{N}'[i_{\max}]}\bar{x}_j(m^*)\tau_j(m^*)}{\sum_{j\in\mathcal{N}'[i_{\max}]}\tau_j(m^*)}+\sqrt{\frac{\log(T\tilde{\Delta}_{m^*}^2)}{2\sum_{j\in\mathcal{N}'[i_{\max}]}\tau_j(m^*)}}+\epsilon\\
			\ge&\frac{\sum_{j\in\mathcal{N}'[i_{\max}]}\mu_j\tau_j(m^*)}{\sum_{j\in\mathcal{N}'[i_{\max}]}\tau_j(m^*)}+\epsilon\ge\frac{\sum_{j\in\mathcal{N}'[i]}\mu_j\tau_j(m^*)}{\sum_{j\in\mathcal{N}'[i]}\tau_j(m^*)}\\
			\ge&\frac{\sum_{j\in\mathcal{N}'[i]}\bar{x}_j(m^*)\tau_j(m^*)}{\sum_{j\in\mathcal{N}'[i]}\tau_j(m^*)}-\sqrt{\frac{\log(T\tilde{\Delta}_{m^*}^2)}{2\sum_{j\in\mathcal{N}'[i]}\tau_j(m^*)}}.
		\end{aligned}
	\end{equation}
	Therefore, the optimal arm $i_{\max}$ will not be eliminated in round $m^*$. Consequently, by (\ref{pfthm5ieq3}) and (\ref{pfthm5ieq4}) the probability that $i_{\max}$ is eliminated by a suboptimal arm $i$ in round $m^*$ is upper bounded by $2/(T\tilde{\Delta}_{m^*}^2)$. Thus the regret contributed by case (b) is upper bounded by
	\begin{equation}
		\begin{aligned}
			R_{b}(T)&\le\sum_{m^*=0}^{m_f}\sum_{i\in\mathcal{V}'\setminus\mathcal{A}}\frac{2}{T\tilde{\Delta}_{m^*}^2}\max_{j\in\mathcal{V}'\setminus\mathcal{A}}\Delta_jT\\
			&\le \sum_{i\in\mathcal{V}'\setminus\mathcal{A}}\sum_{m^*=0}^{m_f}\frac{2}{2^{-2m^*}}\\
			&= \sum_{i\in\mathcal{V}'\setminus\mathcal{A}}\frac{2(2^{2m_f+2}-1)}{3}\\
			&\le \sum_{i\in\mathcal{V}'\setminus\mathcal{A}}\frac{2(16\cdot (\frac{8}{\sqrt{2\lambda}\epsilon})^2-1)}{3}=O(|\mathcal{V}'|).
		\end{aligned}
	\end{equation}
	
\noindent(c) Each arm $i\in\mathcal{Q}$ is eliminated in (or before) round $m_i$. Note that arm $i$ will be played until the last arm in $\mathcal{N}'[i]$ is eliminated or the last round $m_f\le \lceil\log_2(8/\sqrt{2\lambda}\epsilon)\rceil$. Thus,
	\begin{align}
		R_{c}(T)\le\sum_{i\in\mathcal{Q}}\Delta_iz_i\frac{\lambda\log(T\tilde{\Delta}_{m_i'}^2)}{\tilde{\Delta}_{m_i'}^2},
	\end{align}
	where
	\begin{align}
		m_i'\le \min\left\{\max_{j\in\mathcal{N}'[i]}m_j,\left\lceil\log_2\left(\frac{8}{\sqrt{2\lambda}\epsilon}\right)\right\rceil\right\}.
	\end{align}

Therefore, the regret contributed by arms in $\mathcal{Q}$ is upper bounded by
\begin{align}
	R_{\mathcal{Q}}(T)\le\sum_{i\in\mathcal{Q}}\Delta_iz_i\frac{32\log(T\hat{\Delta}_i^2)}{\hat{\Delta}_i^2}+O(|\mathcal{V}'|),
\end{align}
where
\begin{align}
	\hat{\Delta}_i=\max\{\min_{j\in\mathcal{N}'[i]}\Delta_j-3\epsilon,\epsilon\}.
\end{align}

Moreover, for each arm $j\in \mathcal{V}'\setminus(\mathcal{Q}\cup\mathcal{A})$, if $j$ is eliminated before $m_f$, then the number of times that arm $j$ has been played up to time $T$ is upper bounded by
\begin{align}
	\mathbb{E}[\tau_j(T)]\le\frac{32z_j\log(T\epsilon^2)}{\epsilon^2}.
\end{align}
Otherwise, $j$ will only be played when $L_j(t)>L_{i_{\max}}(t)$ if $i_{\max}$ is not eliminated. Since we have already shown in case (b) that the regret caused by the fact that $i_{\max}$ is eliminated before $m_f$ is upper bounded by $O(|\mathcal{V}'|)$, we assume that $i_{\max}$ is not eliminated after $m_f$ rounds. Using an argument similar to that in the proof of Theorem \ref{thmupperbound}, we have
\begin{align}
	\mathbb{E}[\tau_j(T)]\le\frac{8\log T}{\Delta_i^2}.
\end{align}
Note that the constant before $\log T$ becomes $8/\Delta_i^2$ instead of $32/\Delta_i^2$ because the reward distributions are assumed to be $1/2$ sub-Gaussian. Thus, the total regret of LSDT-PSI is upper bounded by
\begin{equation}
		\begin{aligned}
			R(T)&\le\sum_{j\in\mathcal{V}'\setminus(\mathcal{Q\cup A})}\Delta_j\max\left\{\frac{8\log T}{\Delta_j^2},\frac{32z_j\log(T\epsilon^2)}{\epsilon^2}\right\}+\sum_{i\in \mathcal{Q}}\Delta_iz_i\frac{32\log(T\hat{\Delta}_i^2)}{\hat{\Delta}_i^2}+O(|\mathcal{V}'|).
		\end{aligned}
\end{equation}

\section{Proof of Corollary 1}\label{pfoptpsi}

According to Theorem \ref{upperboundpsi}, we have that for every realization of the partially revealed UIG $\mathcal{G}_{\epsilon}=(\mathcal{V},\mathcal{E}_{\epsilon}^{S},\mathcal{E}_{\epsilon}^D)$, the expected regret of LSDT-PSI is upper bounded by 
\begin{align}
	O\Big((|\mathcal{V}\setminus(\mathcal{Q}\cup\mathcal{A})|+\sum_{i\in\mathcal{Q}}z_i)\log T\Big),
\end{align}
where $\mathcal{Q}=\{i\in\mathcal{V}':\Delta_i>4\epsilon\}$. Let $C_{\textrm{PSI}}=|\mathcal{V}\setminus(\mathcal{Q}\cup\mathcal{A})|+\sum_{i\in\mathcal{Q}}z_i$, we need to show that 
\begin{align}\label{expectation_observation}
	\mathbb{E}_{\mathcal{E}_{\epsilon}^S,\mathcal{E}_{\epsilon}^{D}}[C_{\textrm{PSI}}]\le \alpha(1+|\mathcal{B}_{i_{\max}}^*\setminus\mathcal{A}|),
\end{align}
where $\alpha$ is a constant independent of $T$ and the size of the action space. We simplify the notation of expectation in (\ref{expectation_observation}) to $\mathbb{E}[C_{\textrm{PSI}}]$. Note that 
\begin{align}
	\mathbb{E}[C_{\textrm{PSI}}]=\mathbb{E}[C_{\textrm{PSI}}\big| F]\mathbb{P}(F)+\mathbb{E}[C_{\textrm{PSI}}\big|\bar{F}]\mathbb{P}(\bar{F})
\end{align}
where $F=\{\textrm{every~}i\not\in\mathcal{B}^*\textrm{~is eliminated from }\mathcal{B}_0\}$.

From Theorem 5,  the probability that every arm $i\not\in\mathcal{B}^*$ is not eliminated is upper bounded by $1/K^2$, therefore, we have
\begin{align}
\mathbb{P}(\bar{F})\le\sum_{i\not\in\mathcal{B}^{*}}\frac{1}{K^2}\le\frac{1}{K}.
\end{align}
It is clear that $\mathbb{E}[C_{\textrm{PSI}}\big|\bar{F}]\le K$ and $\mathbb{P}(F)\le 1$, therefore it suffices to show that 
\begin{align}\label{pfthm6eq2}
\mathbb{E}[C_{\textrm{PSI}}\big| F]\le \alpha(1+|\mathcal{B}_{i_{\max}}^*\setminus\mathcal{A}|)-1.
\end{align}

Notice that given $F$, every arm out of $\mathcal{B}^*$ is eliminated. Besides, we assumed that $\mathcal{B}_{i_{\min}}^*\subseteq\mathcal{Q}$. Thus, 
\begin{align}\label{pfthm6eq3}
\mathbb{E}[C_{\textrm{PSI}}\big| F]&=\mathbb{E}\Big[\sum_{i\in\mathcal{B}_{i_{\min}}^*}z_i\Big|F\Big]+|\mathcal{B}_{i_{\max}}^*\setminus\mathcal{A}|.
\end{align}
Moreover, it is clear that no matter what realization of the revealed UIG is, every arm $i\in\mathcal{B}_{i_{\min}}^*$ will not be eliminated. The fact that every arm $i\not\in\mathcal{B}^*$ is eliminated only affects the probabilistic assumptions on edges with at least one end point not in $\mathcal{B}^*$. Therefore, we can claim that conditioned on $F$, every type-S edge between arms in $\mathcal{B}_{i_{\min}}^*$ is still observed independently with probability $p_S$ and hence $\mathbb{E}[\sum_{i\in\mathcal{B}_{i_{\min}}^*}z_i|F]$ is equal to the expectation of the optimal value $C_L$ of the following linear program:
\begin{equation}
\begin{aligned}
C_L=&\min_{z_1,...,z_L}\sum_{i=1}^{L}z_i,\\
s.t.& ~~~z_i+\sum_{j\neq i}z_j\mathbb{I}\{(i,j)\in\mathcal{E}\}\ge 1,\forall i,\\
&~~~z_i\ge 0,\forall i.
\end{aligned}
\end{equation}
where $L=|\mathcal{B}_{i_{\min}}^*|$ and $\forall i,j\in[L],~(i,j)\in\mathcal{E}$ happens independently with probability $p=p_S$.
We show that $\mathbb{E}[C_L]\le c_p$ where $c_p$ is a constant only related to $p_S$. We consider a solution $z_i^*=\frac{2}{pL},\forall i\in[L]$. We first show that $\{z_i^*\}$ is in the feasible region with probability at least $1-1/L$. Define $A=\{\{z_i^*\}\textrm{~is feasible}\}$, then 
\begin{align}
\mathbb{P}(\bar{A})\le&\sum_{i=1}^{L}\mathbb{P}\Big(z_i^*+\sum_{j\neq i}z_j^*\mathbb{I}\{(i,j)\in\mathcal{E}\}< 1\Big)\\
=&\sum_{i=1}^{L}\mathbb{P}\Bigg(\frac{1}{L-1}\sum_{j\neq i}\mathbb{I}\{(i,j)\in\mathcal{E}\}<\frac{Lp-2}{2(L-1)}\Bigg)\\
\le&\sum_{i=1}^{L}\mathbb{P}\Bigg(\frac{1}{L-1}\sum_{j\neq i}\mathbb{I}\{(i,j)\in\mathcal{E}\}<\frac{p}{2}\Bigg)\\
\le&\sum_{i=1}^{L}\mathbb{P}\Bigg(\Big(\frac{1}{L-1}\sum_{j\neq i}\mathbb{I}\{(i,j)\in\mathcal{E}\}\Big)-p<-\frac{p}{2}\Bigg)\\
\le&\sum_{i=1}^{L}e^{-2(L-1)\frac{p^2}{4}}\label{pfthm6eq1}
\end{align}
Note that the last inequality is derived through the Hoeffding inequality. If $p>\sqrt{\frac{4\log L}{L-1}}$,\footnote{Without loss of generality, we assume that $\sqrt{\frac{4\log L}{L-1}}<1$. Otherwise, $\mathbb{E}[C_L]$ is trivially upper bounded by a constant independent of $L$.} the RHS of (\ref{pfthm6eq1}) is upper bounded by $1/L$. Since it is obvious that $C_L\le L$, we have
\begin{align}
	\mathbb{E}[C_L]&=\mathbb{E}[C_L|A]\mathbb{P}(A)+\mathbb{E}[C_L|\bar{A}]\mathbb{P}(\bar{A})\\
	&\le\sum_{i=1}^{L}z_i^* + 1=\frac{2}{p}+1=\beta_{p,1}.
\end{align}
On the other hand, if $p\le \sqrt{\frac{4\log L}{L-1}}$ (this is equivalent to that $L$ is smaller than a constant that only depends on $p$, we denote the constant as $\beta_{p,2}$), we have $\mathbb{E}[C_L]\le L\le \beta_{p,2}$. In summary, if we let $c_p=\max(\beta_{p,1},\beta_{p,2})$, we have that $\mathbb{E}[C_L]\le c_p$. Finally, we let $\alpha=c_p+1$ and combining with (\ref{pfthm6eq3}), we get the desired result in (\ref{pfthm6eq2}).


\ifCLASSOPTIONcaptionsoff
  \newpage
\fi

\end{document}